\documentclass[preprint,12pt]{elsarticle}

\usepackage{amssymb}
\usepackage{amsmath}
\usepackage{amsthm}
\usepackage{siunitx}
\usepackage[english]{babel}
\usepackage{caption}
\usepackage{subcaption}
\usepackage{url}            % simple URL typesetting
\usepackage{booktabs}       % professional-quality tables
\usepackage{amsfonts}       % blackboard math symbols
\usepackage{nicefrac}       % compact symbols for 1/2, etc.
\usepackage{microtype}      % microtypography
\usepackage{xcolor}         % colors
\usepackage{acronym}
\usepackage{xspace}
\usepackage{listings}
\usepackage{fancyvrb}
\usepackage{algorithm}
\usepackage[noend]{algpseudocode}
\usepackage{newfloat}
\usepackage{paralist}
\usepackage{graphicx}
\usepackage{multirow}
\usepackage{tikz}
\usetikzlibrary{positioning, shapes.geometric, arrows}
\acrodef{SAM}{Safe Action Model Learning}
\acrodef{N-SAM}{Numeric Safe Action Models Learning}
\acrodef{UQV}{universally quantified variables}
\acrodef{PAC}{Probably Approximately Correct}
\acrodef{PMnr}{PlanMiner}

\newtheorem{definition}{Definition}

\newtheorem{theorem}{Theorem}[section]
\newtheorem{lemma}[theorem]{Lemma}
\newtheorem{example}{Example}
\newtheorem{proposition}{Proposition}
\newtheorem{corollary}{Corollary}
\newcommand{\pac}{\ac{PAC}\xspace}
\newcommand{\sam}{\ac{SAM}\xspace}
\newcommand{\nsam}{\ac{N-SAM}\xspace}
\newcommand{\algname}{N-SAM$^*$\xspace}
\newcommand{\tuple}[1]{\ensuremath{\left \langle #1 \right \rangle }}

\newcommand{\pre}{\textit{pre}\xspace}
\newcommand{\eff}{\textit{eff}\xspace}
\newcommand{\name}{\textit{name}\xspace}

\newcommand{\realm}{\ensuremath{M^*}\xspace}
\newcommand{\pbl}{pb-literal\xspace}
\newcommand{\pbls}{pb-literals\xspace}
\newcommand{\bei}{BEI\xspace}
\newcommand{\moveslow}{\textit{move-slow}}

\newcommand{\params}{\textit{params}}

\newcommand{\norm}[2]{\left \lVert #1 \right \rVert_{#2}}

\newcommand{\liftl}{\ensuremath{\ell}} 
\newcommand{\lifta}{\mathsf{\alpha}}

\newcommand{\pbf}{pb-function\xspace}
\newcommand{\pbfs}{pb-functions\xspace}
\newcommand{\pmn}{\ac{PMnr}\xspace}

\journal{Artificial Intelligence}

\begin{document}

\begin{frontmatter}



\title{Learning Safe Numeric Planning Action Models}


% \author[label1]{Yarin Benyamin}
\author[label1]{Argaman Mordoch}
\author[label1]{Shahaf S. Shperberg}
\author[label1]{Roni Stern}
\author[label2]{Brendan Juba}
\affiliation[label1]{organization={Software and Information Systems Engineering, Ben-Gurion University of the Negev},city={Be'er Sheva},
country={Israel}}
\affiliation[label2]{organization={Department of Computer Science and Engineering, Washington University},city={St. Louis},
country={United States}}

%% Abstract
\begin{abstract}
A significant challenge in applying planning technology to real-world problems lies in obtaining a planning model that accurately represents the problem's dynamics.
Obtaining a planning model is even more challenging in mission-critical domains, where a trial-and-error approach to learning how to act is not an option. 
In such domains, the action model used to generate plans must be \emph{safe}, in the sense that plans generated with it must be applicable and achieve their goals. 
% Learning safe action models for planning has been mostly explored for domains in which states are sufficiently described with Boolean variables. 
% In this work, we go beyond this limitation and propose the \nsam algorithm. 
In this work, we present \nsam, an action model learning algorithm capable of learning safe numeric preconditions and effects.
We prove that \nsam runs in linear time in the number of observations and, under certain conditions, is guaranteed to return safe action models. 
However, to preserve this safety guarantee, \nsam must observe a substantial number of examples for each action before including it in the learned model.
We address this limitation of \nsam and propose \algname, an extension to the \nsam algorithm that always returns an action model where every observed action is applicable at least in some states, even if it was observed only once.
\algname does so without compromising the safety of the returned action model. We prove that \algname is optimal in terms of sample complexity compared to any other algorithm that guarantees safety. 
\nsam and \algname are evaluated over an extensive benchmark of numeric planning domains, and their performance is compared to a state-of-the-art numeric action model learning algorithm. 
We also provide a discussion on the impact of numerical accuracy on the learning process.
% An empirical study on ten benchmark domains shows that \nsam and \algname learn safe action models that can solve numeric planning problems. 
% Comparison with PlanMiner shows that our approach learns action models that accurately solve more planning problems in planning domains that contain more numeric preconditions and effects.
\end{abstract}


\end{frontmatter}

\section{Introduction}
Automated domain-independent planning is a long-standing goal of Artificial Intelligence (AI) research. 
Most existing domain-independent planning algorithms rely on the availability of a \emph{domain model}, written in some description language such as the Planning Domain Definition Language (PDDL)~\cite{aeronautiques1998pddl} or its later extensions~\cite{fox2003pddl2,fox2002pddl+}. 
A domain model in domain-independent planning usually includes an \emph{action model} that specifies which action can be applied and, for each action, the preconditions to apply it and its effects on the environment. 
Manually formulating an action model for real-world problems is notoriously challenging and can be error-prone.
Consequently, algorithms have been developed to automatically learn action models from observations~\cite[inter alia]{cresswell2011generalised,aineto2019learning,yang2007learning}. 

% Why safety, what is NSAM
% However, using a learned action model for planning is risky as it may lead the agent to execute inapplicable actions. 
However, a plan generated using a learned action model may not be \emph{sound}, i.e., it may include actions that cannot be applied or do not achieve their intended effects according to the real (unknown) action model. 
The \sam algorithm~\cite{stern2017efficientAndSafe,juba2021safe} addresses this concern. 
It returns a \emph{safe} action model, which is defined as an action model that, under certain conditions, guarantees plans generated with it are sound w.r.t. the real (unknown) action model. 
\sam only supports learning action models for classical planning domains, but it has been extended to support action models with stochastic effects~\cite{juba2022learning}, partial observability~\cite{le2024learning}, and conditional effects~\cite{mordoch2024safe}.

In this work, we present and expand the \nsam algorithm~\citep{mordoch2023learning}, a safe action model learning algorithm that can learn action models in settings where actions may have numeric preconditions and effects. 
We show that the sample complexity of \nsam is not optimal.
Given an action $a$ and a set of $n$ state variables, \nsam requires observing at least $n+1$ affine independent states in which $a$ was executed before including $a$ in the returned action model. 
Specifically, if \nsam lacks this initial set of observations, it deems the action inapplicable in every state. 
This limits the applicability of \nsam, especially in large domains with many state variables.

To overcome this limitation, we propose \algname, an enhanced version of \nsam. 
\algname includes every observed action in the action model it returns, even if it is only observed once, without compromising \nsam's safety guarantees. 
\algname runs in time that is linear in the number of observations, returns a safe action model, and achieves optimal sample complexity. 
We empirically compare \nsam and \algname on a set of benchmark domains. 
The results show that \algname learns an action model that is as good, and in two domains significantly better than \nsam.

We also compare our algorithms to PlanMiner --- a state-of-the-art numeric action model learning algorithm.
We show that in domains that contain significantly more numeric attributes, \nsam and \algname, outperform PlanMiner. 

Finally, we discuss the impact of numerical imprecision on the learning process.
We provide empirical results demonstrating how varying numerical precision affects the ability to solve planning problems.
% \roni{Last paragraph will be updated with Planminer}

% \argaman{should we explain that the model is lifted?}

\section{Preliminaries}
%TODO: Copy and rephrase from  NSAM
We focus on planning problems in domains with deterministic action outcomes and fully observable states, represented using a mix of Boolean and continuous state variables. 
Such problems can be modeled using the PDDL2.1~\citep{fox2003pddl2} language.
%\footnote{Technically, we focus on level 2 of PDDL2.1.} 
%We use the same notation as presented in~\citep{mordoch2023learning} to define a numeric planning problem in PDDL2.1. 
A \emph{domain} is defined by a tuple $D=\tuple{F, X, A, M}$, 
where $F$ is a finite set of Boolean variables, referred to as fluents;
$X$ is a set of numeric variables referred to as functions;
$A$ is a set of actions; and $M$ is an action model for these actions, as described below. 
A state is the assignment of values to all variables in $F \cup X$. 
For a state variable $v\in F\cup X$, we denote by $s(v)$ the value assigned to $v$ in state $s$. %\roni{Do we need the s(v) notation?}YES
% Actions
Every action $a\in A$ is defined by a tuple $\tuple{\name(a), \params(a)}$,
representing the action's name and parameters. 
An action model $M$ is a pair of functions, $\pre_{M}$ and $\eff_{M}$, that map actions in $A$ to their preconditions and effects, respectively. 
The preconditions of action $a$, denoted $\pre_M(a)$, consist of a set of assignments over the Boolean fluents and a set of conditions over the functions, specifying the states in which $a$ can be applied. 
These conditions are of the form $(\xi, Rel, k)$, where $\xi$ is an arithmetic expression over $X$, $Rel\in\{\leq, <, =, >, \geq\}$, and $k$ is a number. 
We say that an action $a$ is applicable in a state $s$ under the action model $M$, denoted $app_M(a,s)$, if $s$ satisfies $\pre_M(a)$. 
% Effects
The effects of action $a$, denoted $\eff_M(a)$, are a set of assignments over $F$ and $X$ representing how the state changes after applying $a$. 
An assignment over a Boolean fluent is either \texttt{True} or \texttt{False}.  
An assignment over a function $x \in X$ is a tuple of the form $\tuple{x, op, \xi}$,
where $\xi$ is a numeric expression over $X$, and $op$ one of the following operations: increase (``+=''), decrease (``-=''), or assign (``:='').\footnote{Technically, \nsam also supports the \textit{scale-up} and \textit{scale-down} operations but we omit them from this discussion since their usage is extremely rare.}
Applying $a$ in $s$ according to the action model $M$, denoted $a_M(s)$, is a state that differs from $s$ only according to the assignments in $\eff_M(a)$. 
Unless a distinction between action models is required, we omit the subscript and use $a(s)$, $\pre(a)$, and $\eff(a)$ instead of $a_M(s)$, $\pre_M(a)$, and $\eff_m(a)$.

%\roni{Are we using the a(s) notation?} YES
% A planning problem
A planning problem is defined by the tuple $\tuple{D, s_0, G}$, where $D$ is a domain, $s_0$ is the initial state, and $G$ is the set of problem goals.  
The problem goals $G$ consist of assignments of values to a subset of the Boolean fluents and a set of conditions over the numeric functions. 
A solution to a planning problem is a \emph{plan}, i.e., a sequence of actions $\{a_0,a_1,...,a_n\}$ such that $a_0$ is applicable in $s_0$ and $a_n(a_{n-1}(...a_0(s_0)...))$ results in a state $s_G$ in which $G$ is satisfied. 
A \emph{state transition} is represented as a tuple $\tuple{s, a, s'}$, where $s$ denotes the state before the execution of $a$, and $s'$ is the result of applying $a$ in $s$, i.e., $s' = a(s)$.  
The states $s$ and $s'$ are commonly referred to as the \emph{pre-state} and \emph{post-state}, respectively.
A \textit{trajectory} is a list of state transitions.

% \roni{I don't see the point in delving into the lifted discussion. Only complicates things. I minimized it significantly}
Planning domains and problems are often defined in a \emph{lifted} manner, where actions, fluents, and functions are parameterized, and their parameters may have \emph{types}. 
Grounded actions, fluents, and functions are pairs of the form $\tuple{\nu,b_{\nu}}$ where $\nu$ can be an action, fluent, or function, and $b_{\nu}$ is a function that maps parameters of $\nu$ to concrete objects of compatible types.
The term \emph{literal} refers to either a fluent or its negation.
The concepts of binding, lifting, and grounding for fluents naturally extend to literals.
A state is an assignment of values to all grounded literals and functions. 
A plan is a sequence of \emph{grounded actions}. 
% A trajectory is an alternating sequence of states and grounded actions. 
% Generally, the parameters in an action's preconditions and effects are bound to the action's parameters. 
The preconditions and effects of an action in a lifted domain contain \emph{parameter-bound} literals (\pbls) and functions (\pbfs). 
A \pbl for a lifted action $\lifta$ is a pair $\tuple{\liftl,b_{\liftl,\lifta}}$ where $\liftl$ is a lifted literal and $b_{\liftl,\lifta}$ is a function that maps each parameter of $\liftl$ to a parameter in $\lifta$. \pbfs are similarly defined. 
We denote by $L(\lifta)$ and $X(\lifta)$ the set of \pbls and \pbfs, respectively, that can be bound to $\lifta$, i.e.,
all the \pbls and \pbfs with parameters that match the parameters of $\lifta$.
% \roni{@Argaman: can we explain here what ``relevant'' means? or is it stated later?}\argaman{Later in the paper i defined $X(\lifta)$ which has the same meaning basically, thus to maintain the same notation as the text to come I am changing the notation to $L(\lifta)$ and $X(\lifta)$. Also added an explanation.} 

\begin{example}
    % Throughout this paper, we use the action \moveslow\ in the Farmland domain~\cite{scala2016heuristics}, displayed in Figure~\ref{list:farmland-move-slow-action}, as a running example.
    The Farmland domain~\cite{scala2016heuristics} includes two lifted functions: \textit{(x ?b - farm)}, representing the number of workers in the farm $b$, and \textit{(cost)}, representing the cost of executing an action. 
    One of the actions in this domain, given in Figure~\ref{list:farmland-move-slow-action}, is \moveslow.  
    \moveslow\ has two parameters, two preconditions, and two effects. 
    The preconditions are $(>=\; (x\; ?f1)\; 1)$, which states that executing the action requires at least one worker in farm $?f1$;
    and $(adj\; ?f1\; ?f2)$, stating that moving workers from $?f1$ to $?f2$, requires the farms to be adjacent. 
    The effects of the action decrease $(x\; ?f1)$ by one and increase $(x\; ?f2)$ by one. 
    The sets $L(\moveslow)$, and $X(\moveslow)$ are defined as follows:
    $L(\moveslow)=\{\text{(adj ?f1 ?f2)}, \text{(adj ?f2 ?f1)}, \\\text{(not (adj ?f1 ?f2))}, \text{(not (adj ?f2 ?f1))}\}$
    and $X(\moveslow)=\{\text{(x ?f1)},\\ \text{(x ?f2)},\;\text{(cost)}\}$.
    % \roni{Maybe you can also add here the sets $L(\moveslow)$, and $X(\moveslow)$?}
    \label{ex:moveslow}
\end{example}

% However, we will describe our key algorithmic contribution assuming a grounded domain for ease of exposition. 

\begin{figure}[ht]
\begin{center}
\begingroup
    \fontsize{10pt}{10pt}\selectfont
\begin{Verbatim}[commandchars=\\\{\}]
(\textcolor{blue}{:action} move-slow
\textcolor{blue}{:parameters} (?f1 - farm ?f2 - farm)
\textcolor{blue}{:precondition} (and (>= (x ?f1) 1) (adj ?f1 ?f2))
\textcolor{blue}{:effect} (and (decrease (x ?f1) 1) (increase (x ?f2) 1)))
\end{Verbatim}
\endgroup
\vspace{-0.3cm}
\caption{The \textit{move-slow} action in the Farmland domain.}
\vspace{-0.5cm}
\label{list:farmland-move-slow-action}
\end{center}
\end{figure}

% \argaman{change the problem definition to include state transitions instead of trajectories.}

\subsection{Numeric Planning Algorithms}
Significant progress has been made in the field of numeric planning. 
In this section, we present some of the most notable algorithms capable of solving numeric planning tasks.

Metric-FF~\citep{hoffmann2003metric} extends the FF algorithm~\citep{hoffmann2001ff} originally designed for classical planning to linear numeric planning. 
The classical FF heuristic estimates the number of actions required to reach the goal from a given state $s$, assuming that delete effects are ignored.
Metric-FF's heuristic works on planning problems where the goal comparisons are in the form of $comp \in \{\geq,>\}$, i.e., the problem is \emph{monotonic}.
Under this restriction, Metric-FF transfers the classical FF heuristic by ignoring numeric \textit{decrease} effects, e.g., $x-=y$. 
For non-monotonic planning problems, Metric-FF first converts the problems to Linear Normal Form (LNF) and then applies the heuristic.

ENHSP~\citep{scala2016heuristics,scala2016interval,scala2017landmarks,li2018effect} generalizes the interval-based relaxation technique, which approximates reachable values using upper and lower bounds. 
Unlike Metric-FF, which only supports linear conditions and effects, and limited cyclic dependencies among numeric variables, ENHSP supports non-linear conditions and effects, complex cyclic dependencies, and multiple mathematical functions.
ENHSP includes multiple heuristics such as Additive Interval-Based Relaxation (AIBR) and the numeric Landmark heuristic. 

The Numeric Fast-Downward (NFD) algorithm~\citep{aldinger2017interval} adapts the $h_{max}, h_{add}$, and the $h_{FF}$ heuristics to interval-based numeric relaxation framework. 
NFD also supports non-uniform action costs.  
Several extensions have been implemented in NFD, including techniques for identifying structural symmetries in numeric planning~\citep{shleyfman2023symmetry}, identifying lower and upper bounds on numeric variables in linear planning~\citep{kuroiwa2023bound} and the numeric LM cut heuristic~\citep{kuroiwa2022lmcut}.
% \citet{shleyfman2023symmetry} demonstrated that identifying structural symmetries in numeric planning problems effectively prunes search branches, leading to overall performance improvements. Their implementation of symmetry detection within NFD outperformed state-of-the-art algorithms in both the number of expanded nodes and overall coverage.

Beyond traditional planning algorithms, recent research has explored the integration of machine learning techniques to enhance numeric planning. 
Notable examples are Numeric-ASNets~\citep{wang2024learning} and GOOSE~\citep{chen2024graph}.
Numeric-ASNets ($\nu$-ASNets) builds upon ASNets~\citep{toyer2020asnets}, employing an architecture based on action and state layers. 
The action layers represent grounded actions, while the state layers encode Boolean fluents, numeric functions, and comparisons. 
The network structure models data flow between states and actions, with a key feature being \emph{weight sharing}—that is, modules representing the same lifted property share weights. 
This weight-sharing mechanism enables $\nu$-ASNets to generalize from small training instances to larger and more complex planning problems.

GOOSE is a data-efficient and interpretable machine learning model for learning to solve numeric planning problems. 
The algorithm consists of two stages: (1) encoding the problem to a \textit{Numeric Instance Learning Graph} ($\nu ILG$), and (2) running a Weisfeiler-Lehman (WL) kernel algorithm on a Graph Neural Network (GNN) created from the features of the $\nu ILG$. 
The authors also introduce two heuristics for their GNN, namely, cost-to-go and ranking estimates, which were used to train the model and compute the loss function.

All the aforementioned approaches assume that a planning domain is available prior to the planning process. 
However, this assumption may not hold in some cases. 
To address such cases, many action model learning approaches have been proposed throughout the years.
Next, we present several prominent examples of action model learning algorithms.

\subsection{Action Model Learning} 
Different algorithms have been proposed for learning planning action models. 
Some action-model learning algorithms, such as LOCM~\citep{cresswell2011generalised} and LOCM2~\citep{cresswell2013acquiring}, analyze observed plan sequences instead of state transitions. SLAF~\citep{amir2008learning} can learn action models from partially observable state transitions. 
FAMA~\citep{aineto2019learning}, frames the task of learning an action model as a planning problem, ensuring that the returned action model is consistent with the provided observations.
NOLAM~\citep{Lamanna24}, can learn action models from noisy trajectories. 
LatPlan~\citep{asai2018classical} and ROSAME-I~\citep{xi2024neuro} learn propositional action models from visual inputs.

However, none of the presented algorithms provides execution soundness guarantees, i.e., that plans created with the learned action model are applicable in the real action model. 
The \sam learning framework~\citep{stern2017efficientAndSafe,juba2021safe,juba2022learning,le2024learning,mordoch2024safe} addresses this gap by providing the following guarantee: the learned action model is \emph{safe} in the sense that plans generated with it are guaranteed to be applicable and yield the predicted states. 

The above action model learning algorithms learn action models when the state variables contain only Boolean fluents. 
To the best of our knowledge, NLOCM~\citep{gregory2016domain} and PlanMiner~\citep{segura2021discovering} are the only algorithms that support learning numeric properties in action models.
NLOCM learns action models for classical planning that include action costs. However, it does not learn numeric preconditions or effects unrelated to action costs, and therefore is not applicable to general numeric planning.
PlanMiner learns numeric action models from partially known and noisy plan traces, using machine learning methods, specifically symbolic regression and classification.
The above algorithms, while being innovative, do not provide safety guarantees (similar to previously discussed methods).

In this work, we present \nsam and \algname, action model learning algorithms that learn numeric action models while providing safety guarantees. 
Both algorithms use the \sam learning~\citep{juba2021safe} algorithm to learn the Boolean part of the action model they return. Therefore, we provide a brief explanation of this algorithm here for completeness.

\subsubsection{The \sam Learning Algorithm}\label{sec:sam-reminder}
\sam begins by assuming that all the literals are part of the actions' preconditions and that the effects are empty. 
Then, for every state transition $\tuple{s,a,s'}$ it observes, it applies the following inductive rules:
\begin{itemize}
    \item Rule 1 [not a precondition].  $\forall l=\tuple{\liftl,b_{\liftl}} \notin s: \tuple{\ell,b_{\liftl,\lifta}} \notin \pre(\lifta)$
    \item Rule 2 [not an effect].  $\forall l=\tuple{\liftl,b_{\liftl}} \notin s': \tuple{\ell,b_{\liftl,\lifta}} \notin \eff(\lifta)$
    \item Rule 3 [must be an effect].  $\forall l=\tuple{\liftl,b_{\liftl}} \in s'\setminus s: \tuple{\ell,b_{\liftl,\lifta}} \in \eff(\lifta)$
\end{itemize} 
Rule 1 states that if a literal is not in the state before applying $a$, its parameter-bound version cannot be a precondition. 
Rule 2 states that if a literal is not in the state after applying $a$, its parameter-bound version cannot be an effect. 
Rule 3 states that if a literal is in the difference between the post-state and pre-state, its parameter-bound version must be an effect. 
By applying the above inductive rules, \sam removes redundant preconditions and adds newly observed effects.

% \section{Numeric SAM (\nsam)}
\section{Problem Definition}
\label{sec:problem-def}

We consider a problem solver tasked with solving a numeric planning problem $\tuple{D=\tuple{F,X,A},\mathcal{T}, s_0, G}$. 
The main challenge is that the problem solver does not receive the action model \realm. 
Instead, it receives trajectories of \textit{successfully executed} plans $\mathcal{T}$, extracted from a distribution of problems within the same domain $D$.
A human operator, random exploration, or some other domain-specific process could have generated these trajectories. 
We assume the problem solver has full observability of these trajectories, which means it knows the value of every variable in every state in every trajectory $t\in\mathcal{T}$, as well as the name and parameters of every action in the trajectories. %$a\in A$, with their respected types.

% \roni{You want to mention this when the alg. is described (maybe): The order of the state transitions within the trajectory is irrelevant to the \nsam algorithm.}

% Our approach
Our approach to solving this problem comprises two steps. (1) \emph{learning} an action model $M$ using the given trajectories $\mathcal{T}$, and (2) \emph{planning} using the learned action model $M$, i.e., using an off-the-shelf PDDL 2.1 planner to find a plan for the planning problem $\tuple{\tuple{F,X,A,M}, s_0, G}$. 
% There are many planning algorithms to solve such PDDL2.1 problems, such as Metric FF~\citep{hoffmann2003metric} and ENHSP~\citep{scala2016interval}.
Recall that since the learned action model may differ from the actual action model, planning with it raises both safety and completeness risks.  
The relative importance of each risk is application dependent. 
This work emphasizes addressing the safety risk, which is crucial in applying our method to mission-critical applications or applications in which plan failure is very costly. 
To this end, we aim to learn a \emph{safe action model}~\citep{stern2017efficientAndSafe,juba2021safe}. 
% Previous works~\citep{stern2017efficientAndSafe,juba2021safe} defined a \textit{safe action model} in the Boolean setting as follows:

\begin{definition}[Safe Action Model]
\label{def:safe_action_model}
An action model $M'$ is safe with respect to an action model $M$ iff for every state $s$ 
and action $a$ it holds that 
\begin{equation}
app_{M'}(a,s)\rightarrow (app_M(a,s) \wedge (a_M(s)=a_{M'}(s))) 
    \label{eq:safe_action_model}
\end{equation}
\end{definition}

\noindent That is, for any action $a$ and state $s$ which the action is applicable according to $M$, it is also applicable according to $M'$, and the resulting state of applying $a$ on $s$ according to $M$ is exactly the same as $a_{M'}(s)$.
In the context of numeric planning, this safety definition is applicable only if the learning process is free of numerical errors. 
For simplicity, we assume in the rest of this paper that the learning process does not induce numerical error, and discuss the effects of numeric errors in Section~\ref{sec:numeric-precision}.

\subsection{Theoretical Analysis}
\label{sec:negative-results}

% The problem of learning numeric preconditions and effects of an $\epsilon$-safe action model can be mapped to known results in the \pac Learning literature. This allows us to establish learnability results in our settings and identify which assumptions may enable efficient learning. 
The problem of learning numeric preconditions and effects of a safe action model can be mapped to known results in the \pac Learning literature. This allows us to establish learnability results in our settings and identify which assumptions may enable efficient learning. 
To formalize the notion of efficient learning in our setting, we introduce the following terminology. 
We say that an action model solves a given problem if a complete planner using that action model will be able to solve it. 
We say that an action model learning algorithm is $(1-\epsilon)$-complete if it returns an action model that can solve a given problem with probability at least $1-\epsilon$.\footnote{This, of course, assumes the training and testing problems are drawn from the same distribution of problems.}  

% ability to construct a learning algorithm that learns a safe action model that is \emph{probabilistically complete}, that is, the probability that the learned action model is sufficient to solve a given problem efficiently decreases as we obtain more trajectories. 

\subsubsection{Learning Preconditions} Learning the preconditions of actions can be viewed as the problem of learning a Boolean-valued function, where the training examples are the given trajectories that include that action. 
The given trajectories were created by successfully executing plans in the domain. In every trajectory transition, the preconditions of the executed actions have necessarily been satisfied. 
Thus, learning the preconditions of actions is a special case of learning a Boolean-valued function from only positive examples~\citet{kearns1994learning}. 
\begin{definition}[PAC Learning from Positive Examples]
Let $\mathcal{X}$ denote our set of instances. Let $\mathcal{C}$ and $\mathcal{H}$ be sets of Boolean-valued functions on $\mathcal{X}$. The problem of \emph{PAC learning $\mathcal{C}$ with $\mathcal{H}$ using positive examples} is as follows: we are given parameters $\epsilon,\delta \in (0,1)$, and access to examples from $\mathcal{X}$ sampled from a distribution $P$ supported on $\{x\in\mathcal{X}:c(x)=1\}$ for some $c\in\mathcal{C}$. 
With probability $1-\delta$, we must return $h\in\mathcal{H}$ such that
(1) \emph{$h$ has no false positives:} $h(\mathcal{X})\subseteq c(\mathcal{X})$, and 
(2) \emph{$h$ is $1-\epsilon$ accurate:} $\Pr_{x\in P}[h(x)=c(x)]\geq 1-\epsilon$, where $\Pr_{x\in P}[h(x)=c(x)]$ means the probability that $h(x)=c(x)$ for an example $x\in\mathcal{X}$ sampled from a distribution $P$.    
If an algorithm runs in time polynomial in the representation size of members of $\mathcal{X}$, the representation size of $c$, $1/\epsilon$, and $1/\delta$, then we say that the algorithm is \emph{efficient}.
If $\mathcal{C}=\mathcal{H}$, this is known as the \emph{proper} variant of the problem. Otherwise, the learner is said to be \emph{improper}.
\end{definition}

Next, we reduce the problem of PAC learning from positive examples to the problem of learning preconditions of a safe action model. 
% A safe action model learner that learns domains with preconditions from $\mathcal{C}$ using preconditions from $\mathcal{H}$ can be used to obtain a PAC learning algorithm for learning $\mathcal{C}$ with $\mathcal{H}$ using positive examples:
\begin{proposition}\label{safe-to-pos-pac}
% %Suppose there exists a safe action model learning algorithm for domains with preconditions from $\mathcal{C}$ that produces preconditions from $\mathcal{H}$, for a $\mathcal{C}$ that for any $i^{th}$ attribute, contains a nontrivial precondition that depends only on that attribute. Suppose that the algorithm is guaranteed to be $1-\epsilon$-complete with probability $1-\delta$ when given at least $m(\epsilon,\delta)$ trajectories as input.
% \roni{I don't see why we need this relation to the $i^{th}$ attrbituion. I am also not sure what is a nontrivial precondition (I guess I'm unsure what is a trivial precondition - is it always True?)?
% }
% \brendan{A nontrivial precondition is one that can be falsified --- so yes, the trivial precondition is the always-true precondition. I included this condition so that we could introduce an action with a precondition that depends on an attribute that wasn't in the original domain. Looking back at this now, it may not be necessary after all; we should be able to use a single action, on the original set of attributes as fluents plus one for the goal, with just the original positive-example distribution (extended with the goal fluent) for the initial states.}
% \roni{Instead, I think we need to say here something about completeness}
% \brendan{Yes, you are correct, I wrote this assuming that we were including approximate completeness in the definition of our action model learning task; since it is not, we need to add that this is for safe and approximately complete action models.}
Suppose there exists a safe action model learning algorithm $\mathcal{A}$ for domains with preconditions from $\mathcal{C}$ that produces preconditions from $\mathcal{H}$ such that $\mathcal{A}$ is guaranteed to be $1-\epsilon$-complete with probability $1-\delta$ when given at least $m(\epsilon,\delta)$ trajectories as input for some function $m$. %\roni{What is $m$?}
%Suppose, furthermore, that there is a probability distribution with support equal to $\mathcal{X}$ that can be efficiently sampled. 
Then there is a PAC learning algorithm for learning $\mathcal{C}$ with $\mathcal{H}$ using $m(\epsilon,\delta)$ positive examples. Moreover, if the safe action model learning algorithm is efficient, so is the PAC learning algorithm.
\end{proposition}
% \vspace{-0.3cm}
\begin{proof}
%We reduce PAC learning with positive examples to safe action model learning as follows: We 
Consider a domain with one action $a$ and states given by $\mathcal{X}$ extended by one Boolean attribute, $t$. 
%Let $s_0$ and $s_1$ be values for $s$ that would, respectively, violate and satisfy some precondition in $c_s\in \mathcal{C}$. 
We provide the following trajectories to our safe action model learner: 
%with probability $1/2$, we sample $x$ \roni{who is $x$?}\brendan{$x$ is a random variable sampled from $D$; it represents the input for the original learning from positive examples problem, and it will be used to construct an initial state in the planning domain we construct below.}\roni{Isn't $D$ the domain?}\brendan{Not here, no. It was the probability distribution in the definition/statement of the theorem. Hi. I am using the chat window} from $D$, and append it with %\brendan{There was a mistake here, should have been $s_0$ (was previously $s_1$): }[Hi!, I don't see $D$ there, but maybe I'm just tired. Checking.....Ok, so D is used in the PAC learning problem, but not in the planning one. $]
%$s=s_0$ and $t=0$ to obtain the initial state; and, otherwise, we sample $x$ from our distribution with support $\mathcal{X}$, and append it with $s=s_1$ and $t=0$ to obtain the initial state. 
%\roni{I don't see the difference between the two cases. In one you sample from $D$ and in the other you sample from $\mathcal{X}$. What's the difference?}
%\brendan{$D$ is a distribution that is only supported on positive examples, i.e., satisfying the precondition. $\mathcal{X}$ is the entire domain. Sometimes the precondition is not satisfied on $\mathcal{X}$ because the precondition is nontrivial. The learner then needs to know if the action pos --- for which the precondition will be our classifier --- can be taken when encountering a new problem.}
given an example $x$ sampled from $P$ for the problem of learning from positive examples, we construct the trajectory with $(x,0)$ as the first state, i.e., with $t=0$, followed by the action $a$ and $(x,1)$ as the final state. The associated goal is ``$t=1$.''
%In the first case, the trajectory takes the action pos, followed by a state with $x$ appended with $s=s_0$ and $t=1$. In the second case, the trajectory takes the action neg, followed by a state with $x$ appended with $s=s_0$ and $t=1$. In both cases, the goal is ``$t=1$.'' 
We obtain an action model from the algorithm run with $\delta$ and $\epsilon$ and given $m(\epsilon,\delta)$ examples, and return its precondition for $a$ with $t$ set to $0$ as our solution $h$ for PAC learning. Note that the trajectories are consistent with an action model in which $c\in\mathcal{C}$ (the correct hypothesis for the positive-example PAC learning problem) is the precondition of $a$ and the only effect of $a$ is to set $t=1$.
% \argaman{Does this still hold when referring to the numerical error in the preconditions?}\brendan{This model is still consistent --- this part should not be an issue.}
$h$ must have no false positives because the precondition for $a$ must be safe: if there exists $x\in \mathcal{X}$ such that $h(x)=1$ but $c(x)=0$, our action model would permit $a$ for some $x$ where its precondition is violated, thus violating safety. Similarly, $h$ must be $1-\epsilon$-accurate: our action model learner guarantees that with probability $1-\delta$, it is $1-\epsilon$ complete. %\roni{Wait, where was completeness discussed above?}\brendan{Should have been implied by the definition of the learner; I added the hypothesis just now.} Cool
%But if neg has the precondition $c_s$ that is only satisfied by $s_1$, then only pos can be executed for those examples with $s=s_0$, i.e., the examples sampled from $D$. 
Observe that in our distribution over examples, $t=0$ initially, so to satisfy the goal, the plan must include the action $a$.
Hence, the precondition for $a$ must be satisfied with probability at least $1-\epsilon$ on $P$, or else the action model would fail with probability greater than $\epsilon$. Hence, $h$ is indeed as required for PAC learning. The ``moreover'' part is immediate from the construction.
\end{proof}
\vspace{-0.3cm}
The above reduction allows us to identify which classes of preconditions cannot be efficiently learned by the family of Boolean-valued functions that cannot be efficiently learned from positive examples. 
%By the contrapositive of Proposition~\ref{safe-to-pos-pac} and prior results, we obtain that: % most natural classes are not learnable
\begin{corollary}[\citep{goldberg1992pac}]
The family of preconditions given by single linear inequalities with at most two variables cannot be safely learned by any $\mathcal{H}$. 
% \roni{efficiently?}\brendan{no --- simply impossible, an infinite number of examples is necessary over $\mathbb{R}$ or $\mathbb{Q}$; if you assume finite precision, then you can add ``efficient''} 
\end{corollary}
\noindent For example, the inequality $x\leq y$ cannot be learned by any safe action model learning algorithm.

\begin{corollary}[\citep{kivinen1995learning}]
The family of preconditions given by the disjunction of two univariate inequalities cannot be safely learned by any $\mathcal{H}$.
% \roni{efficiently?}\brendan{same as above, infinitely many examples are necessary} 
\end{corollary}
\noindent For example, the inequality disjunction $x + 3y < 7$ \textit{or} $7x - 5y > 13$ cannot be learned by any safe action model learning algorithm.

In particular, classes of preconditions $\mathcal{C}$ that contain the above representations as special cases cannot be safely learned. 
The strongest class of Boolean-valued function that \emph{is} known to be learnable is ``axis-aligned boxes,'' i.e., conjunctions of univariate inequalities \citep{natarajan1991probably}.

\subsubsection{Learning Effects}
The problem of learning effects is essentially similar to regression under the ``sup norm loss'': we demand a bound on the maximum error that holds with high probability. 
We can characterize the sample complexity of learning effects easily when the errors are considered under the $\ell_\infty$-norm, and observe that since all $\ell_q$-norms are equivalent up to polynomial factors in the dimension, this, in turn, characterizes which families of effects are learnable for all $\ell_q$ norms.
% \roni{I don't follow how this relates to learning effects in a safe action model. Where is the safety introduced here? I think it would mean the $L_\infty$ should be zero but I'm not fully following. @Brendan, can you help?}
% \brendan{We had included an $\epsilon$ tolerance for numerical imprecision etc. in the definition of safe action models, so I wanted to allow for that here. Sure, in the idealized version, we would take $\epsilon = 0$ (i.e., perfect predictions in all coordinates), but I don't expect that to be possible when we're working with numerical models, so I stated this in a relaxed way. It may deserve some further discussion.}
\begin{theorem}(cf.\ \citet[Theorem 3]{anthony1996valid})
Let $\mathcal{A}$ be a class of functions mapping $X$ to $X$, such that the true effects function $A^*$ is in $\mathcal{A}$, and let $\mathcal{A}'_\epsilon$ be the set of Boolean-valued functions of the form $\{A'(s)=I[\|A(s)-A^*(s)\|_\infty\leq \epsilon]:A\in\mathcal{A}\}$. Let $d$ be the VC-dimension of $\mathcal{A}'_\epsilon$. Suppose training and test problems are drawn from a common distribution $D$. Then $\Omega(\frac{1}{\delta_1}(d+\log\frac{1}{\delta_2}))$ training trajectories from $D$ are necessary to identify $A\in\mathcal{A}$ that satisfy $\|A(s)-A^*(s)\|_\infty\leq \epsilon$ with probability $1-\delta_1$ on test trajectories with probability $1-\delta_2$ over the training trajectories. In particular, if $d=\infty$, then $\mathcal{A}$ is not learnable.
\end{theorem}
\begin{proof}
The sup norm regression problem can be reduced to learning of effects as follows: given a training set $\{(x_i,f^*(x_i))\}_{i=1}^m$, construct one-step trajectories for a planning domain with a single action, initial states given by $x_i$, and post-states given by $f^*(x_i)$. Then an estimate of the effect $f$ that is $\epsilon$-close to $f^*$ with probability $1$-$\delta_1$ indeed yields a solution to the original regression problem. The bound thus follows from Theorem 3 of \citet{anthony1996valid}. 
\end{proof}

Thus, we see that some restrictions on the family of effects are necessary for learnability. Fortunately, unlike preconditions, these restrictions are relatively mild. For example, for linear functions in $k$ dimensions, the VC-dimension of the corresponding $\mathcal{A}'_\epsilon$ is $O(k^2)$ \citep[Prop.\ 18]{anthony1996valid}.

\subsection{Assumptions}
\label{sec:assumptions}
% In our numeric planning setting, the actions’ names and parameters (including their types) are observable in the trajectories. 
Based on the theoretical results presented above on the learnability of action models, we limit our attention to scenarios that satisfy the following assumptions: 
\begin{enumerate}
    \item The actions’ preconditions are conjunctions of conditions over Boolean and numeric state variables. 
    \item The conditions over the numeric state variables in actions' preconditions are linear inequalities.
    \item The numeric expressions defining actions' effects are linear combinations of state variables.
\end{enumerate}

While these assumptions restrict the types of domains we consider, they still cover a variety of applications. 
Furthermore, these assumptions hold in most of the domains used within the past numeric planning competitions~\citep{fox2003pddl2,scala2017landmarks,taitler20242023}. 
% While the third assumption requires human intervention, specifying the maximal polynomial degree is insignificant compared to the overall complexity of the models.
Next, we propose Numeric SAM (\nsam), an action model learning algorithm for numeric domains that, under the above assumptions, is guaranteed to output a safe action model.

\section{Numeric SAM (N-SAM)}

\begin{figure}[ht]
    \centering
    \includegraphics[width=\columnwidth]{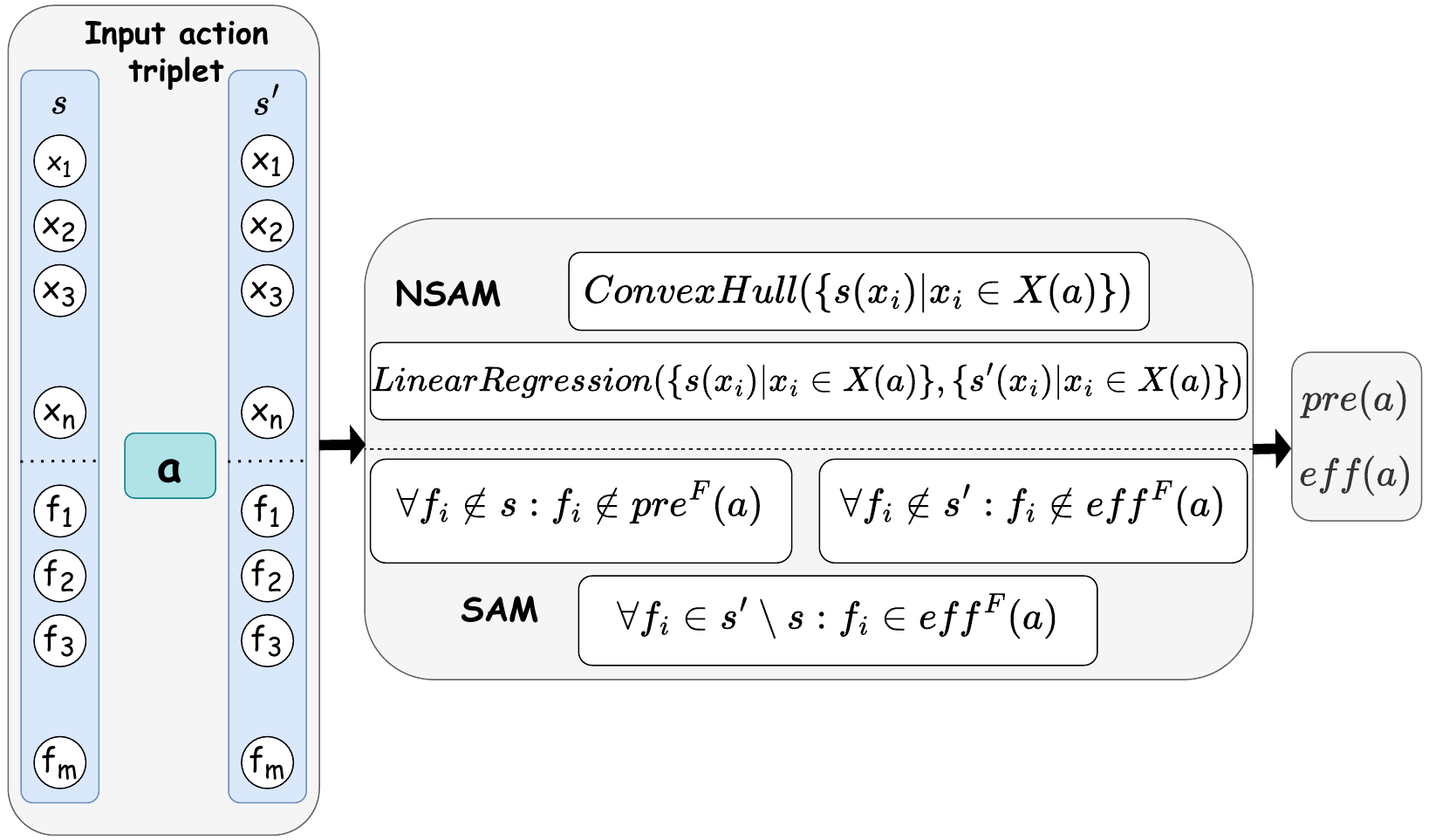}
    \caption{Graphical illustration of the Numeric SAM algorithm. The Boolean part of the action model is learned using the \sam algorithm. 
    The numeric preconditions are learned by applying the convex hull algorithm, and the numeric effects are learned using the linear regression algorithm.}
    \label{fig:nsam-algorithm}
\end{figure}

\begin{algorithm}[tb]
\small
\caption{The \nsam Algorithm}\label{alg:nsam-algorithm}
\begin{algorithmic}[1]
% \scriptsize
\State \textbf{Input}: The observed trajectories $\mathcal{T}$ 
\State \textbf{Output}: a safe action model and a list of actions that are unsafe to use.
\State $A_{\textit{unsafe}}\gets\emptyset$
\For{$\lifta \in A$} 
    \State $\pre^{F}(\lifta) \gets $ all \pbls bound to $\lifta$ \label{alg:nsam-init-bool1}
    \State $\eff^{F}(\lifta)  \gets \emptyset$\label{alg:nsam-init-bool2}
    \State $DB_{pre}(\lifta);\; DB_{post}(\lifta) \gets \emptyset$ \label{alg:nsam-init-num}
\EndFor
\For{$\tuple{s, a, s'}\in\mathcal{T}$}
    \State Apply \sam inductive rules to refine $\pre^{F}$ and $\eff^{F}$\label{alg:nsam-sam-part}
    \State Add $v=\{s(x)|x\in X(\lifta)\}$ to $DB_{pre}(\lifta)$\label{alg:nsam-add-to-db1}
    \State Add $v'=\{s'(x)|x\in X(\lifta)\}$ to $DB_{post}(\lifta)$\label{alg:nsam-add-to-db2}
\EndFor
\For{$\lifta \in A$}
    \State $N_{indp}\gets$ calcMatrixRank($DB_{pre}(\lifta)$)\label{alg:calc-num-independent-points}
    \If{$N_{indp} \geq |X(\lifta)| + 1$}
    % \State \roni{@Argaman: don't we need here an IF to check we have enough independent points?}\argaman{done :)}
        \State$\pre^{X}(\lifta) \gets $ ConvexHull$(DB_{pre}(\lifta))$\label{ch_pre_state}
        \State $\eff^{X}(\lifta) \gets $ LinearRegression$(DB_{pre}(\lifta),DB_{post}(\lifta))$\label{regression-post-state}
    \Else
        \State Add $\lifta$ to $A_{\textit{unsafe}}$.\label{alg:remove-unsafe-action}
        % \roni{Changed to maintain a list of unsafe actions}
        % \State The action $\lifta$ is not safe. Remove $\lifta$ from $A$.\label{alg:remove-unsafe-action}
    \EndIf
\EndFor
\State $\pre \gets \pre^{F} \cup \pre^{X}$ 
\State $\eff \gets \eff^{F} \cup \eff^{X}$ 
\State \Return  $\pre,\eff, A_{\textit{unsafe}}$
\end{algorithmic}
\end{algorithm}

% \argaman{Change the subscript to be superscript and instead of bool to F so it would be consistent with the previous definitions. also change in the picture.}
\nsam is an action model learning algorithm from the \sam learning framework designed to learn safe lifted action models with Boolean and numeric state variables. 
The \nsam algorithm is outlined in Algorithm~\ref{alg:nsam-algorithm}.
\nsam learns an action model that includes all actions observed in the given state transitions $\mathcal{T}$.\footnote{The order of the state transitions in the trajectory is irrelevant to the \nsam algorithm.} 
For each action $\lifta\in A$, \nsam initializes $\pre^{F}(\lifta)$ to include all the \pbls bound to $\lifta$ and initializes $\eff^{F}(\lifta)$ to be an empty set (lines~\ref{alg:nsam-init-bool1}-\ref{alg:nsam-init-bool2}). % the Boolean effects
It also initializes two tables, $DB_{pre}(\lifta)$ and $DB_{post}(\lifta)$, representing the values of all the \pbfs that can be bound to $\lifta$, i.e., $X(\lifta)$ (line~\ref{alg:nsam-init-num}). 
For each state transition $\tuple{s, a, s'}$, \nsam iterates over the states' literals and applies \sam's inductive rules (line~\ref{alg:nsam-sam-part}) to remove incorrect Boolean preconditions and add missing Boolean effects. 
For each numeric function $x$ in the state $s$, \nsam constructs the vector $v=\{s(x)|x\in X(\lifta)\}$, which represents the values of the \pbfs bound to $\lifta$ appearing in $s$, and adds it to $DB_{\pre}(\lifta)$.
Similarly, it adds $v'=\{s'(x)|x\in X(\lifta)\}$ to $DB_{post}(\lifta)$ (lines~\ref{alg:nsam-add-to-db1}-\ref{alg:nsam-add-to-db2}).
After processing all state transitions and collecting data,
\nsam calculates the number of linearly-independent vectors in $DB_{pre}(\lifta)$, using matrix rank, and sets this value in the variable $N_{indp}$ (line~\ref{alg:calc-num-independent-points}).
If $N_{indp} \geq |X(\lifta)| + 1$, then there are enough linearly independent points in the dataset to learn the action's preconditions and effects.
Otherwise, \nsam marks the action as unsafe, so planners using the learned action model know not to use it (line~\ref{alg:remove-unsafe-action}). 
% effectively removing it from the learned action model (line~\ref{alg:remove-unsafe-action}).
Next, we describe the process of learning numeric preconditions and effects.
% \nsam proceeds to learn the numeric preconditions and effects, as follows.

\paragraph{Learning numeric preconditions}
Recall that the preconditions learning process is a special case of learning from positive examples.
\citet{natarajan1991probably} observed that when learning from positive examples, the optimal hypothesis is the intersection of all consistent candidate hypotheses. 
In our case, this is precisely the \emph{convex hull} of the observed points.
Thus,
\nsam computes the convex hull of the points in $DB_{\pre}(\lifta)$ and sets the preconditions of $\lifta$ as the set of linear inequalities that define the convex hull (line~\ref{ch_pre_state}). These inequalities can be obtained with off-the-shelf tools.\footnote{We used the convex hull algorithm available in the SciPy library in our implementation.} 
% \argaman{Look at this sentence - compare with complexity.}

\paragraph{Learning numeric effects}
% Let $\eff^{X}(a)$ be the set of numeric state variables relevant to computing the effects of $a$.  
Under the linear effects assumption, the change in any variable $x\in DB_{post}(\lifta)$ is a linear combination of the values of $DB_{\pre}(\lifta)$. 
Thus, we learn the effects of an action using standard linear regression.
In more detail, for every variable $x\in DB_{post}(\lifta)$ and given state transition $\tuple{s,a,s'}$ \nsam creates an equation of the form:
\begin{equation}\label{eq:sys}
    s'(x) = w_0+\sum_{x'\in X(\lifta)} w_{x'}\cdot s(x')
\end{equation}
% $s'(x) = w_0+\sum_{x'\in X(\lifta)} w_{x'}\cdot s(x')$.
% If the resulting system of linear equations contains fewer than $|X(\lifta)|+1$ linearly independent equations, we consider $\lifta$ unsafe and do not include it in the returned action model. 
\nsam finds the unique solution to this set of equations and obtains the values of $w_0$ and $w_{x'}$ for all $x'\in DB_{\pre}(\lifta)$.\footnote{In our implementation, we use least-squares linear regression to obtain these weights.}
Correspondingly, \nsam sets 
$x := w_0+\sum_{x'\in DB_{\pre}(\lifta)} w_{x'}\cdot x'$ as an effect of $\lifta$ (line~\ref{regression-post-state}).

Finally, \nsam returns the learned action model, which consists of the union of numeric and Boolean preconditions and effects.
Figure~\ref{fig:nsam-algorithm} illustrates \nsam's learning process.

\subsection{Extension to Nonlinear Domains} 
\nsam can also learn preconditions and effects with polynomials of some low degree, e.g., quadratic or cubic polynomials.
To enable \nsam to learn action models with polynomial preconditions and effects, we assume that a human operator provides the maximal polynomial degree, $d$, along with the trajectories $\mathcal{T}$.
Although specifying $d$ requires additional human input, we argue that this effort is minimal compared to the complexity of manually defining numeric preconditions and effects.
Given the maximal polynomial degree $d$ and an action $\lifta \in A$, we modify $DB_{\pre}(\lifta)$ to include all possible monomials of the numeric functions up to the degree $d$.
In other words, if the original set of numeric functions had size $K = |X(\lifta)|$, now it expands to $O(K^d)$ functions, accounting for all possible monomials, composing \pbfs relevant to $\lifta$, up to degree $d$.
For example, assume a state transition $\tuple{s, a, s'}$ with the numeric functions $x$ and $y$ appearing in $s$ with the values $3$ and $5$, respectively. Also, assume that the polynomial degree $d=2$. 
\nsam will add the monomials $\{x,y,xy, x^2,y^2\}$ with their respective values $\{3,5,15,9,25\}$ to $DB_{\pre}(\lifta)$.
% If, for example $d=3$, then \nsam would create the following table:
% \[
% DB_{\pre}(\lifta) =
% \{x,y,x^2,xy,y^2,x^3,x^2y,xy^2,y^3\}=\{3,5,9,15,25,27,45,75,125\}
% \]
% \roni{Say what is $d$, and show also higher order of $d$ in the example}

Next, we apply \nsam to this transformed representation, obtaining a set of linear inequalities and effects in the expanded domain. 
Notably, any polynomial function in the original domain can be expressed as a linear function in the expanded feature space, and vice versa. 
This transformation allows \nsam to operate under its original assumptions while capturing polynomial functions.
% \argaman{when we say the degree of a polynomial what do we mean? also add higher degree.}
% We achieve this by mapping our example trajectories into trajectories in a domain with a more extensive set of fluents: for each possible monomial up to the desired degree, we create a fluent taking value equal to the value of the corresponding monomial evaluated on the original fluents' values. For example, if in some state in an example trajectory fluents $x$ and $y$ take values $3$ and $5$, respectively, the new domain might have a fluent $xy$ taking value $15$ in that state in the new trajectory. 
% Our original numeric fluents $X$ are thus mapped to a set of $O(|X|^d)$ fluents, where $d$ is the degree of the polynomials involved. We run \nsam as before, obtaining a set of linear inequalities and effects in this new domain. We can then substitute the corresponding monomials for the fluents in those expressions to obtain polynomial inequalities and effects. We observe that polynomials in the original domain are representable as linear expressions in this expanded domain (and vice-versa). Thus, the expanded domain satisfies the linear preconditions and effects assumption; therefore, \nsam can be applied, and the resulting preconditions and effects can be applied in the original domain.

\subsection{Theoretical Properties}
\label{sec:theoretical-properties}
\noindent Next, we analyze the theoretical properties of the \nsam algorithm.

\subsubsection{Safety}
\label{sec:nsam-safe}
Next, we prove that the action model learned by \nsam, denoted $M_{\nsam}$ is safe with respect to the real action model \realm.
Since \citet{juba2021safe} already proved that applying the \sam learning algorithm returns a safe action model, it is enough to prove that the numeric preconditions and effects are safe. Proving the safety of the effects learned by \nsam is straightforward, as they are obtained via linear regression and we have enough linearly independent samples to obtain a unique solution. Thus, to prove \nsam returns a safe action model, we only need to prove that the preconditions are learned in a safe manner. 
% Next, we provide the proof showing that the preconditions learned using \nsam are safe.[roni: repetition]

% First, we show that the preconditions we learn are safe when the preconditions are a conjunction of linear inequalities (the same proof applies to the polynomial inequalities as well).
% First, we show that the preconditions we learn are safe for a broad family of planning models in which the preconditions are \emph{convex}. 

% \argaman{Please go over this proof. While it is simple I want another set of eyes to validate. @Shahaf + @Roni}\roni{Done}
% Recall, a set of points $C$ (in our case, the points satisfying the preconditions) is convex if for each pair of points $x,y\in C$ it is true that the line segment connecting $x$ and $y$, i.e., $\bar{xy}$, it is true that $\bar{xy}\subset C$. In particular, all the points that satisfy a given set of linear inequalities define a convex set.

% Recall that a set of points (in our case, the points satisfying the preconditions) is said to be convex if for any two points $s$ and $t$ in the set, every \emph{convex combination} $\lambda s+(1-\lambda)t$ for $\lambda\in [0,1]$ is also in the set. In particular, all the points that satisfy a given set of linear inequalities define a convex set.
%% If preconditions are convex then the convex hull precondition is safe
\begin{theorem}
\label{preconds-from-convex-safe}
For every action $\lifta$, every state $s$ satisfying the numeric preconditions of $\lifta$ according to $M_{\nsam}$ also satisfies the numeric preconditions of $\lifta$ according to \realm.  
\end{theorem}
\begin{proof}
% The preconditions of every action are a conjunction of linear inequality, and thus all points that satisfy them form a convex set.

Let $s$ be a state in which $\lifta$ is applicable according to $M_{\nsam}$. 
Recall that we assumed the preconditions of every action in \realm are a conjunction of linear inequalities. 
Therefore, all points that satisfy them form a convex set.
Let $CH_{\realm}$ and $CH_{\nsam}$ be the convex hulls representing all points in which $\lifta$ is applicable in according to $\realm$ and $M_{\nsam}$, respectively. 
By construction, $CH_{\nsam}$ is the convex hull of all the points in $DB_{\pre}(\lifta)$. 
$DB_{\pre}(\lifta)$ is created from the states in which $\lifta$ was successfully applied (according to $\realm$). 
Therefore, we know that every point in $DB_{\pre}(\lifta)$ is within $CH_{\realm}$. 
Since $CH_{\realm}$ is a convex set, then every 
convex combination of points in it also lies in $CH_{\realm}$~\cite{lay2007convex}.
Consequently, every point in $CH_{\nsam}$ must also be in $CH_{\realm}$, and thus if $\lifta$ is applicable in a state $s$ according to $M_{\nsam}$
then it is also applicable in $s$ according to \realm.
% The preconditions of every action are a conjunction of linear inequality, and thus all points that satisfy them form a convex set.
% Let $s$ be a state in which $\lifta$ is applicable according to $M_{\nsam}$. 
% The vector $v=\{s(x)|x\in X(\lifta)\}$ is a vector such that $v$ is a point in the convex hull created by the points in $DB_{\pre}(\lifta)$.
% $DB_{\pre}(\lifta)$ is created from the points which $\lifta$ was successfully applied. 
% \citet{lay2007convex} proved that if a set $C$ is convex then every convex combination of points of $C$ lies in $C$.
% Thus, the convex hull of the points in $DB_{\pre}(\lifta)$ lies in the convex set defining $\pre^X_{\realm}(\lifta)$.
% Consequently, $v$ is also part of the convex set $\pre^X_{\realm}(\lifta)$ and $\lifta$ is also applicable in $s$ according to \realm.
% \roni{Edited the proof above.}
\end{proof}

% It is straightforward to show that if actions are affine functions of the pre-state, their effects will also be accurately learned. 
% \paragraph{Proof Sketch}\label{sec:learning-effects-safe}
% For each action $a\in A$, \nsam learns its effects by applying a linear regression mapping the observations in $DB_{\pre}(\lifta)$ to their respected result in $DB_{post}(\lifta)$.
% Since we assume that calculations do not introduce numeric error, applying linear regression is equivalent to solving a set of affine equations.
% Consequently, the number of observations in $DB_{\pre}(\lifta)$ is at least than $|X(\lifta)|+1$, resulting in a a single solution $\psi$, which is guaranteed to be the correct solution for the set of equations.
% \argaman{Need to see that this aligns with what we said about the effects in the previous section.}

\subsubsection{Runtime}
The runtime of \nsam depends on the following factors: 
% \argaman{Mention that these statements are relevant to linear domains and how they change on polynomial domains. Only a couple of sentences, no numbers.}
\begin{itemize}
    \item The number of lifted actions, denoted $N_A$. 
    \item The number of state transitions in the given trajectories, denoted $N_T$
    \item The number of \pbls that can be bound to each lifted action. We denote by $N_L$ the maximum of these values over all actions, i.e., $N_L=\max_\lifta |L(\lifta)|$.
    \item The number of \pbfs that can be bound to each lifted action. Similarly, $N_X$ denotes the maximum of these values over all actions. 
\end{itemize}
% \roni{@Argaman: refined the def. of $N_X$ and $N_L$}
\begin{theorem}
The runtime complexity of the \nsam algorithm is: 
\[
    O((N_A+N_T)\cdot(N_X+N_L) + N_A\cdot\frac{N_{T}^{\lfloor N_X/2 \rfloor}}{{\lfloor N_X/2 \rfloor}!} + N_A\cdot({N_X}^3\cdot(N_T+N_X))
\]
which is linear in the number of lifted actions, \pbls, and state transitions, 
% polynomial in the number of \pbls,
and exponential in the number of \pbfs. 
\label{thm:nsam-runtime}
\end{theorem}
Notice that the following runtime proof is based on our assumption that the domains are linear (assumptions 2 and 3 in Section~\ref{sec:assumptions}). 
To support polynomial domains, every computation using $N_X$ will increase exponentially with the polynomial degree $d$. 
\begin{proof}
\nsam starts by initializing the 
Boolean preconditions and effects 
and the datasets $DB_{\pre}(\lifta)$ and $DB_{post}(\lifta)$ for each action $\lifta\in A$. 
This initialization incurs $O(N_A\cdot(N_X+N_L))$. %the number of actions, \pbfs $N_X$, and \pbls $N_L$.
Then, for each state transition $\tuple{s,a,s'}$ in the input observations, \nsam applies \sam's inductive rules 
and collects the values in $s$ and in $s'$ of every \pbf bound to $a$. 
Applying \sam inductive rules requires $O(N_L)$. 
Thus, this step incurs $O(N_T\cdot(N_X+N_L))$. 
Next, \nsam computes for each action $\lifta\in A$ the convex hull for the preconditions and linear regression for the effects. 
\citet{barber1996quickhull} showed that the complexity of computing a convex hull, using the Quickhull algorithm, depends on the number of dimensions, which in our case is $N_X$. 
If $2 \leq N_X<4$ the runtime complexity of computing a convex hull is $O(N_{T}\cdot \log (N_{T}))$. 
If $N_X\geq 4$, the runtime complexity grows exponentially to $O(\frac{N_{T}^{\lfloor N_X/2 \rfloor}}{{\lfloor N_X/2 \rfloor}!})$. 
% \roni{Are you sure this is the complexity? I think the part in the denominator is incorrect and should be just one. Anyhow, provide a reference.}\argaman{Fixed the runtime complexity, it is uglier than before (missed the factorial in the formulation) and added the relevant cite.}
% Running linear regression for every \pbf in $DB_{\post}(\lifta)$ has a worst-case runtime complexity of $O({N_X}^2\cdot(N_T+N_X))$.
We denote the matrix defined from the values in the table $DB_{\pre}(\lifta)$ as $X$ and the column relevant to the evaluated \pbf in $DB_{post}(\lifta)$ as $y$.
Calculating least square linear regression is equivalent to performing $(XX^{-1})X'y$~\citep{watson1967linear}. 
Thus, running linear regression for every \pbf in $DB_{post}(\lifta)$ has a worst-case runtime complexity of $O({N_X}^2\cdot(N_T+N_X))$. 
% \roni{Can you cite something here?}\argaman{Done and fixed.}
The runtime complexity of computing all the numeric effects of a single action is $O({N_X}^3\cdot(N_T+N_X))$. 
Thus, this step incurs $O(N_A\cdot({N_X}^3\cdot(N_T+N_X))$.
% \roni{Add text on the Boolean}
Finally, the total runtime complexity of the \nsam algorithm is $O((N_A+N_T)\cdot(N_X+N_L) + N_A\cdot\frac{N_{T}^{\lfloor N_X/2 \rfloor}}{{\lfloor N_X/2 \rfloor}!} + N_A\cdot({N_X}^3\cdot(N_T+N_X))$.
% \roni{I got the following complexity just by summing the above. Please check:}
% $O(N_A\cdot N_L+N_T\cdot N_L+N_A\cdot N_{T} \cdot N_{X}^3+N_{X}^4+N_A\cdot\frac{N_{T}^{\lfloor N_X/2 \rfloor}}{{\lfloor N_X/2 \rfloor}})$
\end{proof}
Theorem~\ref{thm:nsam-runtime} states that the runtime of \nsam is only linear in the number of lifted actions, \pbls, and state transitions,  
but it is exponential in the number of \pbfs bounded to an action. The impact of the number of state transitions on the runtime of \nsam is at most polynomial, where the degree of the polynomial is the number of \pbfs. 
More accurately, the degree of the polynomial is the number of \pbfs that are bound to the actions' preconditions. 
% \argaman{This sentence is incorrect in our context as the number of functions bound to the actions is the same as the number of functions bound to the preconditions. The next sentence is incorrect as well...}
In most IPC domains, this number is very small. In the polynomial IPC domains, this number increases based on the number of monomials (See Table~\ref{tab:numereric-domains} for the exact numbers). %\argaman{Added comment on polynomial domains with reference to the table in the experiments.}
% The runtime of \nsam is polynomial in the number of state transitions, state variables, and actions because computing convex hulls and solving linear regression problems can be done in polynomial time. 
% For polynomial domains, the runtime is polynomial in the number of relevant monomials, which can be exponential in the degree of the polynomial. 

\subsection{Removing Linear Dependencies}
The numeric functions in some domains may be linearly dependent.
For example, consider a bidirectional road and the numeric functions \textit{(distance ?from ?to)}. 
Given two cities, $A$ and $B$, the functions $(distance\; A\; B)$ and $(distance\; B\; A)$ are always equal. 
Another example arises when constant functions maintain the same value across all observations.
In such cases, applying the convex hull directly will fail, as the input becomes rank-deficient, whereas convex hull algorithms assume fully ranked input.
\nsam detects such cases and eliminates linearly dependent functions through linear regression, in a pre-processing step. 
Specifically, for each numeric function $x_i$ we attempt to fit a regression model using the remaining numeric functions as predictors.
Functions identified as linearly dependent are subsequently incorporated as additional equality preconditions and are excluded from the dataset used for convex hull construction.
% \argaman{Added the text about linear dependencies.}

\subsection{Limitations}
% Explain the limitation of NSAM
% \nsam  has a critical limitation that manifests when the number of state transitions available for an action $\lifta$ is small with respect to the number of \pbfs bound to the action $\lifta$, i.e., $|X(\lifta)|$.%\roni{Changed above} 
  % two critical limitations in both the preconditions and effects learning stages. These limitations manifest when the number of observations available for an action $a$ is small with respect to its dimensionality. 
% \paragraph{Limitation in Learning Preconditions}
If the maximal number of affine independent points 
% \roni{non-collinear? or affine independent? or linearly independent?}
in $DB_{\pre}(\lifta)$ is less than $|X(\lifta)| + 1$, constructing a convex hull becomes infeasible, and \nsam cannot learn the preconditions of $\lifta$. 
Furthermore, if the number of linearly independent equations constructed by \nsam (Equation~\ref{eq:sys}) is less than $|X(\lifta)| + 1$, no unique solution exists, and \nsam cannot learn the effects of $\lifta$. 
Consequently, $\lifta$ will not be returned in the learned action model. 
% These restrictions hinder the algorithm's performance since they require many samples to learn actions. This becomes highly noticeable in domains with many numeric variables. [Roni: why denigrate NSAM? recall that NSAM* does not work much better in most cases, and NSAM itself is actually pretty good in many cases.]
These restrictions can hinder the algorithm's performance in domains with many numeric variables or when the available data is very limited or highly correlated.
This raises the question of whether this is an inherent limitation of the problem of learning a safe action model, or a limitation of the \nsam algorithm. 
Below, we show that it is the latter.

\begin{definition}[Strong Action Model]
Action model $M$ is \emph{stronger} than action model $M'$ 
if there exists a state $s$ and an action $a$ such that $a$ is applicable in $s$ according to $M$ but not according to $M'$.  
\label{def:strong}
\end{definition}
\noindent Having a strong action model is beneficial, as it allows solving more problems. 
For discrete domains, \sam was proven to be the strongest 
safe action model that can be learned with a given set of trajectories~\cite{juba2021safe}.\footnote{Technically, the proof in the paper applies to ESAM, but \sam and ESAM are equivalent in domains that satisfy the injective binding assumption.} 
The example below shows that \nsam does not have this property, i.e., it is possible to learn a \emph{stronger} safe action model than the one returned by \nsam using the same given set of trajectories. 
% \roni{Edits above}

\begin{table}[ht]
\centering
\resizebox{0.3\textwidth}{!}{%
\begin{tabular}{@{}ccc@{}}
\toprule
\textbf{(x ?f1)} & \textbf{(x ?f2)} & \textbf{(cost )} \\ \midrule
2 & 0 & 1 \\
1 & 0 & 1 \\
11 & 0 & 0 \\ \bottomrule
\end{tabular}%
}
\caption{Example of $DB_{\pre}(\moveslow)$ for the action \moveslow\ in Farmland domain.}
\label{tab:example-dataset}
\end{table}

% To learn preconditions, \nsam creates a convex hull from the samples in DB$_\pre(a)$.  
% % To create a convex hull from a set of $S=\{v_1,v_2,...v_m\}$ examples such that $\forall 0\leq i\leq m: v_i \in\mathbb{R}^n$, 
% To this end, \nsam requires at least $X(a)+1$ linearly independent samples. 
% If $DB_{\pre(a)}$ contains fewer linearly independent samples, \nsam cannot create a convex hull for the input examples. 
% Similarly, without at least $X(a)$ linearly independent samples \nsam cannot find a single solution to the system of equations specified for learning effects. 
% Each of these reasons is sufficient for \nsam to decide not to include $a$ in the returned action model. 

% Example where NSAM fails
% \paragraph{Example of \nsam's limitation}

% \roni{Important: we are using three terms: linear independent, affine independent, and non-collinear. It is not clear when which term is needed. Are you sure and can explain this?}\argaman{non-colinearity is redundant as it is generalized by affine independence. Regarding the preconditions we use affine independence as the original points are shifted from the center axis. affine-independence also generalize linear independence to when the lines are formed as $y=ax+b$ so I think we can use affine independence throughout the paper but I am not confident enough to change it throughout the paper. @Shahaf and @Brendan what do you think?}

\begin{example}
    Consider the \moveslow\ action in the Farmland domain described in Example~\ref{ex:moveslow}, and assume \nsam is given trajectories where this action has been executed three times, resulting in $DB_{\pre}(\moveslow)$ containing the values displayed in Table~\ref{tab:example-dataset}. 
    \nsam requires a minimum of four affine-independent observations to learn the convex hull of the action's preconditions as $|X(\moveslow)|=3$.
    Since these observations include only three affine-independent observations, \nsam cannot learn the preconditions of the action and thus deems it unsafe.     
    Trivially, \moveslow\ can be safely applied in the observed states $\tuple{2,0,1}$, $\tuple{1,0,1}$, and $\tuple{11,0,0}$. 
    Moreover, we can identify additional sets of points in which \moveslow\ can be safely applied since the action's preconditions are conjunctions of linear inequalities. 
    Specifically, since the preconditions form a convex set, the line segment connecting any pair of points in which \moveslow\ was applied represents a state in which \moveslow is applicable (assuming the Boolean part is applicable, of course). 
    For example, \moveslow\ can be safely executed in the state $\tuple{1.5,0,1}$, as it lies along the line segment between the observed states $\tuple{2,0,1}$ and $\tuple{1,0,1}$. 
\label{ex:nsam-not-strong}
\end{example}
% \roni{Edits above, please read}\argaman{Wrote response above and made minor grammatical changes.}

\section{The \algname Algorithm}

In this section, we introduce the \algname algorithm, an improved version of \nsam that overcomes the limitation described in Example~\ref{ex:nsam-not-strong}. 
It returns a safe action model that is stronger than \nsam and includes every observed action, even if it was only observed once. 

% Considering that the domain has eight lifted functions, \nsam would require to observe at least  linearly independent examples of the action's execution to consider the action as safe to learn, i.e., prior to observing \textit{fly} 36 independent times the action will not be learned.
% % \shahaf{what about the limitations on preconditions? (I.e., having to learn a complete convex hull)}
% \shahaf{is this because the number of monomials in a polynomial with $n$ variables and a degree of $k$ is $n-1+k \choose n-1$? if so, we should explicitly state that.}

% We overcome these limitations by introducing an improved approach to learning numeric action models that includes in the returned action model every action it has observed, even if it was only observed once. 
% initiates the learning process from the first observation.

% Learning the effects: simple, just explain that it's OK even if we don't have enough equations

%\subsection{Projection to a Lower Dimension}
%than $|X(a)| + 1$, constructing a convex hull becomes infeasible, and \nsam cannot learn the preconditions of $a$. Furthermore, if the number of linearly independent equations constructed by \nsam (Equation~\ref{eq:sys}) is less than $|X(a)| + 1$, no unique solution exists, and \nsam cannot learn the effects of $a$. 

\subsection{Overview}
\algname resolves the problem of not having enough linearly-independent observations of an action by \emph{projecting} the available observations into a smaller dimensional space in which a convex hull of these observations can be created. 
% Let $m$ be the number of observations in $DB_{\pre}(\lifta)$. Even if $m<|X(\lifta)|+1$, it is still possible to project the observations in $DB_{\pre}(\lifta)$ into an $(m-1)$-dimensional space, facilitating the construction of the convex hull and ensuring that for any state within that space, any solution of the effects' system of equations yields the same effects. \roni{This is too formal and detailed for an overview}
We illustrate this concept using Figure~\ref{fig:GS-example}. 
% Three observations are given: $\mathbb{R}^3: (1,0,0), (0,1,0), (0,0,1)$. Since at least four points are needed to construct a convex hull in $\mathbb{R}^3$, \nsam cannot learn an action model given these observations. 
Figure~\ref{fig:GS-example-pre-change} shows three 3D points representing three states where an action $a$ has been observed. 
3 points are not enough to construct a convex hull in a 3D space, so \algname projects these points to a 2D plane and computes a convex hull there, as shown in Figure~\ref{fig:GS-example-lower}.  
% As shown in Figure~\ref{fig:GS-example-lower}, these three observations can be projected to two-dimensional space, in which a convex hull (triangle) can be constructed.
% Since the numeric preconditions are conjunctions of linear inequalities, we know that every state that lies in the subspace spanned by $DB_{\pre}(a)$ and is within the convex hull created in that subspace must be applicable in the true action model. The \algname algorithm we introduce next follows this rationale. 

\begin{figure}[ht]
     \centering
     \begin{subfigure}[b]{0.45\columnwidth}
         \centering
         \includegraphics[width=\textwidth]{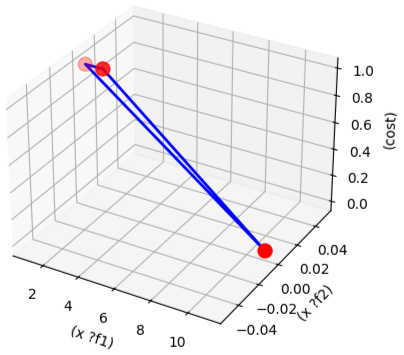}
         \caption{}
         \label{fig:GS-example-pre-change}
     \end{subfigure}
     \hfill
     \begin{subfigure}[b]{0.45\columnwidth}
         \centering
         \includegraphics[width=\textwidth]{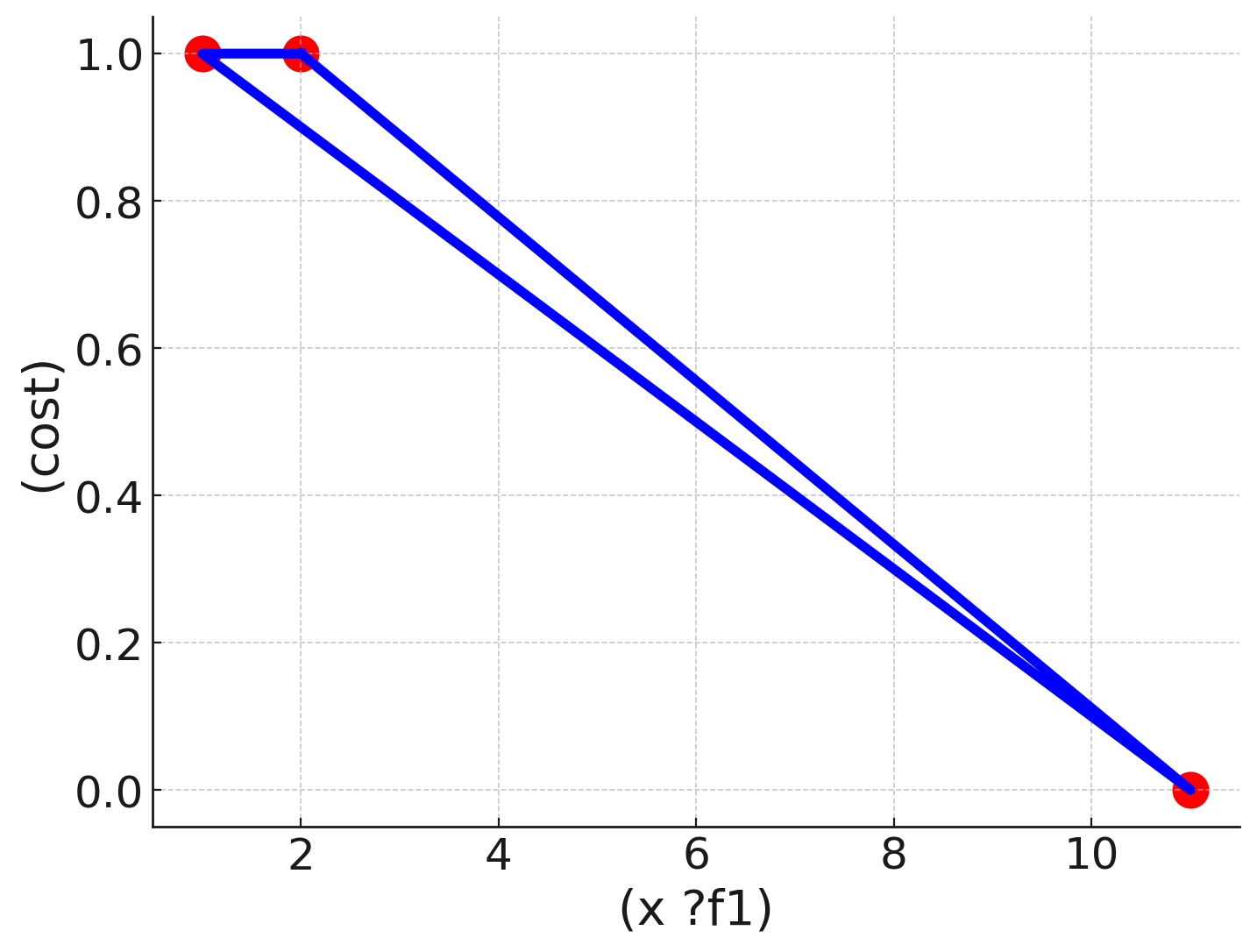}
         \caption{}
         \label{fig:GS-example-lower}
     \end{subfigure}     
     \caption{Graphical illustration of the observed points $(2, 0, 1), (1, 0, 1), (11, 0, 0)$ from Table~\ref{tab:example-dataset}. In Figure~\ref{fig:GS-example-pre-change}, the points are displayed on the $\mathbb{R}^3$ space, and in Figure~\ref{fig:GS-example-lower}, they are displayed as projections on a 2D plane with (x ?f2) being constant zero.}
     \label{fig:GS-example}
\end{figure}

% \subsection{Algorithmic Details}
%we introduce \algname, an extension of \nsam designed to address scenarios where the number of independent observations is too little to learn action models in the original dimension. REPETITION
Algorithm~\ref{alg:fnsam-algorithm} lists the pseudo-code of \algname. 
\algname starts by running \nsam (line~\ref{alg:apply-nsam}), which populates $\pre$, $\eff$ and the data structures $DB_{\pre}(\lifta)$ and $DB_{post}(\lifta)$ for every lifted action $\lifta$.\footnote{Technically, \nsam only returns the preconditions and effects, but in \algname we also require the 
$DB_{\pre}$ and $DB_{post}$ data structures that \nsam initializes.} 
% First, it learns the preconditions and effects of the Boolean variables,  using the \sam Learning algorithm~\cite{juba2021safe}. 

For every lifted action $\lifta$ in $A_{\textit{unsafe}}$ 
it extracts the first observation vector, i.e., $v_0$, and shifts the vectors in $DB_{\pre}(\lifta)$ according to $v_0$ (line~\ref{subtract_first_sample}). 
The resulting set of vectors is called $Shifted_{v0}$. 
Then, \algname creates an orthonormal basis for the subspace spanned by the vectors 
in $Shifted_{v0}$.
The resulting basis is called $Basis$. 
Note that the dimension of the subspace spanned by $Basis$ may be smaller than $|X(\lifta)|+1$. 
Next, \algname projects all the vectors in $Shifted_{v0}$ to this subspace, and computes the convex hull of these projected vectors (line~\ref{projection}). 
The resulting convex hull is called $CH_{\textit{proj}}$.  
To compute the numeric preconditions of $\lifta$, \algname uses $CH_{\textit{proj}}$ 
and an orthonormal basis for the subspace that includes all points except those spanned by $DB_{\pre}(\lifta)$ (lines~\ref{find_base_func}-\ref{orthonormal_base}). 
The resulting numeric preconditions allow applying $\lifta$ in states where the relevant functions correspond to a point that is (1) in $span(Base)$, and (2) its projection onto $span(Base)$ is within $CH_{\textit{proj}}$. 
The numeric effects of $\lifta$ are learned using any linear regression method, as in \nsam, but allowing it to return solutions even if they are not unique (line~\ref{linear-regression-fnsam}). 
That is, any regression solution with an $R^2$ score of 1, can be used. %indicating an exact solution to the system, as the effects for its learned actions.

The preconditions learning part of \algname is not trivial. Thus, we provide below a more detailed explanation of this process. % on how preconditions are learned by \algname. 

\begin{algorithm}[tb]
\small
\caption{The \algname Algorithm}\label{alg:fnsam-algorithm}
\begin{algorithmic}[1]
% \scriptsize
\State \textbf{Input}: the observed trajectories $\mathcal{T}$ 
\State \textbf{Output}: a safe action model.
\State $\pre$, $\eff$, $A_{\textit{unsafe}}$, $DB_{pre}$, $DB_{post}$ $\gets$ NSAM($\mathcal{T}$) \label{alg:apply-nsam}
%Apply \nsam to initialize the datasets and iterate over $\mathcal{T}$ \label{alg:apply-nsam}
% \For{$\lifta \in A$} 
    % \State $\pre^{F},\; \eff^{F}  \gets $ Apply \sam learning 
    % \State $DB_{pre}(a) \gets $ numeric states in $\mathcal{O}$ where $a$ was applied .
    % \State $DB_{post}(a) \gets $ numeric states in $\mathcal{O}$ after $a$ was applied.
    % \State Initialize $\pre^{F}(\lifta),\; \eff^{F}(\lifta),\; DB_{pre}(\lifta), \; DB_{post}(\lifta)$ the same as \nsam
% \EndFor
\For{$\lifta \in A_{\textit{unsafe}}$} 
    \State $v_0 \gets \{s_0(x)|x\in X(\lifta)\}$ from $DB_{\pre}(\lifta)$ \Comment{Get first observation vector} \label{alg:get-first-observation}
    \State $Shifted_{v0} \gets \{v_i - v_0 | \forall v_i \in DB_{\pre}(\lifta)\}$\label{subtract_first_sample}
    \State $Basis \gets$ FindBasis$(Shifted_{v0}, \emptyset)$\label{find_base_func}
    \State $Projected\gets \{project(s,Basis) | s\in Shifted_{v0}\}$ \label{projection}
    \State $CH_{proj} \gets $ ConvexHull$(Projected)$\label{ch}
    \State $CompBasis \gets$ FindBasis$(\mathbb{R}^{|X(a)|}, Basis)$ \label{orthonormal_base}
    \State $\pre^{X}(\lifta) \gets $ CreatePre($CH_{proj}$, $Basis$, $CompBasis$) %\roni{Not very clear}
    \State $\eff^{X}(\lifta) \gets $ LinearRegression$(DB_{pre}(\lifta),DB_{post}(\lifta))$\label{linear-regression-fnsam}
\EndFor
\State \Return  $(\pre,\eff)$
\end{algorithmic}
\end{algorithm}

\subsection{Learning Numeric Preconditions}
\label{sec:numeric-preconditions}

Learning the preconditions in \algname computes two sets of vectors, $Basis$, and $CompBasis$. 
The former is an orthonormal basis of the subspace spanned by the vectors in $DB_{\pre}$, and the latter is an orthonormal basis of the complementing subspace (called $CompBasis$ in Alg.~\ref{alg:fnsam-algorithm}). 
To find these bases, we first introduce the $FindBasis$ auxiliary function, which is based on the classical Gram-Schmidt Process (GSP)~\citep{leon2013gram}. 
GSP transforms a set of linearly independent vectors into an orthonormal set. 
$FindBasis$ receives $points$ and $basisVecs$, a set of vectors which may be linearly dependent, and a set of vectors $basisVecs$ respectively.
The function returns an orthonormal set of vectors $NVec$ such that $NVec \cup basisVecs$ spans the space defined by $points$. 
% To relax GSP's requirement for linear independence, $FindBase$ includes a check that considers only points whose projections relative to the other points in the set exceed a small threshold $\epsilon$ (line \ref{epsilon-diff}).

% Original
% Employing dimensionality reduction for learning preconditions and effects involves finding a basis for the lower dimensional subspace, projecting existing observations to this subspace, 
% and determining whether a new observation lies in that subspace. 
% \roni{I don't think ``determining whether a new observation ...'' is part of the projection. I rephrase and edits}
% To perform these steps, we first introduce the following auxiliary function $FindBase$, based on the classical Gram-Schmidt process~\citep{leon2013gram} and listed in Algorithm~\ref{alg:gs-algorithm}.

\begin{algorithm}[htb]
% \small
\caption{Find Orthonormal Basis}\label{alg:gs-algorithm}
\begin{algorithmic}[1]
% \scriptsize
\State {\bf Function} FindBasis($points, basisVecs$)
% \State \textbf{Input}: $points, basisVecs$
\State \textbf{Output}: Orthonormal basis $NVec$
\State $Vec \gets basisVecs$\label{init-find-base}
\State $NVec \gets \emptyset$ 
\For{$p \in points$}
    \State $p_{proj} \gets p - \sum_{v\in Vec} \frac{p\times v}{(\norm{v}{2})^2}\cdot v$
% \If{$|p_{proj}| > \epsilon$} \Comment{$\epsilon$ controls numeric accuracy}\label{epsilon-diff}
\If{$|p_{proj}| > 0$} \label{epsilon-diff} 
    \State $Vec \gets Vec \cup \{p_{proj}\}$
    \State $NVec \gets NVec \cup \{\frac{p_{proj}}{\norm{p_{proj}}{2}}\}$ 
\EndIf
\EndFor
% \State $B' \gets NVec \setminus Base$ \Comment{set difference}
\State \Return $NVec$
\end{algorithmic}
\end{algorithm}

The pseudo-code for $FindBasis$ is listed in Algorithm~\ref{alg:gs-algorithm}. 
The variable \( Vec \) is initialized with \( basisVecs \) and is iteratively expanded with vectors orthogonal to all previously added vectors until \( Vec \) spans the space of the input points. $FindBasis$ iterates over every input vector \( p \in points \), applying the same vector manipulation as GSP. 
If the projection of \( p \) onto the intermediate basis ($Vec$) is not a zero vector, we add it to $Vec$. 
Otherwise, it means $p$ is already spanned by $Vec$, and we do not add it to $Vec$. 
The variable \( NVec \) represents the orthonormal vectors generated from the vectors added to \( Vec \) (excluding \( basisVec \)) by applying \( L_2 \)-normalization. 
This ensures \( NVec \) forms an orthonormal set orthogonal to \( basisVec \). 
Note that since \( Vec \) spans all the vectors in $points$ (due to the GSP), the union \( NVec \cup basisVec\) also spans them. %\roni{Not clear: what is "all points" and what is "the entire set of points"?}

To learn the numeric preconditions of an action $\lifta$, 
\algname computes $Basis$ by calling $FindBasis(DB_{\pre}(\lifta, \emptyset))$ 
and then computes $CompBasis$ by calling $FindBasis(DB_{pre}(\mathbb{R}^{|X(\lifta)|}, Basis)$. 
Then, it projects $DB_{pre}(\lifta)$ to the subspace spanned by $Basis$ and
computes the convex hull $CH_{proj}$ of the projected vectors (line~\ref{ch}). 
Next, \algname creates preconditions for $\lifta$ that ensure it is applicable only in states that (1) are spanned by $Basis$, 
and (2) their projections on the subspace spanned by $Basis$ are within $CH_{proj}$. 
To achieve (1), we create a PDDL precondition that verifies the dot product of a state $s$ with every vector in $CompBasis$ is zero. 
To achieve (2), we add a PDDL precondition for every facet of $CH_{proj}$. 

\begin{example}
    In our example, $v_0=\tuple{2,0,1}$. Subtracting $v_0$ from the points in $DB_{\pre}(\moveslow)$, results in the matrix
    \[
    Shifted_{v0}=\begin{bmatrix}
    0 & 0 & 0 \\
    -1 & 0 & 0 \\
    9 & 0 & -1
    \end{bmatrix}
\]
    The orthonormal set returned from the $FindBasis$ is:
\[ Basis = \begin{bmatrix}
    -1 & 0 & 0 \\
    0 & 0 & -1 
    \end{bmatrix}
\]
and $CompBasis=[0,1,0]$.
The result of $\{project(s,Basis) | s\in Shifted_{v0}\}$ is a 2-D matrix with new variables, denoted $x''$ and $y''$ as follows:
\[
\begin{bmatrix}
    0 & 0 \\
    1 & 0 \\
    -9 & 1
\end{bmatrix}
\]
Finally, the inequalities representing the facets of the convex hull $CH_{proj}$ (line~\ref{projection}) are as follows:
$-y'' \leq 0$, $-0.11x'' -0.99y'' \leq 0$ and $0.10x''+0.99y'' \leq 0.10$. 
\end{example}
\label{example_projection_and_calculation}

\subsection{Translating the Preconditions to PDDL}

% \begin{algorithm}[htb]
% % \small
% \caption{Create PDDL Preconditions}\label{alg:gs-algorithm}
% \begin{algorithmic}[1]
% % \scriptsize
% \State {\bf Function} CreatePre($Funcs,v_0, CH_{proj}, Basis, CompBasis$)
% \State $Conds \gets \emptyset$
% \State $Subtractions \gets \emptyset$
% \For{$x_i \in Funcs$}
%     \State $Subtractions \gets Subtractions \cup$ $ (-\;x_i\;\; v_0[i])$  
% \EndFor
% \For{$NVec \in Basis$}
%     \State $Products \gets \emptyset$
%     \For{$coeff, sub \in \tuple{NVec, Subtractions}$}
%         \State $Products \gets Products\; \cup$ $ (*\;sub\;\; coeff)$
%     \EndFor
%     \State Construct
% \EndFor
% \State \Return $NVec$
% \end{algorithmic}
% \end{algorithm}
The conditions derived in Example~\ref{example_projection_and_calculation} are not initially formatted in a way that can be parsed by a planning algorithm. 
To generate the preconditions in PDDL format, \algname first constructs strings representing the subtraction of each numeric function from its corresponding value in $v_0$. 
For the \moveslow\ action, \algname generates the strings ``\texttt{(- (x ?f1) 2)}'', ``\texttt{(x ?f2)}'', and ``\texttt{(- (cost) 1)}''.

Subsequently, \algname constructs strings that represent the functions projected onto $Basis$. 
This step involves multiplying each function by the appropriate coefficient from the corresponding vector in $Basis$ and forming a linear combination. 
In our example, the PDDL representation of $Basis$ consists of the two vectors ``\texttt{(* (- (x ?f1) 2) -1)}'' and ``\texttt{(* (- (cost) 1) -1)}''. 
The complementary basis, $CompBasis$, is incorporated as additional equality conditions --- ``\texttt{(= (x ?f2) 0)}''.

Next, \algname generates the inequalities defining the convex hull hyperplanes. 
Recall, in Example~\ref{example_projection_and_calculation} \algname created the inequalities $-y'' \leq 0$, $-0.11x'' - 0.99y'' \leq 0$, and $0.10x'' + 0.99y'' \leq 0.10$.
These inequalities are translated into the following PDDL strings:
``\texttt{(<= (* (* (- (cost) 1) -1) -1) 0)}'', 
``\texttt{(<= (+ (* (* (- (x ?f1) 2) -1) -0.11) (* (* (- (cost) 1) -1) -0.99)) 0)}'', and
``\texttt{(<= (+ (* (* (- (x ?f1) 2) -1) 0.1) (* (* (- (cost) 1) -1) 0.99)) 0.10)}'', respectively.

The final output of the PDDL translation of the preconditions is presented in Figure~\ref{list:farmland-move-slow-preconditions}.
% \argaman{Added a better explanation of the translation process for reproducibility}
% \roni{Add pseudo code for this step}
% The result of the FindBase is a coefficient matrix that multiplies the observation points to create the lower dimensionality transformed points. 
% Creating the actions' preconditions in PDDL requires the projected vectors to be expressed in the original vectors' space. 
% Given the vector $v_0$ representing the first sample in $DB_{pre}(\lifta)$, the algorithm begins by subtracting every numeric function $x_i$ from the corresponding value $v_0[i]$ (to match line~\ref{subtract_first_sample}). 
% Then, each shifted function is multiplied by the relevant element in $Basis$, resulting in a set of factored shifted functions for each row in the GSP matrix.
% Finally, a linear combination is created from all the components mentioned above.
% The result is the renamed functions in the original space that represent their projected values.

% \begin{example}
%     Following our Farmland example, the numeric function (x ?f1) would be subtracted by two resulting in the PDDL expression (- (x ?f1) 2).
%     Similarly, the cost would be subtracted by one resulting in the PDDL expression (- (cost ) 1).
%     The PDDL representation of the convex hull inequalities and the conditions defined by $CompBasis$ is presented in Figure~\ref{list:farmland-move-slow-preconditions}.
% \end{example}

\begin{figure}[ht]
\begin{center}
\centering
\begingroup
    \fontsize{10pt}{10pt}\selectfont
    \begin{Verbatim}[commandchars=\\\{\}]
            (and (<= (* (* (- (cost ) 1) -1) -1) 0)
            (<=(+ (* (* (- (x ?f1) 2) -1) -0.11) 
                (* (* (- (cost ) 1) -1) -0.99)) 0)
            (<= (+ (* (* (- (x ?f1) 2) -1) 0.1) 
                (* (* (- (cost ) 1) -1) 0.99)) 0.1)
            (= (x ?f2) 0)) 
    \end{Verbatim}
\endgroup
\vspace{-0.3cm}
\caption{The learned numeric preconditions for \moveslow .}
\vspace{-0.5cm}
\label{list:farmland-move-slow-preconditions}
\end{center}
\end{figure}

Next, we provide proofs for the algorithm's theoretical properties.

\subsection{Correctness}
First, we prove that the algorithm will run correctly on any valid input. This is not trivial because the number of projected vectors $Projected$ may not be enough to compute a convex hull in the subspace spanned by $Basis$.  

A necessary and sufficient condition for computing the convex hull of a set \( D \) of \( d \)-dimensional points is that \( D \) must contain at least \( d+1 \) affine independent points (denoted as \( D' \)). I.e., no point in \( D' \) can be expressed as an affine combination of the other points. 
Thus, it is sufficient to show that $Projected$ includes at least $dim(span(Projected))+1$ vectors. 
This is true by construction since it includes the zero vector and additional $|Basis|$ linearly-independent vectors.

\subsection{Safety}

Next, we prove that the preconditions and effects learned by \algname are safe. 
% From Theorem~\ref{preconds-from-convex-safe} we know that an action $a$ is applicable from a state $s$ that lies within the convex hull of $a$. 
% Consequently, if a state $s'$, resulting from applying the aforementioned projection of a state $s$, resides within the convex hull, then action $a$ is applicable on $s'$. 
% However, it is essential to note that the projection is guaranteed to be valid only for states within the span of $Basis$. 

% Next, we show that the actions learned using \algname are safe to apply.
\begin{lemma} \label{th:effects-safe}
    For each lifted action, $\lifta\in A_{\textit{unsafe}}$ constructed in Algorithm~\ref{alg:fnsam-algorithm}, 
    the preconditions, $\pre(\lifta)$, are guaranteed to be safe, 
    and for each applicable state $s$, the effects $\eff(\lifta)$ are guaranteed to be the exact effects as observed in the real unknown action model.
\begin{proof}
    Given an action $\lifta$, recall that $X(\lifta)$ is the set of \pbfs bound to $\lifta$.
    % If $DB_{\pre}(\lifta)$ contains $m=|X(\lifta)|+1$ affine independent observations, then the action \algname learns will be identical to the one learned by \nsam.
    % Otherwise, $DB_{\pre}(\lifta)$ contains $m<|X(\lifta)|+1$ linearly independent vectors.
    Since $\pre^{X}(\lifta)$ contains equality constraints, we know that any point permitted according to the preconditions is in the linear subspace of the original points.
    Thus, it is sufficient to prove that the projected preconditions are safe in the lower-dimensional space, but this results directly from the preconditions being a convex hull.

    We denote $\psi$ as one of the infinite number of solutions to the linear equations selected by \algname, and $\psi^*$ as the real function defining the effects.
    Since $\psi$ is a solution to the set of equations defined by $DB_{\pre}(\lifta)$ then $\forall v_i\in DB_{\pre}(\lifta): f(v_i)=f^*(v_i)$.
    Let $s''$ be a state such that \algname deems applicable.
    Since $s''$ is spanned by the base created from the original points, there exists a $W\in\mathbb{R}^{|X(\lifta)|}$ such that
    $W\times DB_{\pre}(\lifta)=s''$ and $f(s'')=f^*(W\times DB_{\pre}(\lifta))=f^*(s'')$. 

    \noindent Overall, \algname ensures the acquisition of safe preconditions and effects, generating action models that are safe.
    % First, we prove that the preconditions are safe.
    % The preconditions, i.e., $\pre^{X}(a)$, contain the convex hull projected on the lower dimensionality base $Basis$ and the equality conditions created from $CompBasis$.
    % Since $\pre(a)$ contains equality constraints, we know that any point permitted according to the preconditions is in the linear subspace of the original points.
    % Thus, it is sufficient to prove that the projected preconditions are $\epsilon$-safe in the lower dimension space, but this results directly from the preconditions being a convex hull.
    
    % Finally, we prove that the effects are also $\epsilon$-safe.
    % We denote $f$ as one of the infinite number of solutions to the linear equations selected by \algname, and $f^*$ as the real function defining the effects.
    % Since $f$ is a solution to the set of equations defined by $\hat{S}$ then $\forall s_i\in \hat{S}: f(s_i)=f^*(s_i)$.
    % Let $s'$ be a state such that \algname deems applicable.
    % Since $s'$ is spanned by the base created from the original points, exists a $W$ such that $W\times\hat{S}=s'$. Then:
    % \begin{align*}
    %     f(s')=f(W\times\hat{S})=f(\sum_{i}^{k}w_i\cdot s_i)=\sum_{i}^{k}w_i\cdot f(s_i)=\sum_{i}^{k}w_i\cdot f^*(s_i) \\
    %     =f^*(W\times\hat{S})=f^*(s')
    % \end{align*}
    % We note that the third transition is valid since $f$ and $f^*$ are linear functions. 
\end{proof}
\end{lemma}

\subsection{Runtime}
Next, we provide runtime analysis of the \algname algorithm. 
Recall, $N_A,N_T,N_X,N_L$, are the number of lifted actions, state transitions, and the maximum number of \pbfs and \pbls bounded to an action across all actions.
\begin{theorem}
The runtime complexity of the \algname algorithm is:
\[
    O((N_A+N_T)\cdot(N_X+N_L) + N_A\cdot N_T \cdot N_X^2 + N_A\cdot\frac{N_{T}^{\lfloor N_X/2 \rfloor}}{{\lfloor N_X/2 \rfloor}!} + N_A\cdot({N_X}^3\cdot(N_T+N_X))
\]
which is linear in the number of lifted actions, \pbls, and state transitions, 
% polynomial in the number of \pbfs, \pbls, and state transitions,
and exponential in the number of \pbfs. 
% $O(N_T\cdot N_L + N_A \cdot \frac{N_{T}^{\lfloor N_X/2 \rfloor}}{{\lfloor N_X/2 \rfloor}})$ which is linear in the number of \pbls and state transitions, polynomial in the number of actions and state transitions, and exponential in the number of \pbfs. 
\end{theorem}
\begin{proof}
    Applying the \nsam algorithm has a runtime complexity of $O((N_A+N_T)\cdot(N_X+N_L) + N_A\cdot\frac{N_{T}^{\lfloor N_X/2 \rfloor}}{{\lfloor N_X/2 \rfloor}!} + N_A\cdot({N_X}^3\cdot(N_T+N_X))$. 
    For each lifted action $\lifta \in A_{\textit{unsafe}}$ \algname performs the following:
    It shifts all the vectors in $DB_{\pre}(\lifta)$ according to the first observation $v_0$.
    The complexity of shifting the observations in $DB_{\pre}(\lifta)$ is $O(N_X \cdot N_T)$.
    Next, it applies $FindBasis$ to find $Basis$ and $CompBasis$. 
    The runtime complexity of applying $FindBasis$ is the same runtime complexity of applying GSP, i.e., $O(N_T \cdot N_X^2)$.
    % $O(N_{T}^{\lfloor N_X/2 \rfloor})$
    Applying the convex hull algorithm on the projected points is exponential in the number of dimensions of $Basis$.
    In the worst case $Basis$ contains $|X(\alpha)|$ vectors, 
    resulting in a runtime complexity of $O(\frac{N_{T}^{\lfloor N_X/2 \rfloor}}{{\lfloor N_X/2 \rfloor}!})$.
    Converting the preconditions to their PDDL format involves subtracting each \pbf by the matching value in $v_0$ and then creating the linear combinations using $Basis$.
    The runtime complexity for this process is $O(N_X+N_X\cdot N_X\cdot N_T)=O(N_X^2\cdot N_T)$.
    Finally, running linear regression has the same runtime complexity as before --- $O({N_X}^3\cdot(N_T+N_X))$.
    Thus, the total runtime complexity of \algname is $O((N_A+N_T)\cdot(N_X+N_L) + N_A\cdot N_T \cdot N_X^2 + N_A\cdot\frac{N_{T}^{\lfloor N_X/2 \rfloor}}{{\lfloor N_X/2 \rfloor}!} + N_A\cdot({N_X}^3\cdot(N_T+N_X))$
\end{proof}
% Overall, \algname ensures the acquisition of $\epsilon$-safe preconditions and effects, generating action models that are $\epsilon$-safe.

\subsection{Optimality}

We now show that \algname is \emph{optimal} by establishing that no alternative method for learning \textit{safe} numeric action models 
can return a stronger action model (according to Definition~\ref{def:strong}) than the one returned by \algname (given the same input). 
% can declare an action applicable at a certain state if \algname deems it inapplicable. 

% Definition~\ref{def:strong}
% Definition~\ref{def:safe_action_model}

% Strength
\begin{theorem}[Optimality] \label{th:opt}
    Let $\hat{M}$ be the model generated by \algname given the set of trajectories $\mathcal{T}$, and let $\realm$ be the real action model. 
    For every model $M$ that is safe w.r.t. $\realm$ created by the same set of trajectories $\mathcal{T}$, 
    and every state $s$ and grounded action $a$, if $a$ is applicable in $s$ according to $M$ then it is also applicable in $s$ according to $\hat{M}$.     
\end{theorem}
\begin{proof}
Let $a^*$ be an action that is not applicable in a state $s^*$ according to $\hat{M}$. 
We will show that if $a^*$ is applicable in $s^*$ according to another action model $M$, then $M$ is not safe, from which the theorem follows immediately. 
Let $v_s$ be the vector $v_s=\{s(x)|x\in s\}$, 
and let $c$ be the numeric expression representing a precondition of the action $a^*$. 
Consequently, $c(v_s)$ is the result of assigning every numeric function its appropriate value according to $v_s$.
If $a^*$ is not applicable in $s^*$ according to $\hat{M}$, then $s^*$ violates one of the preconditions of $a^*$:
if it is an equality precondition, i.e., $c(v_{s^*}) = 0$ then then either $c(v_{s^*})>0$ or $c(v_{s^*})<0$.
But notice, \algname only adds the equality precondition if $c(v_{s'}) = 0$ for all $s'$ appearing in the example trajectories.
Thus, there exists a safe action model, consistent with the input trajectories, in which $a^*$ has the precondition $c(v_s)\leq 0$ or $c(v_s)\geq 0$.
Thus, $a$ is not applicable in some state $s^*$ for a possible action model; hence, $M$ is unsafe. 

Similarly, if we have $c(v_{s^*})>0$ for one of the convex hull's facets $c(v_s) \leq 0$, then since all of the states $s'$ appearing in example state transitions satisfied $c(v_{s'}) \leq 0$ (or else the convex hull would not have a facet $c(v_s)\leq 0$) there is an action model in which $a^*$ has the precondition $c(v_s)\leq 0$ that is consistent with all of the example trajectories. Therefore, again, $M$ is not safe.
\end{proof}

\section{Experimental Results}

We conducted an experimental evaluation to evaluate the performance of \algname and \nsam on a diverse set of numeric planning domains. The code and the datasets are available at \url{https://github.com/SPL-BGU/numeric-sam}.  

\subsection{Baseline and Domains}

As a baseline, we compared our algorithms to \pmn~\citep{segura2021discovering}. 
We used the version of \pmn supplied by the authors and publicly available at \url{https://github.com/Leontes/PlanMiner}.
We observed that the runtime of \pmn can be prohibitively long, and thus we restricted the learning process to end within 10 minutes. 
If, for a given experiment iteration the process did not end within the restricted time or resulted with an error, we treated the iteration as if no model was learned.

We experimented on the following domains: Depot, Driverlog (two versions, one with linear and one with polynomial preconditions and effects, denoted as ``L'' and ``P'', respectively), Rovers, Satellite, and Zenotravel from the 3rd International Planning Competition (IPC3)~\citep{long20033rd}; Farmland and Sailing, from Scala et al.~\cite{scala2016heuristics}; Counters, from IPC 2023~\cite{scala2020subgoaling,taitler20242023}; and  Minecraft from Benyamin et al.~\citep{benyamin2023model}. 
Both Driverlog-P and Zenotravel domains contain preconditions and effects that are polynomial combinations of \pbfs. 

Table~\ref{tab:numereric-domains} provides information about the experimented domains. 
The ``Domain'' column represents the name of the experimented domain, and columns ``$|A|$'', ``$|A_X|$'', ``$|F|$'', and ``$|X|$'' represent the total number of lifted actions, the number of lifted actions with numeric preconditions and effects, and the number of lifted fluents and numeric functions in the domains, respectively. 
The columns ``$avg$ $|\pre^{X}|$'' and ``$avg$ $|\eff^{X}|$'' represent the average number of numeric preconditions and effects per action in the domain. 
The numbers in the parentheses represent the maximal and minimal number of numeric preconditions and effects in an action, respectively.
The columns ``$max$ \pbls'' and ``$max$ \pbfs'' represent the maximal number of \pbls and \pbfs in an action in the domain (this is relevant to both preconditions and effects). 
In the polynomial domain, these values represent the number of parameter-bound monomials up to the degree of the domain's polynomial. 
For example, in the Zenotravel domain, the maximal number of \pbfs in an action in the domain is 9, thus the number of parameter-bound monomials is $9 + {9\choose2} + 9=54$. 

To calculate the convex hull, we used an implementation of the Quickhull~\cite{barber1996quickhull} algorithm called QHull~\cite{barber2011qhull}. 
Qhull is restricted in the number of dimensions its data is permitted to have. Specifically, the algorithm cannot learn convex hulls for shapes with more than eight dimensions.\footnote{Qhull issues --- \url{http://www.qhull.org/html/qh-code.htm#performance}} 
This occurred only in
the Driverlog-P and Zenotravel domains, which include more than eight parameter-bound monomials. 
To address this, we provided \algname and \nsam the set of monomials used in each action (in both the preconditions and the effects). 
We did not provide their composition structure, i.e., their sign or the coefficients multiplying the monomials.
We denote this domain knowledge as Relevant-Functions (RF). The maximal number of monomials used in the learning process after using the RF is denoted in parentheses in the column $max$ \pbfs of Table~\ref{tab:numereric-domains}.
% \roni{Wait, didn't you only use this for Driverlog-P and Zeno travel? then why are there valuesin parenthesis for all the other domains?}
% \argaman{There are not values in parenthesis for the other domains in terms of linear VS polynomial. I added additional statistics that count the max and min number of preconditions and effects in the actions in the domains. This was a statistic that was helpfull in understanding the results shown in planMiner and why it sometimes behaved well while other times it crashed.}
% Table~\ref{tab:numereric-domains} provides information about the experimented domains. The `Domain` column represents the name of the experimented domain, and columns $|A|$, $|A_X|$ $|P|$, $|X|$ represent the total number of actions, the number of actions with numeric properties, and, the number of predicates, and numeric fluents in the domains. Finally, the columns $max\; \pre_P$, $max\; \pre^{X}$, $max\; \eff_P$, and $max\; \eff^{X}$ represent the maximal number of parameter bound predicates and monomials relevant to the preconditions or the effects, respectively. 
% We note that we experimented on seven linear domains and two polynomial domains. In addition, the polynomial domains are composed of degree two or lower polynomials.
% \roni{Maybe remove the last sentence?}
% \roni{I suggest explaining everything in this table, including the other parentheses.}
% \argaman{But isn't it explained above? should I add something to the caption?}\roni{You are right, my bad}
\begin{table}[H]
\centering
\resizebox{\textwidth}{!}{%
\begin{tabular}{@{}lcccccccc@{}}
\toprule
\multicolumn{4}{c}{} &  & $avg$ & $avg$ & $max$ & $max$ \\
Domain & $|A|$ & $|A_X|$ & $|F|$ & $|X|$ & \multicolumn{1}{c}{$|\pre^{X}|$} & \multicolumn{1}{c}{$|\eff^{X}|$} & \pbls & \pbfs \\ \midrule
\textbf{Counters} & 4 & 4 & 0 & 3 & 1.0 (1,1) & 1.0 (1,1) & 0 & 3 \\
\textbf{Depot} & 5 & 4 & 6 & 4 & 0.2 (1,0) & 0.8 (1,0) & 16 & 4 \\
\textbf{Driverlog-L} & 6 & 2 & 6 & 4 & 0.0 (0,0) & 0.3 (1,0) & 20 & 6 \\
\textbf{Farmland} & 2 & 2 & 1 & 2 & 1.0 (1,1) & 2.5 (3,2) & 4 & 3 \\
\textbf{Minecraft} & 6 & 6 & 1 & 6 & 1.6 (3,1) & 2.3 (3,2) & 2 & 6 \\
\textbf{Sailing} & 8 & 8 & 1 & 3 & 0.5 (4,0) & 1.3 (2,0) & 2 & 3 \\
\textbf{Satellite} & 5 & 2 & 8 & 6 & 0.4 (1,0) & 0.8 (2,0) & 16 & 6 \\
\textbf{Rovers} & 10 & 9 & 26 & 2 & 0.9 (1,0) & 0.9 (1,0) & 88 & 2 \\ \midrule
\textbf{Driverlog-P} & 6 & 3 & 6 & 4 & 0.0 (0,0) & 0.8 (2,0) & 20 & 20 (3) \\
\textbf{Zenotravel} & 5 & 5 & 2 & 8 & 0.8 (2,0) & 1.4 (2,1) & 6 & 54 (5) \\ \bottomrule
\end{tabular}%
}
\caption{The domains used in the experiments. }
\label{tab:numereric-domains}
\end{table}

% We observe that in Satellite and Driverlog (both linear and polynomial) domains, the number of actions containing numeric properties is at most half of the total number of actions. Consequently, the learning algorithms are oblivious to this and will learn numeric preconditions based on the input observations. 

\subsection{Experimental Setup}
\label{sec:experimental-setup}

We created our dataset of training trajectories and test problems by generating random problems using the IPC problem generator~\citep{seipp-et-al-zenodo2022}. 
We also created our own generators for Farmland, Counters, and Sailing domains. 
Finally, we used the generator created by \citet{benyamin2023model} for the Minecraft domain.
Then, we used state-of-the-art numeric planners, specifically ENHSP~\citep{scala2016heuristics} and Metric-FF~\citep{hoffmann2003metric}, 
to solve the generated problems, and only the solved problems were used as our dataset.
We used a total of 100 problems to create trajectories, which were then used as input in our experiments.
We split the dataset in an 80:20 ratio for the train and test sets. 
This process was repeated five times as a part of a five-fold cross-validation, and all the results were averaged over the five folds. 
% We note that we set the domains to output numbers with four decimal digits accuracy.
All experiments were run on a SLURM CPU cluster with a 32 GB RAM limit. 

\subsection{Evaluation Metrics}
\label{sec:metrics}
We evaluated two aspects of the learned models: correctness and effectiveness. 
Correctness means similarity between the learned model and the real action model, 
while effectiveness means the ability of the learned model to enable finding correct plans. 

\subsubsection{Model Correctness Metrics}
Following prior work on action model learning~\citep{aineto2019learning}, we measured the precision and recall of the Boolean preconditions and effects of the learned action models. 
This is done by measuring the rate of the redundant and missing fluents in the actions' preconditions and effects according to the learned action model, i.e., $\hat{M}$, compared to \realm.
Formally, we define the precision and recall of an action's preconditions as follows:
\begin{align*}
    P^{syn}_{\pre}(\lifta) &= \frac{|\pre_{\realm}(\lifta)\cap \pre_{\hat{M}}(\lifta)|}{|\pre_{\hat{M}}(\lifta)|} \\
    R^{syn}_{\pre}(\lifta) &= \frac{|\pre_{\realm}(a)\cap \pre_{\hat{M}}(\lifta)|}{|\pre_{\realm}(\lifta)|} 
\end{align*}
The precision and recall for the action's effects are calculated in the same way.
If an action $\lifta$ was not learned, we set $P^{syn}_{\pre}(\lifta)=0$ and $R^{syn}_{\pre}(\lifta)=1$ as if it includes every \pbl. 
Similarly, we assume that no effect was learned for $\lifta$, thus $P^{syn}_{\eff}(\lifta)=1$ and $R^{syn}_{\eff}(\lifta)=0$.

% Notice that since the learned action model is safe, then $R^{\pre}(a)$ and $P^{\eff}(a)$ for any action $a$ must always be one.
% In case an action was not learned we set its precision value to zero.
% To measure the total precision and recall for the action model, we average the precision and recall values over all the actions in the domain. 

% \paragraph{Applicability Rate}
Relying solely on these measures may be misleading, as two action models may differ \emph{syntactically} yet still enable solving the same set of problems.  
Specifically, having additional preconditions for an action will reduce an action’s preconditions' precision but may not necessarily hinder its applicability. 
In addition, it is not clear how to correctly measure the precision and recall of \emph{numeric} preconditions and effects. 
To address these limitations, we followed prior work and also measured the \textit{semantic} precision and recall for the preconditions~\citep{mordoch2023learning,le2024learning}. We describe these metrics below and refer to the precision and recall metrics defined above as the \emph{syntactic} precision and recall. 
% \roni{this provides a limited view because ... depends on syntax}

First, we created a dataset that includes our test set trajectories and trajectories created from executing 200 random actions on each test set problem. 
To measure both precision and recall of the preconditions, we ensured that 25\% of these random actions are inapplicable according to \realm. 
Then, for each state transition $\tuple{s,a,s'}$ in these trajectories, we checked if this transition is valid according to the learned action model using VAL~\citep{howey2004val}.
If a transition was deemed applicable according to the learned action model, it was added to $app_{\hat{M}}(\lifta,s)$. Similarly, if a transition was deemed applicable according to the real action model (\realm), then it was added to $app_{\realm}(\lifta, s)$.
Using $app_{\hat{M}}(a)$ and $app_{\realm}(a)$, we measured the \emph{semantic} precision and recall of the learned actions' preconditions, 
denoted $P^{sem}_{\pre}(\lifta)$ and $R^{sem}_{\pre}(\lifta)$, respectively, as follows:
\begin{align*}\small
    P^{sem}_{\pre}(\lifta) &= \sum_{s}\frac{|app_{\realm}(\lifta,s)\cap app_{\hat{M}}(\lifta,s)|}{|app_{\hat{M}}(\lifta,s)|}  \\
    R^{sem}_{\pre}(\lifta) &= \sum_{s}\frac{|app_{\realm}(\lifta,s)\cap app_{\hat{M}}(\lifta,s)|}{|app_{\realm}(\lifta,s)|}
\end{align*}
Similar to the syntactic metrics case, if an action $\lifta$ was not learned, we consider it as if it is always inapplicable, setting $R^{sem}_{\pre}(\lifta)=0$ and $P^{sem}_{\pre}(\lifta)=1$.

The MSE value represents the difference in states resulting from the numeric effects of the actions in the learned domains compared with the respective effects according to \realm.
We measured the MSE of the numeric effects by comparing for every state $s$ and action $a$ in our evaluation dataset, the states $s'_{\hat{M}}=a_{\hat{M}}(s)$ and $s'_{\realm}=a_{\realm}(s)$.
Let $N$ be the number of numeric functions in the state $s'_{\realm}$, i.e., $N=|\{x_i| x_i\in s'_{\realm}\}|$.
The MSE calculated from applying the action $a$ on a state $s$ is:
\[
\text{MSE}(s,a) = \frac{1}{N}\sum_{\{x_i| x_i\in s'_{\realm}\}} (s'_{\hat{M}}(x_i)-s'_{\realm}(x_i))^2
\]
The MSE of an action is the average over all the states in the dataset where $a$ was applied.
To perform a symmetric evaluation, we measured the MSE only on states where both the learned action model and \realm claimed the tested action is applicable. 
Thus, if an action was not learned, its default MSE value is zero.

\subsubsection{Model Effectiveness Metrics}
% \subsubsection{Coverage and Validation}
To assess the effectiveness of the learned model in handling previously unseen scenarios, we used the \emph{coverage} metric.
Coverage is the percentage of test set problems successfully solved using the learned action model.  
To compute this metric, we used a portfolio of two solvers, Metric-FF~\cite{hoffmann2003metric} and ENHSP~\cite{scala2016interval}, both restricted to solving each planning problem within 30 minutes.\footnote{This time restriction is common in planning competitions.}
For each planning problem in the test set, we first tried to solve it using Metric-FF, and if it failed or experienced a timeout, we transitioned to solving it with ENHSP. 
The resulting plans were validated using VAL. 
% We co

Both planners --- Metric-FF and ENHSP --- have many configurations and parameters. 
As we are not interested in optimal planning but in coverage, we configured both planners to run on a satisfiable configuration.
Specifically, in our experiments, we configured ENHSP to run with Greedy Best First Search with the MRP heuristic, helpful actions, and helpful transitions ("sat-hmrphj" configuration).
Metric FF was configured to run with Enforced Hill Climbing (EHC) search strategy and then BFS, 
which is the standard configuration for the planner. 
Both planners, as well as VAL, were configured to support a numeric threshold of 0.1 in their calculations.

\subsection{Results: Model Correctness}

% In this section, we present the evaluation results of the algorithms \nsam, \algname and \pmn.
% Consider first the results for the domain evaluation in terms of precision, recall and MSE and then present the coverage results.
% \subsubsection{Domain Evaluation}

For each algorithm and domain, we report our model correctness metrics as functions of the number of input trajectories until convergence, i.e., until no change was observed in the statistics when given more input trajectories. Additionally, table entries that exhibited no change compared to the preceding number of input trajectories were omitted.

Consider first the results obtained for the syntactic metrics, presented in  
Tables~\ref{tab:driverlog-l_and_minecraft_syntactic}, \ref{tab:sailing-sat-rover_syntactic}, and \ref{tab:depots-farmland-driverlog-p-zeno_syntactic}.
The column $|T|$ indicates the number of trajectories given as input for the learning algorithm.
Columns $P^{syn}_{\pre},R^{syn}_{\pre},P^{syn}_{\eff}$, and $R^{syn}_{\eff}$ represent the syntactic precision and recall of preconditions and effects, respectively.
Since the Counters domain contains no fluents, a syntactic comparison was irrelevant. 
Thus, we discuss this domain only within the semantic evaluation.
Additionally, the \pmn algorithm consistently encountered exceptions in the Minecraft domain, and therefore its results are marked as N/A.

% Compare \pmnr and our algs

% Compare between our algs

We observed that in all domains \algname has $P^{syn}_{\pre}$ and $R^{syn}_{\eff}$ values comparable or better than \nsam. 
For example, in the Sailing domain (Table~\ref{tab:sailing-sat-rover_syntactic}) \algname consistently outperforms \nsam in both $P^{syn}_{\pre}$ and $R^{syn}_{\eff}$ with 13\% and 12\% difference respectively when using 80 input trajectories.
These results are to be expected, as actions that are deemed unsafe according to \nsam are not added to the output action model which lowers the $P^{syn}{\pre}$ and $R^{syn}{\eff}$ values.
Conversely, \algname learns every action it observes, even if the action was observed once.% of $P^{syn}{\pre}$ and $R^{syn}{\eff}$.
% \roni{But the Boolean part is the same. I think this is because it just doesn't learn some actions. This needs to be explained}

Because \algname and \nsam are safe, their $R^{syn}_{\pre}$ and $P^{syn}_{\eff}$ values must be equal to one.
\pmn, however, does not provide safety guarantees, resulting in comparable and in some domains lower $P^{syn}_{\eff}$ and $R^{syn}_{\pre}$ values.
For example, in the Farmland domain (Table~\ref{tab:depots-farmland-driverlog-p-zeno_syntactic}), the $R^{syn}_{\pre}$ value for \pmn are always zero since it constantly failed to learn the Boolean precondition for the actions. When we investigated the low $P^{syn}_{\eff}$ values we noticed that the algorithm learned the wrong effects for the actions it observed, e.g., the \moveslow\ action has \emph{only} numeric effects but \pmn deduced that the action has two Boolean effects as well --- $(adj\; ?param\_0\; ?param\_1)$ and $(adj\; ?param\_1\; ?param\_0)$. 

We observe that in most domains, the syntactic precision of the preconditions ($P^{syn}_{\pre}$) of  \nsam and \algname is lower than \pmn.
For example, in the Driverlog-L domain (Table~\ref{tab:driverlog-l_and_minecraft_syntactic}),
for both \nsam and \algname the $P^{syn}_{\pre}$ value is at most 0.54 whereas \pmn achieves $P^{syn}_{\pre}$ values of 0.86.
Lower $P^{syn}_{\pre}$ values indicate the addition of redundant preconditions, which are not included in the original domain, by \nsam and \algname.
The additional preconditions could be the result of either, 
(1) mutex conditions added due to safety or 
(2) the algorithms did not observe enough trajectories, resulting in preconditions that are too strict.
A prominent example of mutex conditions can be observed in the \textit{driver-truck} action of the Driverlog-L domain.
To drive the truck, the truck needs to be in the location \textit{?loc-from}, i.e., it cannot be in the destination location --- \textit{?loc-to}.
Subsequently, both \nsam and \algname add the precondition $(not\; (at\; ?truck\; ?loc-to))$ which lowers the $P^{syn}_{\pre}$ value.

However, there are domains where \nsam and \algname add additional preconditions that are indeed redundant and affect the actions' applicability. 
The Rovers domain (Table~\ref{tab:sailing-sat-rover_syntactic}) is such an example.
In this domain, using 80 input trajectories, both \nsam and \algname achieved $P^{syn}_{\pre}$ value of 0.54 whereas \pmn had $P^{syn}_{\pre}$ of 0.71. 
This results from the domain having significantly more fluents than other domains, i.e., 26 lifted fluents and 88 \pbls matching some of the domain's actions (see Table~\ref{tab:numereric-domains}). 
Specifically, we observed that the \textit{communicate\_rock\_data} action included redundant preconditions such as $(equipped\_for\_soil\_analysis \; ?r)$. 
This additional literal might prevent a rover that is not equipped to execute the soil analysis from communicating the results, even though it should be able to.
Another relevant example is the \textit{drive-truck} action in the Driverlog-L domain. 
Since \nsam and \algname both observed the action always being executed when the truck is not empty, they both added the literal $(not\; (empty\; ?truck))$ as a precondition. 
This precondition prevents the action from being applied when the truck is empty.

In the Satellite domain (Table~\ref{tab:sailing-sat-rover_syntactic}) with up to 40 trajectories \pmn achieves results comparable to both \nsam and \algname and even outperforms both in $P^{syn}_{\pre}$. 
However, when using 50 trajectories, \pmn's performance drops drastically to zero in both the $P^{syn}_{\pre}$ and $R^{syn}_{\eff}$.
This results from the algorithm being stuck on deducing the relevant \pbfs for the action \textit{take\_image}, and after a ten-minute timeout, the process was terminated. 
We assume that since the action contains five \pbfs it is relatively harder to learn and the algorithm does not scale well when given additional data under these conditions.
In these cases, we assume that the model contains all the \pbls as preconditions and no effects, thus, $P^{syn}_{\pre}$ and $R^{syn}_{\eff}$ are both equal to zero. 

The Minecraft domain (Table~\ref{tab:driverlog-l_and_minecraft_syntactic}) is the only domain where the $P^{syn}_{\pre}$ is always zero. 
This results from the fact that all actions in this domain do not include Boolean preconditions. 
However, all actions are observed when the goal fluent, i.e., $(have\_pogo\_stick)$, is false. 
Thus, both \nsam and \algname incorrectly add the literal $(not\; (have\_pogo\_stick))$ as a precondition, resulting in a $P^{syn}_{\pre}$ value of zero.

% \usepackage{graphicx}
\begin{table}[H]
\centering
\resizebox{\textwidth}{!}{%
\begin{tabular}{ccccccccccc}
\hline
\multicolumn{11}{c}{\textbf{Driverlog-L}} \\ \hline
\textbf{} & \multicolumn{3}{c}{\textbf{$P^{syn}_{\pre}$}} & \multicolumn{2}{c}{\textbf{$P^{syn}_{\eff}$}} & \multicolumn{2}{c}{\textbf{$R^{syn}_{\pre}$}} & \multicolumn{3}{c}{\textbf{$R^{syn}_{\eff}$}} \\ \midrule
\textbf{$|T|$} & \textbf{\nsam} & \textbf{\algname} & \textbf{\pmn} & \textbf{\begin{tabular}[c]{@{}l@{}}\nsam\\ \algname\end{tabular}} & \textbf{\pmn} & \textbf{\begin{tabular}[c]{@{}l@{}}\nsam\\ \algname\end{tabular}} & \textbf{\pmn} & \textbf{\nsam} & \textbf{\algname} & \textbf{\pmn} \\
1 & 0.00 & 0.36 & \textbf{0.50} & 1.00 & 1.00 & 1.00 & 1.00 & 0.00 & \textbf{0.70} & 0.53 \\
2 & 0.05 & \textbf{0.47} & 0.00 & 1.00 & 1.00 & 1.00 & 1.00 & 0.17 & \textbf{0.90} & 0.00 \\
3 & 0.24 & \textbf{0.49} & \textbf{0.52} & 1.00 & 1.00 & 1.00 & 1.00 & 0.53 & \textbf{0.97} & 0.57 \\
4 & 0.32 & \textbf{0.50} & 0.34 & 1.00 & 1.00 & 1.00 & 1.00 & 0.70 & \textbf{0.97} & 0.37 \\
5 & 0.45 & 0.53 & 0.54 & 1.00 & 1.00 & 1.00 & 1.00 & 0.90 & \textbf{1.00} & 0.60 \\
6 & 0.48 & 0.53 & \textbf{0.72} & 1.00 & 1.00 & 1.00 & 1.00 & 0.93 & \textbf{1.00} & 0.80 \\
7 & 0.50 & 0.53 & \textbf{0.72} & 1.00 & 1.00 & 1.00 & 1.00 & 0.97 & \textbf{1.00} & 0.80 \\
8 & 0.50 & 0.53 & \textbf{0.90} & 1.00 & 1.00 & 1.00 & 1.00 & 0.97 & \textbf{1.00} & \textbf{1.00} \\
% 9 & 0.50 & 0.53 & \textbf{0.90} & 1.00 & 1.00 & 1.00 & 1.00 & 0.97 & \textbf{1.00} & \textbf{1.00} \\
% 10 & 0.50 & 0.53 & \textbf{0.90} & 1.00 & 1.00 & 1.00 & 1.00 & 0.97 & \textbf{1.00} & \textbf{1.00} \\
20 & 0.54 & 0.54 & \textbf{0.90} & 1.00 & 1.00 & 1.00 & 1.00 & 1.00 & 1.00 & 1.00 \\
% 30 & 0.54 & 0.54 & \textbf{0.90} & 1.00 & 1.00 & 1.00 & 1.00 & 1.00 & 1.00 & 1.00 \\
% 40 & 0.54 & 0.54 & \textbf{0.90} & 1.00 & 1.00 & 1.00 & 1.00 & 1.00 & 1.00 & 1.00 \\
% 50 & 0.54 & 0.54 & \textbf{0.90} & 1.00 & 1.00 & 1.00 & 1.00 & 1.00 & 1.00 & 1.00 \\
% 60 & 0.54 & 0.54 & \textbf{0.90} & 1.00 & 1.00 & 1.00 & 1.00 & 1.00 & 1.00 & 1.00 \\
% 70 & 0.54 & 0.54 & \textbf{0.90} & 1.00 & 1.00 & 1.00 & 1.00 & 1.00 & 1.00 & 1.00 \\
80 & 0.54 & 0.54 & \textbf{0.72} & 1.00 & 1.00 & 1.00 & 1.00 & \textbf{1.00} & \textbf{1.00} & 0.80 \\ \hline
\multicolumn{11}{c}{\textbf{Minecrat}} \\ \hline
\textbf{} & \multicolumn{3}{c}{\textbf{$P^{syn}_{\pre}$}} & \multicolumn{2}{c}{\textbf{$P^{syn}_{\eff}$}} & \multicolumn{2}{c}{\textbf{$R^{syn}_{\pre}$}} & \multicolumn{3}{c}{\textbf{$R^{syn}_{\eff}$}} \\ \midrule
\textbf{$|T|$} & \textbf{\nsam} & \textbf{\algname} & \textbf{\pmn} & \textbf{\begin{tabular}[c]{@{}l@{}}\nsam\\ \algname\end{tabular}} & \textbf{\pmn} & \textbf{\begin{tabular}[c]{@{}l@{}}\nsam\\ \algname\end{tabular}} & \textbf{\pmn} & \textbf{\nsam} & \textbf{\algname} & \textbf{\pmn} \\
1 & 0.00 & 0.00 & N/A & 1.00  & N/A & 1.00 & N/A & 0.00 & \textbf{0.73} & N/A \\
2 & 0.00 & 0.00 & N/A & 1.00  & N/A & 1.00 & N/A & 0.00 & \textbf{0.77} & N/A \\
% 3 & 0.00 & 0.00 & N/A & 1.00  & N/A & 1.00 & N/A & 0.00 & \textbf{0.83} & N/A \\
4 & 0.00 & 0.00 & N/A & 1.00  & N/A & 1.00 & N/A & 0.00 & \textbf{0.83} & N/A \\
% 5 & 0.00 & 0.00 & N/A & 1.00  & N/A & 1.00 & N/A & 0.00 & \textbf{0.83} & N/A \\
6 & 0.00 & 0.00 & N/A & 1.00  & N/A & 1.00 & N/A & 0.17 & \textbf{0.83} & N/A \\
7 & 0.00 & 0.00 & N/A & 1.00  & N/A & 1.00 & N/A & 0.23 & \textbf{0.83} & N/A \\
8 & 0.00 & 0.00 & N/A & 1.00  & N/A & 1.00 & N/A & 0.47 & \textbf{0.83} & N/A \\
9 & 0.00 & 0.00 & N/A & 1.00  & N/A & 1.00 & N/A & 0.73 & \textbf{0.87} & N/A \\
10 & 0.00 & 0.00 & N/A & 1.00  & N/A & 1.00 & N/A & 0.83 & \textbf{0.90} & N/A \\
20 & 0.00 & 0.00 & N/A & 1.00  & N/A & 1.00 & N/A & 0.83 & \textbf{0.97} & N/A \\
30 & 0.00 & 0.00 & N/A & 1.00  & N/A & 1.00 & N/A & 0.93 & \textbf{1.00} & N/A \\
40 & 0.00 & 0.00 & N/A & 1.00  & N/A & 1.00 & N/A & 0.97 & \textbf{1.00} & N/A \\ \bottomrule
\end{tabular}%
}
\caption{The syntactic precision and recall evaluation results of the actions as a function of the number of trajectories for the Driverlog-L and Minecraft domains.}
\label{tab:driverlog-l_and_minecraft_syntactic}
\end{table}

\begin{table}[H]
\centering
\resizebox{\textwidth}{!}{%
\begin{tabular}{ccccccccccc}
\hline
\multicolumn{11}{c}{\textbf{Sailing}} \\ \hline
\textbf{} & \multicolumn{3}{c}{\textbf{$P^{syn}_{\pre}$}} & \multicolumn{2}{c}{\textbf{$P^{syn}_{\eff}$}} & \multicolumn{2}{c}{\textbf{$R^{syn}_{\pre}$}} & \multicolumn{3}{c}{\textbf{$R^{syn}_{\eff}$}} \\ \midrule
\textbf{$|T|$} & \textbf{\nsam} & \textbf{\algname} & \textbf{\pmn} & \textbf{\begin{tabular}[c]{@{}l@{}}\nsam\\ \algname\end{tabular}} & \textbf{\pmn} & \textbf{\begin{tabular}[c]{@{}l@{}}\nsam\\ \algname\end{tabular}} & \textbf{\pmn} & \textbf{\nsam} & \textbf{\algname} & \textbf{\pmn} \\
1 & 0.30 & 0.33 & \textbf{0.45} & 1.00 & 1.00 & 1.00 & 1.00 & 0.33 & \textbf{0.45} & \textbf{0.45} \\
2 & 0.53 & 0.65 & \textbf{0.78} & 1.00 & 1.00 & 1.00 & 1.00 & 0.65 & \textbf{0.78} & \textbf{0.78} \\
3 & 0.55 & 0.73 & \textbf{0.85} & 1.00 & 1.00 & 1.00 & 1.00 & 0.68 & \textbf{0.85} & \textbf{0.85} \\
4 & 0.60 & 0.73 & \textbf{0.85} & 1.00 & 1.00 & 1.00 & 1.00 & 0.73 & \textbf{0.85} & \textbf{0.85} \\
% 5 & 0.60 & 0.73 & \textbf{0.85} & 1.00 & 1.00 & 1.00 & 1.00 & 0.73 & \textbf{0.85} & \textbf{0.85} \\
6 & 0.60 & 0.78 & \textbf{0.90} & 1.00 & 1.00 & 1.00 & 1.00 & 0.73 & \textbf{0.90} & \textbf{0.90} \\
7 & 0.60 & 0.80 & \textbf{0.93} & 1.00 & 1.00 & 1.00 & 1.00 & 0.73 & \textbf{0.93} & \textbf{0.93} \\
8 & 0.60 & 0.83 & \textbf{0.95} & 1.00 & 1.00 & 1.00 & 1.00 & 0.73 & \textbf{0.95} & \textbf{0.95} \\
9 & 0.63 & 0.83 & \textbf{0.95} & 1.00 & 1.00 & 1.00 & 1.00 & 0.75 & \textbf{0.95} & \textbf{0.95} \\
% 10 & 0.63 & 0.83 & \textbf{0.95} & 1.00 & 1.00 & 1.00 & 1.00 & 0.75 & \textbf{0.95} & \textbf{0.95} \\
20 & 0.63 & 0.88 & \textbf{1.00} & 1.00 & 1.00 & 1.00 & 1.00 & 0.75 & \textbf{1.00} & \textbf{1.00} \\
30 & 0.65 & 0.88 & \textbf{1.00} & 1.00 & 1.00 & 1.00 & 1.00 & 0.78 & \textbf{1.00} & \textbf{1.00} \\
40 & 0.68 & 0.88 & \textbf{1.00} & 1.00 & 1.00 & 1.00 & 1.00 & 0.80 & \textbf{1.00} & \textbf{1.00} \\
50 & 0.70 & 0.88 & \textbf{1.00} & 1.00 & 1.00 & 1.00 & 1.00 & 0.83 & \textbf{1.00} & \textbf{1.00} \\
60 & 0.75 & 0.88 & \textbf{1.00} & 1.00 & 1.00 & 1.00 & 1.00 & 0.88 & \textbf{1.00} & \textbf{1.00} \\ \hline
\multicolumn{11}{c}{\textbf{Satellite}} \\ \hline
\textbf{} & \multicolumn{3}{c}{\textbf{$P^{syn}_{\pre}$}} & \multicolumn{2}{c}{\textbf{$P^{syn}_{\eff}$}} & \multicolumn{2}{c}{\textbf{$R^{syn}_{\pre}$}} & \multicolumn{3}{c}{\textbf{$R^{syn}_{\eff}$}} \\ \midrule
\textbf{$|T|$} & \textbf{\nsam} & \textbf{\algname} & \textbf{\pmn} & \textbf{\begin{tabular}[c]{@{}l@{}}\nsam\\ \algname\end{tabular}} & \textbf{\pmn} & \textbf{\begin{tabular}[c]{@{}l@{}}\nsam\\ \algname\end{tabular}} & \textbf{\pmn} & \textbf{\nsam} & \textbf{\algname} & \textbf{\pmn} \\
1 & 0.16 & 0.57 & \textbf{1.00} & 1.00 & 1.00 & 1.00 & 1.00 & 0.27 & \textbf{0.93} & \textbf{0.93} \\
2 & 0.41 & 0.59 & \textbf{1.00} & 1.00 & 1.00 & 1.00 & 1.00 & 0.67 & \textbf{0.95} & \textbf{0.95} \\
3 & 0.61 & 0.61 & \textbf{1.00} & 1.00 & 1.00 & 1.00 & 1.00 & 0.99 & \textbf{0.99} & \textbf{0.99} \\
% 4 & 0.61 & 0.61 & \textbf{1.00} & 1.00 & 1.00 & 1.00 & 1.00 & 0.99 & \textbf{0.99} & \textbf{0.99} \\
5 & 0.63 & 0.63 & \textbf{1.00} & 1.00 & 1.00 & 1.00 & 1.00 & 1.00 & 1.00 & 1.00 \\
% 6 & 0.63 & 0.63 & \textbf{1.00} & 1.00 & 1.00 & 1.00 & 1.00 & 1.00 & 1.00 & 1.00 \\
% 7 & 0.63 & 0.63 & \textbf{1.00} & 1.00 & 1.00 & 1.00 & 1.00 & 1.00 & 1.00 & 1.00 \\
% 8 & 0.63 & 0.63 & \textbf{1.00} & 1.00 & 1.00 & 1.00 & 1.00 & 1.00 & 1.00 & 1.00 \\
% 9 & 0.63 & 0.63 & \textbf{1.00} & 1.00 & 1.00 & 1.00 & 1.00 & 1.00 & 1.00 & 1.00 \\
% 10 & 0.63 & 0.63 & \textbf{1.00} & 1.00 & 1.00 & 1.00 & 1.00 & 1.00 & 1.00 & 1.00 \\
% 20 & 0.63 & 0.63 & \textbf{1.00} & 1.00 & 1.00 & 1.00 & 1.00 & 1.00 & 1.00 & 1.00 \\
30 & 0.63 & 0.63 & \textbf{0.80} & 1.00 & 1.00 & 1.00 & 1.00 & \textbf{1.00} & \textbf{1.00} & 0.80 \\
40 & \textbf{0.63} & \textbf{0.63} & 0.40 & 1.00 & 1.00 & 1.00 & 1.00 & \textbf{1.00} & \textbf{1.00} & 0.40 \\
50 & \textbf{0.63} & \textbf{0.63} & 0.00 & 1.00 & 1.00 & 1.00 & 1.00 & \textbf{1.00} & \textbf{1.00} & 0.00 \\ \hline
% 60 & \textbf{0.63} & \textbf{0.63} & 0.00 & 1.00 & 1.00 & 1.00 & 1.00 & \textbf{1.00} & \textbf{1.00} & 0.00 \\
% 70 & \textbf{0.63} & \textbf{0.63} & 0.00 & 1.00 & 1.00 & 1.00 & 1.00 & \textbf{1.00} & \textbf{1.00} & 0.00 \\
% 80 & \textbf{0.63} & \textbf{0.63} & 0.00 & 1.00 & 1.00 & 1.00 & 1.00 & \textbf{1.00} & \textbf{1.00} & 0.00 \\ \hline
\multicolumn{11}{c}{\textbf{Rovers}} \\ \hline
\textbf{} & \multicolumn{3}{c}{\textbf{$P^{syn}_{\pre}$}} & \multicolumn{2}{c}{\textbf{$P^{syn}_{\eff}$}} & \multicolumn{2}{c}{\textbf{$R^{syn}_{\pre}$}} & \multicolumn{3}{c}{\textbf{$R^{syn}_{\eff}$}} \\ \midrule
\textbf{$|T|$} & \textbf{\nsam} & \textbf{\algname} & \textbf{\pmn} & \textbf{\begin{tabular}[c]{@{}l@{}}\nsam\\ \algname\end{tabular}} & \textbf{\pmn} & \textbf{\begin{tabular}[c]{@{}l@{}}\nsam\\ \algname\end{tabular}} & \textbf{\pmn} & \textbf{\nsam} & \textbf{\algname} & \textbf{\pmn} \\
1 & 0.01 & \textbf{0.22} & \textbf{0.00} & 1.00 & 1.00 & 1.00 & 1.00 & 0.02 & \textbf{0.64} & 0.00 \\
2 & 0.03 & \textbf{0.29} & \textbf{0.11} & 1.00 & 1.00 & \textbf{1.00} & 1.00 & 0.08 & \textbf{0.72} & 0.15 \\
3 & 0.05 & 0.30 & \textbf{0.44} & 1.00 & 1.00 & \textbf{1.00} & 0.99 & 0.14 & \textbf{0.72} & 0.61 \\
4 & 0.16 & 0.34 & \textbf{0.57} & 1.00 & 1.00 & \textbf{1.00} & 0.99 & 0.37 & 0.75 & \textbf{0.76} \\
5 & 0.22 & 0.36 & \textbf{0.59} & \textbf{1.00} & 0.98 & \textbf{1.00} & 0.97 & 0.45 & 0.75 & \textbf{0.76} \\
6 & 0.29 & 0.40 & \textbf{0.61} & \textbf{1.00} & 0.94 & \textbf{1.00} & 0.94 & 0.54 & \textbf{0.76} & \textbf{0.76} \\
7 & 0.30 & 0.42 & \textbf{0.64} & \textbf{1.00} & 0.98 & \textbf{1.00} & 0.98 & 0.54 & \textbf{0.76} & \textbf{0.76} \\
8 & 0.31 & 0.43 & \textbf{0.65} & \textbf{1.00} & 0.96 & \textbf{1.00} & 0.96 & 0.55 & \textbf{0.76} & \textbf{0.76} \\
9 & 0.32 & 0.43 & \textbf{0.65} & \textbf{1.00} & 0.94 & \textbf{1.00} & 0.94 & 0.55 & \textbf{0.76} & \textbf{0.76} \\
10 & 0.32 & 0.43 & \textbf{0.65} & \textbf{1.00} & 0.94 & \textbf{1.00} & 0.94 & 0.56 & \textbf{0.76} & \textbf{0.76} \\
20 & 0.41 & 0.48 & \textbf{0.70} & \textbf{1.00} & 0.98 & \textbf{1.00} & 0.98 & 0.63 & \textbf{0.76} & \textbf{0.76} \\
30 & 0.46 & 0.49 & \textbf{0.68} & \textbf{1.00} & 0.94 & \textbf{1.00} & 0.94 & 0.72 & \textbf{0.76} & \textbf{0.76} \\
40 & 0.51 & 0.51 & \textbf{0.69} & \textbf{1.00} & 0.94 & \textbf{1.00} & 0.94 & 0.76 & 0.76 & 0.76 \\
50 & 0.54 & 0.54 & \textbf{0.69} & \textbf{1.00} & 0.94 & \textbf{1.00} & 0.94 & 0.76 & 0.76 & 0.76 \\
60 & 0.54 & 0.54 & \textbf{0.73} & 1.00 & 1.00 & 1.00 & 1.00 & 0.76 & 0.76 & 0.76 \\
70 & 0.54 & 0.54 & \textbf{0.70} & \textbf{1.00} & 0.94 & \textbf{1.00} & 0.94 & 0.76 & 0.76 & 0.76 \\
80 & 0.55 & 0.55 & \textbf{0.72} & \textbf{1.00} & 0.98 & \textbf{1.00} & 0.98 & 0.76 & 0.76 & 0.76 \\ \bottomrule
\end{tabular}%
}
\caption{The syntactic precision and recall evaluation results of the actions as a function of the number of trajectories for the Sailing, Satellite, and Rovers domains.}
\label{tab:sailing-sat-rover_syntactic}
\end{table}

\begin{table}[H]
\centering
\resizebox{\textwidth}{!}{%
\begin{tabular}{ccccccccccc}
\hline
\multicolumn{11}{c}{\textbf{Depots}} \\ \hline
\textbf{} & \multicolumn{3}{c}{\textbf{$P^{syn}_{\pre}$}} & \multicolumn{2}{c}{\textbf{$P^{syn}_{\eff}$}} & \multicolumn{2}{c}{\textbf{$R^{syn}_{\pre}$}} & \multicolumn{3}{c}{\textbf{$R^{syn}_{\eff}$}} \\ \midrule
\textbf{$|T|$} & \textbf{\nsam} & \textbf{\algname} & \textbf{\pmn} & \textbf{\begin{tabular}[c]{@{}l@{}}\nsam\\ \algname\end{tabular}} & \textbf{\pmn} & \textbf{\begin{tabular}[c]{@{}l@{}}\nsam\\ \algname\end{tabular}} & \textbf{\pmn} & \textbf{\nsam} & \textbf{\algname} & \textbf{\pmn} \\
1 & 0.53 & 0.53 & \textbf{0.97} & 1.00 & 1.00 & 1.00 & 1.00 & 1.00 & 1.00 & 1.00 \\ \hline
\multicolumn{11}{c}{\textbf{Farmland}} \\ \hline
\textbf{} & \multicolumn{3}{c}{\textbf{$P^{syn}_{\pre}$}} & \multicolumn{2}{c}{\textbf{$P^{syn}_{\eff}$}} & \multicolumn{2}{c}{\textbf{$R^{syn}_{\pre}$}} & \multicolumn{3}{c}{\textbf{$R^{syn}_{\eff}$}} \\ \midrule
\textbf{$|T|$} & \textbf{\nsam} & \textbf{\algname} & \textbf{\pmn} & \textbf{\begin{tabular}[c]{@{}l@{}}\nsam\\ \algname\end{tabular}} & \textbf{\pmn} & \textbf{\begin{tabular}[c]{@{}l@{}}\nsam\\ \algname\end{tabular}} & \textbf{\pmn} & \textbf{\nsam} & \textbf{\algname} & \textbf{\pmn} \\
1 & 0.00 & \textbf{0.25} & 0.00 & \textbf{1.00} & 0.50 & \textbf{1.00} & 0.50 & 0.00 & \textbf{0.50} & \textbf{0.50} \\
2 & 0.25 & \textbf{0.30} & 0.00 & \textbf{1.00} & 0.40 & \textbf{1.00} & 0.40 & 0.50 & \textbf{0.60} & \textbf{0.60} \\
3 & 0.30 & \textbf{0.35} & 0.00 & \textbf{1.00} & 0.30 & \textbf{1.00} & 0.30 & 0.60 & \textbf{0.70} & \textbf{0.70} \\
4 & 0.30 & \textbf{0.50} & 0.00 & \textbf{1.00} & 0.00 & \textbf{1.00} & 0.00 & 0.60 & \textbf{1.00} & \textbf{1.00} \\
5 & 0.40 & \textbf{0.50} & 0.00 & \textbf{1.00} & 0.00 & \textbf{1.00} & 0.00 & 0.80 & \textbf{1.00} & \textbf{1.00} \\
6 & 0.45 & \textbf{0.50} & 0.00 & \textbf{1.00} & 0.00 & \textbf{1.00} & 0.00 & 0.90 & \textbf{1.00} & \textbf{1.00} \\
7 & \textbf{0.50} & \textbf{0.50} & 0.00 & \textbf{1.00} & 0.00 & \textbf{1.00} & 0.00 & 1.00 & 1.00 & 1.00 \\ \hline
\multicolumn{11}{c}{\textbf{Driverlog-P}} \\ \hline
\textbf{} & \multicolumn{3}{c}{\textbf{$P^{syn}_{\pre}$}} & \multicolumn{2}{c}{\textbf{$P^{syn}_{\eff}$}} & \multicolumn{2}{c}{\textbf{$R^{syn}_{\pre}$}} & \multicolumn{3}{c}{\textbf{$R^{syn}_{\eff}$}} \\ \midrule
\textbf{$|T|$} & \textbf{\nsam} & \textbf{\algname} & \textbf{\pmn} & \textbf{\begin{tabular}[c]{@{}l@{}}\nsam\\ \algname\end{tabular}} & \textbf{\pmn} & \textbf{\begin{tabular}[c]{@{}l@{}}\nsam\\ \algname\end{tabular}} & \textbf{\pmn} & \textbf{\nsam} & \textbf{\algname} & \textbf{\pmn} \\
1 & 0.06 & 0.09 & \textbf{0.17} & 1.00 & 1.00 & 1.00 & 1.00 & 0.17 & \textbf{0.23} & \textbf{0.23} \\
2 & 0.06 & 0.32 & \textbf{0.54} & 1.00 & 1.00 & 1.00 & 1.00 & 0.17 & \textbf{0.63} & \textbf{0.63} \\
3 & 0.06 & 0.49 & \textbf{0.66} & 1.00 & 1.00 & 1.00 & 1.00 & 0.17 & \textbf{0.93} & 0.73 \\
4 & 0.11 & 0.54 & \textbf{0.90} & 1.00 & 1.00 & 1.00 & 1.00 & 0.23 & \textbf{1.00} & \textbf{1.00} \\
5 & 0.33 & 0.54 & \textbf{0.90} & 1.00 & 1.00 & 1.00 & 1.00 & 0.57 & \textbf{1.00} & \textbf{1.00} \\
6 & 0.45 & 0.54 & \textbf{0.90} & 1.00 & 1.00 & 1.00 & 1.00 & 0.80 & \textbf{1.00} & \textbf{1.00} \\
% 7 & 0.49 & 0.54 & \textbf{0.90} & 1.00 & 1.00 & 1.00 & 1.00 & 0.93 & \textbf{1.00} & \textbf{1.00} \\
8 & 0.49 & 0.54 & \textbf{0.90} & 1.00 & 1.00 & 1.00 & 1.00 & 0.93 & \textbf{1.00} & \textbf{1.00} \\
9 & 0.54 & 0.54 & \textbf{0.72} & 1.00 & 1.00 & 1.00 & 1.00 & \textbf{1.00} & \textbf{1.00} & 0.80 \\
10 & 0.54 & 0.54 & \textbf{0.90} & 1.00 & 1.00 & 1.00 & 1.00 & 1.00 & 1.00 & 1.00 \\
% 20 & 0.54 & 0.54 & \textbf{0.90} & 1.00 & 1.00 & 1.00 & 1.00 & 1.00 & 1.00 & 1.00 \\
% 30 & 0.54 & 0.54 & \textbf{0.90} & 1.00 & 1.00 & 1.00 & 1.00 & 1.00 & 1.00 & 1.00 \\
% 40 & 0.54 & 0.54 & \textbf{0.90} & 1.00 & 1.00 & 1.00 & 1.00 & 1.00 & 1.00 & 1.00 \\
50 & 0.54 & 0.54 & \textbf{0.72} & 1.00 & 1.00 & 1.00 & 1.00 & \textbf{1.00} & \textbf{1.00} & 0.80 \\
60 & 0.54 & 0.54 & \textbf{0.90} & 1.00 & 1.00 & 1.00 & 1.00 & 1.00 & 1.00 & 1.00 \\
% 70 & 0.54 & 0.54 & \textbf{0.90} & 1.00 & 1.00 & 1.00 & 1.00 & 1.00 & 1.00 & 1.00 \\
80 & \textbf{0.54} & \textbf{0.54} & 0.18 & 1.00 & 1.00 & 1.00 & 1.00 & \textbf{1.00} & \textbf{1.00} & 0.20 \\ \hline
\multicolumn{11}{c}{\textbf{Zenotravel}} \\ \hline
\textbf{} & \multicolumn{3}{c}{\textbf{$P^{syn}_{\pre}$}} & \multicolumn{2}{c}{\textbf{$P^{syn}_{\eff}$}} & \multicolumn{2}{c}{\textbf{$R^{syn}_{\pre}$}} & \multicolumn{3}{c}{\textbf{$R^{syn}_{\eff}$}} \\ \midrule
\textbf{$|T|$} & \textbf{\nsam} & \textbf{\algname} & \textbf{\pmn} & \textbf{\begin{tabular}[c]{@{}l@{}}\nsam\\ \algname\end{tabular}} & \textbf{\pmn} & \textbf{\begin{tabular}[c]{@{}l@{}}\nsam\\ \algname\end{tabular}} & \textbf{\pmn} & \textbf{\nsam} & \textbf{\algname} & \textbf{\pmn} \\
1 & 0.00 & 0.57 & \textbf{0.80} & 1.00 & 1.00 & 1.00 & 1.00 & 0.00 & \textbf{0.80} & \textbf{0.80} \\
2 & 0.00 & 0.61 & \textbf{0.88} & 1.00 & 1.00 & 1.00 & 1.00 & 0.00 & \textbf{0.88} & \textbf{0.88} \\
3 & 0.00 & 0.67 & \textbf{1.00} & 1.00 & 1.00 & 1.00 & 1.00 & 0.00 & \textbf{1.00} & \textbf{1.00} \\
% 4 & 0.00 & 0.67 & \textbf{1.00} & 1.00 & 1.00 & 1.00 & 1.00 & 0.00 & \textbf{1.00} & \textbf{1.00} \\
5 & 0.20 & 0.67 & \textbf{1.00} & 1.00 & 1.00 & 1.00 & 1.00 & 0.32 & \textbf{1.00} & \textbf{1.00} \\
6 & 0.33 & 0.67 & \textbf{1.00} & 1.00 & 1.00 & 1.00 & 1.00 & 0.52 & \textbf{1.00} & \textbf{1.00} \\
7 & 0.35 & 0.67 & \textbf{1.00} & 1.00 & 1.00 & 1.00 & 1.00 & 0.56 & \textbf{1.00} & \textbf{1.00} \\
8 & 0.37 & 0.67 & \textbf{1.00} & 1.00 & 1.00 & 1.00 & 1.00 & 0.60 & \textbf{1.00} & \textbf{1.00} \\
9 & 0.41 & 0.67 & \textbf{1.00} & 1.00 & 1.00 & 1.00 & 1.00 & 0.64 & \textbf{1.00} & \textbf{1.00} \\
10 & 0.53 & 0.67 & \textbf{1.00} & 1.00 & 1.00 & 1.00 & 1.00 & 0.76 & \textbf{1.00} & \textbf{1.00} \\
20 & 0.57 & 0.67 & \textbf{1.00} & 1.00 & 1.00 & 1.00 & 1.00 & 0.80 & \textbf{1.00} & \textbf{1.00} \\
% 30 & 0.57 & 0.67 & \textbf{1.00} & 1.00 & 1.00 & 1.00 & 1.00 & 0.80 & \textbf{1.00} & \textbf{1.00} \\
% 40 & 0.57 & 0.67 & \textbf{1.00} & 1.00 & 1.00 & 1.00 & 1.00 & 0.80 & \textbf{1.00} & \textbf{1.00} \\
% 50 & 0.57 & 0.67 & \textbf{1.00} & 1.00 & 1.00 & 1.00 & 1.00 & 0.80 & \textbf{1.00} & \textbf{1.00} \\
% 60 & 0.57 & 0.67 & \textbf{1.00} & 1.00 & 1.00 & 1.00 & 1.00 & 0.80 & \textbf{1.00} & \textbf{1.00} \\
% 70 & 0.57 & 0.67 & \textbf{1.00} & 1.00 & 1.00 & 1.00 & 1.00 & 0.80 & \textbf{1.00} & \textbf{1.00} \\
80 & 0.65 & 0.67 & \textbf{1.00} & 1.00 & 1.00 & 1.00 & 1.00 & 0.96 & \textbf{1.00} & \textbf{1.00} \\ \bottomrule
\end{tabular}%
}
\caption{The syntactic precision and recall evaluation results of the actions as a function of the number of trajectories for the Depots, Farmland, Driverlog-P, and Zenotravel domains.}
\label{tab:depots-farmland-driverlog-p-zeno_syntactic}
\end{table}

Next, consider the semantic precision and recall results shown in Tables~\ref{tab:counters-depots-semantic}-\ref{tab:rovers-driverlog-p-semantic}.
The column $|T|$ represents the number of input trajectories provided to the learning algorithms, and the columns $P^{sem}_{\pre}$, $R^{sem}_{\pre}$, and $MSE$  represent semantic precision and recall of the preconditions, and MSE of the effects, respectively.

Similar to the syntactic evaluation, \algname performs comparably to \nsam, and in certain domains achieves higher $R^{sem}_{\pre}$ values.
For example, in the Rovers domain (Table~\ref{tab:rovers-driverlog-p-semantic}), \algname displays higher $R^{sem}_{\pre}$ values up to 40 input trajectories. 
Specifically, when using fewer than 40 trajectories, \algname achieves $R^{sem}_{\pre}$ values up to 9\% higher than those obtained by \nsam.
This indicates that \algname learned more actions than \nsam, thus increasing their applicability across more states.
Similar trends can be observed in the Zenotravel domain (Table~\ref{tab:sailing-satellite-zenotravel-semantic}) where \algname achieves 56\% higher $R^{sem}_{\pre}$ than \nsam when using 4 input trajectories.

\pmn displays consistently higher $R^{sem}_{\pre}$ values in all the domains it succeeded in learning compared to both \nsam and \algname.
For example, in the Depots domain (Table~\ref{tab:counters-depots-semantic}) after two trajectories, the actions learned using \pmn were applicable in all the evaluated states. 
In contrast, actions learned by \nsam and \algname were applicable in at most 36\% of the evaluated states.
This performance gap arises because \algname and \nsam learn preconditions based on convex hull methods, whereas \pmn learns simpler numeric preconditions involving fewer comparisons.

Because \algname and \nsam provide safety guarantees, their $P^{sem}_{\pre}$ values must always equal to one.
\pmn however, is not safe, resulting in its actions being applicable in states where they are not applicable according to \realm.
In the Farmland domain (Table~\ref{tab:driverlog-l-farmland-minecrat-semantic}), with 80 input trajectories, \pmn's $R^{sem}_{\pre}$ is equal to 0.90, the highest value among the learning algorithms.
However, \pmn achieves a $P^{sem}_{\pre}$ value of 0.57, the lowest among all compared algorithms. 
This indicates that in 43\% of states where actions should not be applicable (according to \realm), \pmn incorrectly considered these actions applicable.
To demonstrate, in the \moveslow\ action, \pmn did not learn the precondition $(>=\; (x\; ?f1)\; 1)$, and in some experiments \pmn could not learn any of the numeric preconditions for the action.

Finally, the MSE of both \nsam and \algname is always one in all the experimented domains. 
This results from the effects being polynomial equations that can be learned using linear regression. 
Conversely, \pmn learns action effects via symbolic regression, an evolutionary algorithm inherently prone to errors.
In some domains, such as the Zenotravel domain (Table~\ref{tab:sailing-satellite-zenotravel-semantic}), the $MSE$ values for \pmn are extremely large, with values surpassing $10^3$.
Specifically, the correct numeric effect for the \textit{refuel} action is $(assign\; (fuel\; ?a)\ (capacity\; ?a))$, whereas \pmn incorrectly learned this effect as $(increase\; (fuel\; ?a)\; (capacity\; ?a))$.
This alteration causes the states resulting from applying the action to differ by the value of $(fuel\; ?a)$ before the change.

\begin{table}[H]
\centering
\resizebox{0.75\textwidth}{!}{%
\begin{tabular}{@{}cccccccc@{}}
\toprule
\multicolumn{8}{c}{\textbf{Counters}} \\ \hline
 & \multicolumn{2}{c}{\textbf{$P^{sem}_{\pre}$}} & \multicolumn{3}{c}{\textbf{$R^{sem}_{\pre}$}} & \multicolumn{2}{c}{\textbf{$MSE$}} \\ \midrule
\textbf{$|T|$} & \textbf{\begin{tabular}[c]{@{}l@{}}\nsam\\ \algname\end{tabular}} & \textbf{\pmn} & \textbf{\nsam} & \textbf{\algname} & \textbf{\pmn} & \textbf{\begin{tabular}[c]{@{}l@{}}\nsam\\ \algname\end{tabular}} & \textbf{\pmn} \\
1 & \textbf{1.00} & 0.99 & 0.01 & 0.01 & \textbf{0.50} & \textbf{0.00} & 0.04 \\
2 & \textbf{1.00} & 0.97 & 0.06 & 0.06 & \textbf{0.75} & \textbf{0.00} & 0.11 \\
3 & \textbf{1.00} & 0.97 & 0.13 & 0.13 & \textbf{0.60} & 0.00 & 0.00 \\
4 & \textbf{1.00} & 0.91 & 0.29 & 0.29 & \textbf{0.85} & 0.00 & 0.00 \\
5 & \textbf{1.00} & 0.92 & 0.31 & 0.32 & \textbf{0.70} & 0.00 & 0.00 \\
6 & \textbf{1.00} & 0.86 & 0.33 & 0.33 & \textbf{0.95} & 0.00 & 0.00 \\
7 & \textbf{1.00} & 0.92 & 0.35 & 0.35 & \textbf{0.55} & 0.00 & 0.00 \\
8 & \textbf{1.00} & 0.89 & 0.37 & 0.37 & \textbf{0.60} & 0.00 & 0.00 \\
9 & \textbf{1.00} & 0.93 & 0.37 & 0.38 & \textbf{0.40} & 0.00 & 0.00 \\
10 & \textbf{1.00} & 0.93 & 0.38 & 0.39 & \textbf{0.40} & 0.00 & 0.00 \\
20 & \textbf{1.00} & 0.94 & 0.46 & 0.47 & \textbf{0.80} & \textbf{0.00} & 0.04 \\
30 & \textbf{1.00} & 0.87 & 0.52 & 0.53 & \textbf{0.95} & 0.00 & 0.00 \\
40 & \textbf{1.00} & 0.84 & 0.57 & 0.57 & \textbf{1.00} & 0.00 & 0.00 \\
50 & \textbf{1.00} & 0.83 & 0.61 & 0.61 & \textbf{1.00} & 0.00 & 0.00 \\
60 & \textbf{1.00} & 0.83 & 0.63 & 0.63 & \textbf{1.00} & 0.00 & 0.00 \\
% 70 & \textbf{1.00} & 0.83 & 0.62 & 0.62 & \textbf{1.00} & 0.00 & 0.00 \\
80 & \textbf{1.00} & 0.83 & 0.64 & 0.64 & \textbf{1.00} & 0.00 & 0.00 \\ \hline
\multicolumn{8}{c}{\textbf{Depots}} \\ \hline
 & \multicolumn{2}{c}{\textbf{$P^{sem}_{\pre}$}} & \multicolumn{3}{c}{\textbf{$R^{sem}_{\pre}$}} & \multicolumn{2}{c}{\textbf{$MSE$}} \\ \midrule
\textbf{$|T|$} & \textbf{\begin{tabular}[c]{@{}l@{}}\nsam\\ \algname\end{tabular}} & \textbf{\pmn} & \textbf{\nsam} & \textbf{\algname} & \textbf{\pmn} & \textbf{\begin{tabular}[c]{@{}l@{}}\nsam\\ \algname\end{tabular}} & \textbf{\pmn} \\
1 & 1.00 & 1.00 & 0.06 & 0.06 & \textbf{0.99} & 0.00 & 0.00 \\
2 & 1.00 & 1.00 & 0.10 & 0.10 & \textbf{1.00} & 0.00 & 0.00 \\
3 & 1.00 & 1.00 & 0.13 & 0.13 & \textbf{1.00} & 0.00 & 0.00 \\
4 & 1.00 & 1.00 & 0.17 & 0.17 & \textbf{1.00} & 0.00 & 0.00 \\
5 & 1.00 & 1.00 & 0.18 & 0.18 & \textbf{1.00} & 0.00 & 0.00 \\
6 & 1.00 & 1.00 & 0.19 & 0.19 & \textbf{1.00} & 0.00 & 0.00 \\
7 & 1.00 & 1.00 & 0.21 & 0.21 & \textbf{1.00} & 0.00 & 0.00 \\
% 8 & 1.00 & 1.00 & 0.21 & 0.21 & \textbf{1.00} & 0.00 & 0.00 \\
9 & 1.00 & 1.00 & 0.22 & 0.22 & \textbf{1.00} & 0.00 & 0.00 \\
10 & 1.00 & 1.00 & 0.23 & 0.23 & \textbf{1.00} & 0.00 & 0.00 \\
20 & 1.00 & 1.00 & 0.27 & 0.27 & \textbf{1.00} & 0.00 & 0.00 \\
30 & 1.00 & 1.00 & 0.30 & 0.30 & \textbf{1.00} & 0.00 & 0.00 \\
40 & 1.00 & 1.00 & 0.32 & 0.32 & \textbf{1.00} & 0.00 & 0.00 \\
50 & 1.00 & 1.00 & 0.34 & 0.34 & \textbf{1.00} & 0.00 & 0.00 \\
60 & 1.00 & 1.00 & 0.35 & 0.35 & \textbf{1.00} & 0.00 & 0.00 \\
70 & 1.00 & 1.00 & 0.36 & 0.36 & \textbf{1.00} & 0.00 & 0.00 \\ \bottomrule
\end{tabular}%
}
\caption{Comparison of the average semantic performance for \nsam, \algname and \pmn for the Counters and Depots domains.}
\label{tab:counters-depots-semantic}
\end{table}

\begin{table}[H]
\centering
\resizebox{0.75\textwidth}{!}{%
\begin{tabular}{@{}cccccccc@{}}
\toprule
\multicolumn{8}{c}{\textbf{Driverlog-L}} \\ \hline
 & \multicolumn{2}{c}{\textbf{$P^{sem}_{\pre}$}} & \multicolumn{3}{c}{\textbf{$R^{sem}_{\pre}$}} & \multicolumn{2}{c}{\textbf{$MSE$}} \\ \midrule
\textbf{$|T|$} & \textbf{\begin{tabular}[c]{@{}l@{}}\nsam\\ \algname\end{tabular}} & \textbf{\pmn} & \textbf{\nsam} & \textbf{\algname} & \textbf{\pmn} & \textbf{\begin{tabular}[c]{@{}l@{}}\nsam\\ \algname\end{tabular}} & \textbf{\pmn} \\
1 & 1.00 & 1.00 & 0.00 & \textbf{0.01} & 0.00 & 0.00 & 0.00 \\
% 2 & 1.00 & 1.00 & 0.00 & \textbf{0.01} & 0.00 & 0.00 & 0.00 \\
3 & 1.00 & 1.00 & 0.00 & 0.01 & \textbf{0.40} & 0.00 & 0.00 \\
4 & 1.00 & 1.00 & 0.01 & 0.02 & \textbf{0.20} & 0.00 & 0.00 \\
5 & 1.00 & 1.00 & 0.02 & 0.03 & \textbf{0.60} & 0.00 & 0.00 \\
6 & 1.00 & 1.00 & 0.03 & 0.04 & \textbf{0.80} & 0.00 & 0.00 \\
7 & 1.00 & 1.00 & 0.04 & 0.05 & \textbf{0.80} & 0.00 & 0.00 \\
8 & 1.00 & 1.00 & 0.05 & 0.05 & \textbf{1.00} & 0.00 & 0.00 \\
9 & 1.00 & 1.00 & 0.05 & 0.06 & \textbf{1.00} & 0.00 & 0.00 \\
10 & 1.00 & 1.00 & 0.06 & 0.06 & \textbf{1.00} & 0.00 & 0.00 \\
20 & 1.00 & 1.00 & 0.10 & 0.10 & \textbf{1.00} & 0.00 & 0.00 \\
30 & 1.00 & 1.00 & 0.15 & 0.15 & \textbf{1.00} & 0.00 & 0.00 \\
40 & 1.00 & 1.00 & 0.23 & 0.23 & \textbf{1.00} & 0.00 & 0.00 \\
50 & 1.00 & 1.00 & 0.35 & 0.35 & \textbf{1.00} & 0.00 & 0.00 \\
60 & 1.00 & 1.00 & 0.40 & 0.40 & \textbf{1.00} & 0.00 & 0.00 \\
70 & 1.00 & 1.00 & 0.45 & 0.45 & \textbf{1.00} & 0.00 & 0.00 \\
80 & \textbf{1.00} & 1.00 & 0.52 & 0.52 & \textbf{0.80} & 0.00 & 0.00 \\ \hline
\multicolumn{8}{c}{\textbf{Farmland}} \\ \hline
 & \multicolumn{2}{c}{\textbf{$P^{sem}_{\pre}$}} & \multicolumn{3}{c}{\textbf{$R^{sem}_{\pre}$}} & \multicolumn{2}{c}{\textbf{$MSE$}} \\ \midrule
\textbf{$|T|$} & \textbf{\begin{tabular}[c]{@{}l@{}}\nsam\\ \algname\end{tabular}} & \textbf{\pmn} & \textbf{\nsam} & \textbf{\algname} & \textbf{\pmn} & \textbf{\begin{tabular}[c]{@{}l@{}}\nsam\\ \algname\end{tabular}} & \textbf{\pmn} \\
1 & \textbf{1.00} & 0.95 & 0.00 & 0.01 & \textbf{0.44} & 0.00 & 0.00 \\
2 & \textbf{1.00} & 0.85 & 0.03 & 0.03 & \textbf{0.60} & 0.00 & 0.00 \\
3 & \textbf{1.00} & 0.83 & 0.07 & 0.07 & \textbf{0.30} & \textbf{0.00} & 0.04 \\
4 & \textbf{1.00} & 0.54 & 0.10 & 0.11 & \textbf{1.00} & 0.00 & 0.00 \\
5 & \textbf{1.00} & 0.54 & 0.16 & 0.16 & \textbf{1.00} & 0.00 & 0.00 \\
6 & \textbf{1.00} & 0.55 & 0.17 & 0.17 & \textbf{0.90} & \textbf{0.00} & 0.01 \\
7 & \textbf{1.00} & 0.55 & 0.19 & 0.19 & \textbf{0.90} & \textbf{0.00} & 0.01 \\
8 & \textbf{1.00} & 0.62 & 0.21 & 0.21 & \textbf{0.90} & \textbf{0.00} & 0.01 \\
9 & \textbf{1.00} & 0.62 & 0.21 & 0.21 & \textbf{0.90} & 0.00 & 0.00 \\
10 & \textbf{1.00} & 0.69 & 0.21 & 0.21 & \textbf{0.80} & \textbf{0.00} & 0.01 \\
20 & \textbf{1.00} & 0.64 & 0.25 & 0.25 & \textbf{0.82} & 0.00 & 0.00 \\
30 & \textbf{1.00} & 0.57 & 0.30 & 0.30 & \textbf{0.92} & 0.00 & 0.00 \\
40 & \textbf{1.00} & 0.57 & 0.32 & 0.32 & \textbf{0.92} & 0.00 & 0.00 \\
50 & \textbf{1.00} & 0.58 & 0.32 & 0.32 & \textbf{0.90} & 0.00 & 0.00 \\
60 & \textbf{1.00} & 0.57 & 0.33 & 0.33 & \textbf{0.90} & 0.00 & 0.00 \\
70 & \textbf{1.00} & 0.57 & 0.36 & 0.36 & \textbf{0.90} & 0.00 & 0.00 \\
80 & \textbf{1.00} & 0.57 & 0.35 & 0.35 & \textbf{0.90} & 0.00 & 0.00 \\ \hline
\multicolumn{8}{c}{\textbf{Minecraft}} \\ \hline
 & \multicolumn{2}{c}{\textbf{$P^{sem}_{\pre}$}} & \multicolumn{3}{c}{\textbf{$R^{sem}_{\pre}$}} & \multicolumn{2}{c}{\textbf{$MSE$}} \\ \midrule
\textbf{$|T|$} & \textbf{\begin{tabular}[c]{@{}l@{}}\nsam\\ \algname\end{tabular}} & \textbf{\pmn} & \textbf{\nsam} & \textbf{\algname} & \textbf{\pmn} & \textbf{\begin{tabular}[c]{@{}l@{}}\nsam\\ \algname\end{tabular}} & \textbf{\pmn} \\
1 & \textbf{1.00} & N/A & 0.00 & 0.00 & N/A & 0.00 & N/A \\
20 & \textbf{1.00} & N/A & 0.01 & 0.01 & N/A & 0.00 & N/A \\
40 & \textbf{1.00} & N/A & 0.02 & 0.02 & N/A & 0.00 & N/A \\
60 & \textbf{1.00} & N/A & 0.03 & 0.03 & N/A & 0.00 & N/A \\ \bottomrule
\end{tabular}%
}
\caption{Comparison of the average semantic performance for \nsam, \algname and \pmn for the Driverlog-L, Farmland, and Minecraft domains.}
\label{tab:driverlog-l-farmland-minecrat-semantic}
\end{table}

\begin{table}[H]
\centering
\resizebox{0.75\textwidth}{!}{%
\begin{tabular}{@{}cccccccc@{}}
\toprule
\multicolumn{8}{c}{\textbf{Sailing}} \\ \hline
 & \multicolumn{2}{c}{\textbf{$P^{sem}_{\pre}$}} & \multicolumn{3}{c}{\textbf{$R^{sem}_{\pre}$}} & \multicolumn{2}{c}{\textbf{$MSE$}} \\ \midrule
\textbf{$|T|$} & \textbf{\begin{tabular}[c]{@{}l@{}}\nsam\\ \algname\end{tabular}} & \textbf{\pmn} & \textbf{\nsam} & \textbf{\algname} & \textbf{\pmn} & \textbf{\begin{tabular}[c]{@{}l@{}}\nsam\\ \algname\end{tabular}} & \textbf{\pmn} \\
1 & \textbf{1.00} & 0.90 & 0.07 & 0.08 & \textbf{0.45} & 0.00 & 0.00 \\
2 & \textbf{1.00} & 0.90 & 0.17 & 0.17 & \textbf{0.77} & 0.00 & 0.00 \\
3 & \textbf{1.00} & 0.90 & 0.20 & 0.20 & \textbf{0.83} & 0.00 & 0.00 \\
4 & \textbf{1.00} & 0.90 & 0.22 & 0.22 & \textbf{0.82} & 0.00 & 0.00 \\
5 & \textbf{1.00} & 0.90 & 0.23 & 0.23 & \textbf{0.82} & 0.00 & 0.00 \\
6 & \textbf{1.00} & 0.90 & 0.24 & 0.25 & \textbf{0.88} & 0.00 & 0.00 \\
7 & \textbf{1.00} & 0.90 & 0.25 & 0.25 & \textbf{0.92} & 0.00 & 0.00 \\
8 & \textbf{1.00} & 0.90 & 0.26 & 0.26 & \textbf{0.95} & 0.00 & 0.00 \\
9 & \textbf{1.00} & 0.90 & 0.26 & 0.26 & \textbf{0.94} & 0.00 & 0.00 \\
10 & \textbf{1.00} & 0.90 & 0.26 & 0.27 & \textbf{0.95} & 0.00 & 0.00 \\
20 & \textbf{1.00} & 0.90 & 0.29 & 0.29 & \textbf{1.00} & 0.00 & 0.00 \\
% 30 & \textbf{1.00} & 0.90 & 0.29 & 0.30 & \textbf{1.00} & 0.00 & 0.00 \\
% 40 & \textbf{1.00} & 0.90 & 0.30 & 0.30 & \textbf{1.00} & 0.00 & 0.00 \\
% 50 & \textbf{1.00} & 0.90 & 0.30 & 0.30 & \textbf{1.00} & 0.00 & 0.00 \\
60 & \textbf{1.00} & 0.90 & 0.31 & 0.31 & \textbf{1.00} & 0.00 & 0.00 \\ \hline
\multicolumn{8}{c}{\textbf{Satellite}} \\ \hline
 & \multicolumn{2}{c}{\textbf{$P^{sem}_{\pre}$}} & \multicolumn{3}{c}{\textbf{$R^{sem}_{\pre}$}} & \multicolumn{2}{c}{\textbf{$MSE$}} \\ \midrule
\textbf{$|T|$} & \textbf{\begin{tabular}[c]{@{}l@{}}\nsam\\ \algname\end{tabular}} & \textbf{\pmn} & \textbf{\nsam} & \textbf{\algname} & \textbf{\pmn} & \textbf{\begin{tabular}[c]{@{}l@{}}\nsam\\ \algname\end{tabular}} & \textbf{\pmn} \\
1 & 1.00 & 1.00 & 0.00 & 0.00 & \textbf{1.00} & 0.00 & 0.00 \\
% 2 & 1.00 & 1.00 & 0.00 & 0.00 & \textbf{1.00} & 0.00 & 0.00 \\
% 3 & 1.00 & 1.00 & 0.00 & 0.00 & \textbf{1.00} & 0.00 & 0.00 \\
4 & 1.00 & 1.00 & 0.01 & 0.01 & \textbf{0.99} & 0.00 & 0.00 \\
5 & 1.00 & 1.00 & 0.02 & 0.02 & \textbf{0.99} & 0.00 & 0.00 \\
% 6 & 1.00 & 1.00 & 0.02 & 0.02 & \textbf{0.99} & 0.00 & 0.00 \\
% 7 & 1.00 & 1.00 & 0.02 & 0.02 & \textbf{0.99} & 0.00 & 0.00 \\
8 & 1.00 & 1.00 & 0.02 & 0.03 & \textbf{1.00} & 0.00 & 0.00 \\
9 & 1.00 & 1.00 & 0.03 & 0.03 & \textbf{1.00} & 0.00 & 0.00 \\
% 10 & 1.00 & 1.00 & 0.03 & 0.03 & \textbf{1.00} & 0.00 & 0.00 \\
20 & 1.00 & 1.00 & 0.06 & 0.07 & \textbf{1.00} & 0.00 & 0.00 \\
30 & 1.00 & 1.00 & 0.08 & 0.08 & \textbf{0.80} & \textbf{0.00} & 5.85 \\
40 & 1.00 & 1.00 & 0.09 & 0.09 & \textbf{0.40} & \textbf{0.00} & 19.43 \\
50 & 1.00 & 1.00 & \textbf{0.10} & \textbf{0.10} & 0.00 & \textbf{0.00} & 31.93 \\
60 & 1.00 & 1.00 & \textbf{0.12} & \textbf{0.12} & 0.00 & \textbf{0.00} & 31.93 \\
% 70 & 1.00 & 1.00 & \textbf{0.12} & \textbf{0.12} & 0.00 & \textbf{0.00} & 31.93 \\
80 & 1.00 & 1.00 & \textbf{0.13} & \textbf{0.13} & 0.00 & \textbf{0.00} & 31.93 \\ \hline
\multicolumn{8}{c}{\textbf{Zenotravel}} \\ \hline
 & \multicolumn{2}{c}{\textbf{$P^{sem}_{\pre}$}} & \multicolumn{3}{c}{\textbf{$R^{sem}_{\pre}$}} & \multicolumn{2}{c}{\textbf{$MSE$}} \\ \midrule
\textbf{$|T|$} & \textbf{\begin{tabular}[c]{@{}l@{}}\nsam\\ \algname\end{tabular}} & \textbf{\pmn} & \textbf{\nsam} & \textbf{\algname} & \textbf{\pmn} & \textbf{\begin{tabular}[c]{@{}l@{}}\nsam\\ \algname\end{tabular}} & \textbf{\pmn} \\
1 & 1.00 & 1.00 & 0.00 & 0.26 & \textbf{0.48} & 0.00 & 14105475.87 \\
2 & 1.00 & 1.00 & 0.00 & 0.50 & \textbf{0.80} & \textbf{0.00} & 28090.52 \\
3 & \textbf{1.00} & 0.99 & 0.00 & 0.53 & \textbf{0.84} & \textbf{0.00} & 34213.81 \\
4 & \textbf{1.00} & 0.99 & 0.00 & 0.56 & \textbf{0.96} & \textbf{0.00} & 26099.49 \\
5 & \textbf{1.00} & 0.99 & 0.25 & 0.56 & \textbf{0.88} & \textbf{0.00} & 31208.53 \\
6 & \textbf{1.00} & 0.99 & 0.41 & 0.57 & \textbf{0.88} & \textbf{0.00} & 30444.20 \\
7 & \textbf{1.00} & 0.99 & 0.43 & 0.58 & \textbf{0.84} & \textbf{0.00} & 32842.21 \\
8 & \textbf{1.00} & 0.99 & 0.44 & 0.58 & \textbf{0.92} & \textbf{0.00} & 26726.17 \\
9 & \textbf{1.00} & 0.99 & 0.48 & 0.59 & \textbf{0.96} & \textbf{0.00} & 24403.07 \\
10 & \textbf{1.00} & 0.98 & 0.57 & 0.59 & \textbf{1.00} & \textbf{0.00} & 22159.24 \\
20 & \textbf{1.00} & 0.98 & 0.62 & 0.62 & \textbf{1.00} & \textbf{0.00} & 22159.24 \\
30 & \textbf{1.00} & 0.99 & 0.62 & 0.62 & \textbf{0.96} & \textbf{0.00} & 23980.00 \\
40 & \textbf{1.00} & 0.98 & 0.63 & 0.63 & \textbf{1.00} & \textbf{0.00} & 22159.24 \\
50 & \textbf{1.00} & 0.99 & 0.64 & 0.64 & \textbf{0.88} & \textbf{0.00} & 28673.23 \\
60 & \textbf{1.00} & 0.98 & 0.64 & 0.64 & \textbf{1.00} & \textbf{0.00} & 22159.24 \\
70 & \textbf{1.00} & 0.99 & 0.64 & 0.65 & \textbf{0.96} & \textbf{0.00} & 24256.44 \\
80 & \textbf{1.00} & 0.99 & 0.65 & 0.65 & \textbf{0.96} & \textbf{0.00} & 24177.78 \\ \bottomrule
\end{tabular}%
}
\caption{Comparison of the average semantic performance for \nsam, \algname and \pmn for the Sailing, Satellite, and Zenotravel domains.}
\label{tab:sailing-satellite-zenotravel-semantic}
\end{table}

\begin{table}[H]
\centering
\resizebox{0.75\textwidth}{!}{%
\begin{tabular}{@{}cccccccc@{}}
\toprule
\multicolumn{8}{c}{\textbf{Rovers}} \\ \hline
 & \multicolumn{2}{c}{\textbf{$P^{sem}_{\pre}$}} & \multicolumn{3}{c}{\textbf{$R^{sem}_{\pre}$}} & \multicolumn{2}{c}{\textbf{$MSE$}} \\ \midrule
\textbf{$|T|$} & \textbf{\begin{tabular}[c]{@{}l@{}}\nsam\\ \algname\end{tabular}} & \textbf{\pmn} & \textbf{\nsam} & \textbf{\algname} & \textbf{\pmn} & \textbf{\begin{tabular}[c]{@{}l@{}}\nsam\\ \algname\end{tabular}} & \textbf{\pmn} \\
1 & 1.00 & 1.00 & 0.00 & \textbf{0.02} & 0.00 & 0.00 & 0.00 \\
2 & 1.00 & 1.00 & 0.01 & 0.08 & \textbf{0.09} & \textbf{0.00} & 0.14 \\
3 & 1.00 & 1.00 & 0.03 & 0.12 & \textbf{0.40} & \textbf{0.00} & 0.14 \\
4 & 1.00 & 1.00 & 0.06 & 0.14 & \textbf{0.55} & \textbf{0.00} & 0.09 \\
5 & 1.00 & 1.00 & 0.10 & 0.17 & \textbf{0.62} & \textbf{0.00} & 0.06 \\
6 & 1.00 & 1.00 & 0.17 & 0.24 & \textbf{0.75} & 0.00 & 0.00 \\
7 & 1.00 & 1.00 & 0.18 & 0.26 & \textbf{0.76} & 0.00 & 0.00 \\
8 & 1.00 & 1.00 & 0.19 & 0.28 & \textbf{0.80} & 0.00 & 0.00 \\
9 & 1.00 & 1.00 & 0.19 & 0.28 & \textbf{0.82} & 0.00 & 0.00 \\
10 & 1.00 & 1.00 & 0.20 & 0.29 & \textbf{0.82} & 0.00 & 0.00 \\
20 & 1.00 & 1.00 & 0.27 & 0.34 & \textbf{0.83} & \textbf{0.00} & 2.33 \\
30 & 1.00 & 1.00 & 0.35 & 0.37 & \textbf{0.84} & \textbf{0.00} & 5.17 \\
40 & 1.00 & 1.00 & 0.39 & 0.39 & \textbf{0.89} & \textbf{0.00} & 6.15 \\
50 & 1.00 & 1.00 & 0.43 & 0.43 & \textbf{0.86} & \textbf{0.00} & 11.03 \\
% 60 & 1.00 & 1.00 & 0.44 & 0.44 & \textbf{0.86} & \textbf{0.00} & 11.37 \\
70 & 1.00 & 1.00 & 0.44 & 0.44 & \textbf{0.87} & \textbf{0.00} & 9.61 \\
80 & 1.00 & 1.00 & 0.44 & 0.44 & \textbf{0.86} & \textbf{0.00} & 10.82 \\ \hline
\multicolumn{8}{c}{\textbf{Driverlog-P}} \\ \hline
 & \multicolumn{2}{c}{\textbf{$P^{sem}_{\pre}$}} & \multicolumn{3}{c}{\textbf{$R^{sem}_{\pre}$}} & \multicolumn{2}{c}{\textbf{$MSE$}} \\ \midrule
\textbf{$|T|$} & \textbf{\begin{tabular}[c]{@{}l@{}}\nsam\\ \algname\end{tabular}} & \textbf{\pmn} & \textbf{\nsam} & \textbf{\algname} & \textbf{\pmn} & \textbf{\begin{tabular}[c]{@{}l@{}}\nsam\\ \algname\end{tabular}} & \textbf{\pmn} \\
1 & 1.00 & 1.00 & 0.17 & \textbf{0.20} & 0.03 & 0.00 & 0.00 \\
2 & 1.00 & 1.00 & 0.17 & \textbf{0.40} & 0.30 & \textbf{0.00} & 0.04 \\
3 & 1.00 & 1.00 & 0.17 & \textbf{0.60} & \textbf{0.47} & \textbf{0.00} & 103.74 \\
4 & 1.00 & 1.00 & 0.23 & 0.64 & \textbf{0.80} & \textbf{0.00} & 56.45 \\
5 & 1.00 & 1.00 & 0.41 & 0.66 & \textbf{0.93} & \textbf{0.00} & 0.02 \\
6 & 1.00 & 1.00 & 0.53 & 0.63 & \textbf{1.00} & \textbf{0.00} & 0.01 \\
7 & 1.00 & 1.00 & 0.59 & 0.65 & \textbf{1.00} & \textbf{0.00} & 0.01 \\
8 & 1.00 & 1.00 & 0.59 & 0.65 & \textbf{1.00} & \textbf{0.00} & 0.01 \\
9 & 1.00 & 1.00 & 0.60 & 0.60 & \textbf{0.77} & \textbf{0.00} & 0.01 \\
10 & 1.00 & 1.00 & 0.60 & 0.60 & \textbf{0.97} & \textbf{0.00} & 0.02 \\
20 & 1.00 & 1.00 & 0.62 & 0.62 & \textbf{1.00} & \textbf{0.00} & 0.01 \\
30 & 1.00 & 1.00 & 0.65 & 0.65 & \textbf{1.00} & \textbf{0.00} & 0.01 \\
40 & 1.00 & 1.00 & 0.69 & 0.69 & \textbf{1.00} & \textbf{0.00} & 0.01 \\
50 & 1.00 & 1.00 & 0.73 & 0.73 & \textbf{0.80} & \textbf{0.00} & 0.01 \\
60 & 1.00 & 1.00 & 0.74 & 0.74 & \textbf{1.00} & \textbf{0.00} & 0.01 \\
% 70 & 1.00 & 1.00 & 0.74 & 0.74 & \textbf{1.00} & \textbf{0.00} & 0.01 \\
80 & 1.00 & 1.00 & \textbf{0.75} & \textbf{0.75} & 0.20 & 0.00 & 0.00 \\ \hline
\end{tabular}%
}
\caption{Comparison of the average semantic performance for \nsam, \algname and \pmn for the Rovers and Driverlog-P domains.}
\label{tab:rovers-driverlog-p-semantic}
\end{table}

\subsection{Results: Model Effectiveness}
Figures~\ref{fig:problem-solved1}–\ref{fig:problem-solved2} show the coverage results,
i.e., the average percentage of test set problems \textit{solved} and \textit{validated}, as a function of the number of input trajectories for \nsam, \algname, and \pmn.
On average, \algname performs as well as \nsam, and in the Rovers and Zenotravel domains, \algname highly outperforms \nsam when given a few trajectories.
In the Rovers domain, given between 10 and 30 input trajectories, domains learned by \algname solved approximately 3\%–16\% more problems than those learned by \nsam.
In the Zenotravel domain, this trend is even more profound: domains learned by \algname solved 75\% of the test set problems, whereas those learned by \nsam failed to solve any.

When comparing \nsam and \algname against \pmn, we observed mixed results.
In the domains Depots, Driverlog (L and P), and Rovers, \pmn mostly outperforms \nsam and \algname.
Conversely, in the Counters, Farmland, Sailing, and Zenotravel domains, \pmn’s coverage rates are significantly lower than those of \algname and \nsam.

In the Counters, Farmland, Sailing, and Rovers domains, the domains learned by \pmn resulted in a high percentage of inapplicable plans.
Table~\ref{tab:inapplicable-plan-rates} presents the average percentage of the inapplicable plans created using the domains learned by \pmn for each of the experimented domains. 
We note that since \pmn could not learn a domain for the Minecraft experiments the domain was omitted. 
In contrast to \pmn, neither \nsam nor \algname produced any inapplicable plans, consistent with their safety guarantees.

In the Zenotravel domain, the planners deemed between 95\%-100\% of the test set problems to be unsolvable using the domain learned by \pmn. 
This can be explained by errors observed in the preconditions and effects learned by the \pmn algorithm (Table~\ref{tab:sailing-satellite-zenotravel-semantic}).

Generally, domains with more complex numeric preconditions and effects resulted in poorer performance for \pmn.
In the Satellite domain, \pmn's learning process timed out when given 50 or more trajectories, and the algorithm could not provide a domain.
Even more severely, in the Minecraft domain --- the only domain featuring more than one numeric \pbf in (see Table~\ref{tab:numereric-domains} columns ``$avg$ $|\pre^{X}|$'' and ``$avg$ $|\eff^{X}|$'') --- \pmn raised exceptions and was unable to learn the action model given any number of trajectories.

\begin{table}[H]
    \centering
    \resizebox{\textwidth}{!}{%
        \begin{tabular}{@{}ccccccccc@{}}
        \toprule
            Domain& Counters & Depots & Driverlog L+P & Farmland & Sailing & Rovers & Satellite  & Zenotravel \\ \midrule
           \%Inapplicable & 31 & 4 & 0 & 75 & 100 & 36 & 1  & 1 \\ \bottomrule
        \end{tabular}%
    }
    \caption{Average percent of the test set problems that were deemed inapplicable for the \pmn algorithm for each of the experimented domains.}
    \label{tab:inapplicable-plan-rates}
\end{table}

% \paragraph{Summary} \algname produces results comparable to \nsam, and notably outperforming it in two domains.
% The comparison between \nsam, \algname, and \pmn revealed mixed results. 
% In domains with fewer numeric preconditions and effects, \pmn exhibited better performance.
% However, in multiple domains \pmn learned actions that were inapplicable according to \realm.
% In Satellite domain, we observed that given more data, both \nsam and \algname's performance improved. 
% However, \pmn could not learn its action model within the time restrictions resulting in a drastic decline in performance.
% Additionally, both \nsam and \algname successfully learned an action model for the Minecraft domain, whereas \pmn encountered errors and failed entirely.

\begin{figure}[H]
     \centering
     \begin{subfigure}[b]{0.44\textwidth}
         \centering
         \caption{Counters}
         \includegraphics[width=\columnwidth]{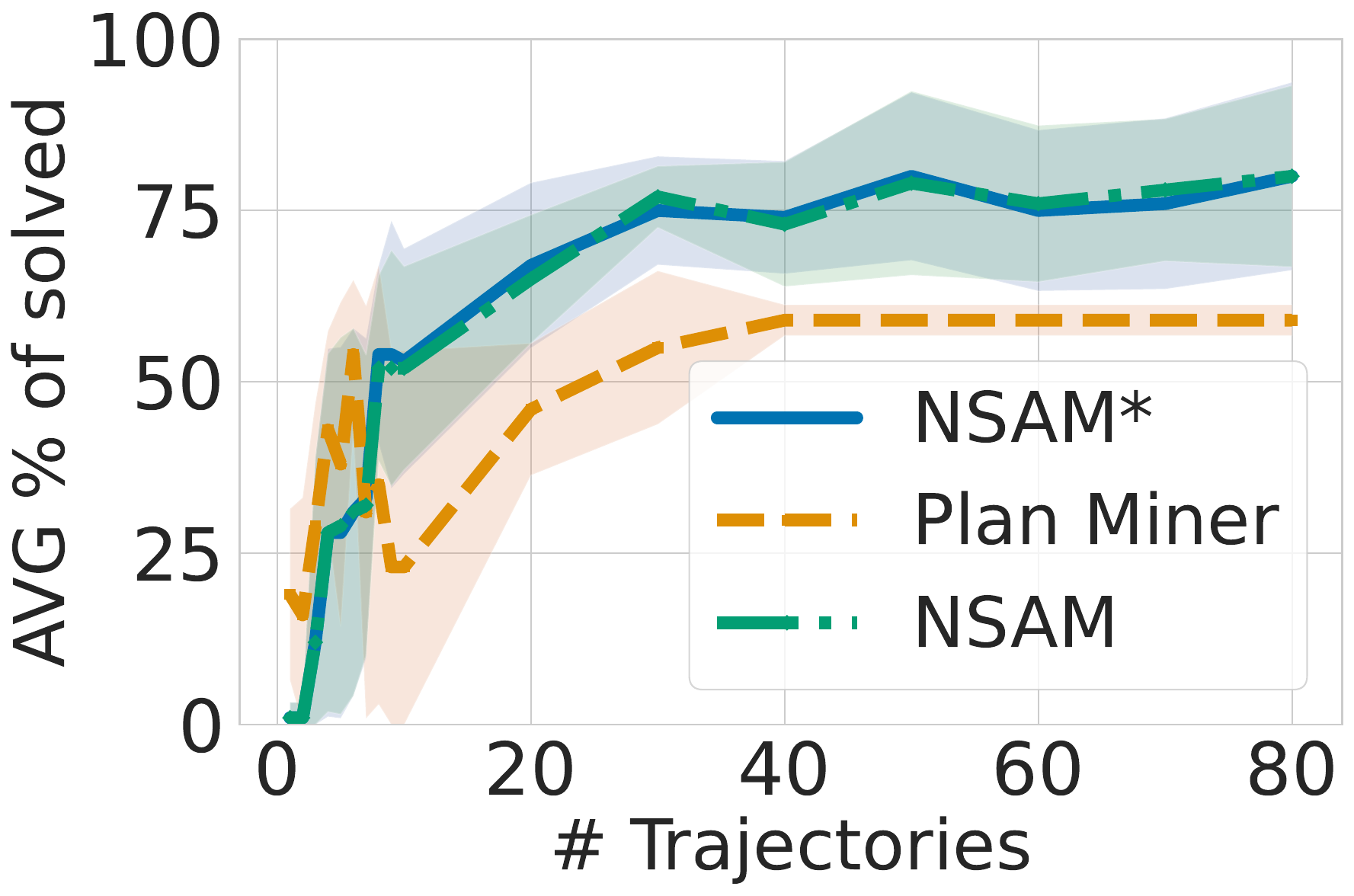}
     \end{subfigure}
      \begin{subfigure}[b]{0.44\textwidth}
         \centering
         \caption{Depots}
         \includegraphics[width=\columnwidth]{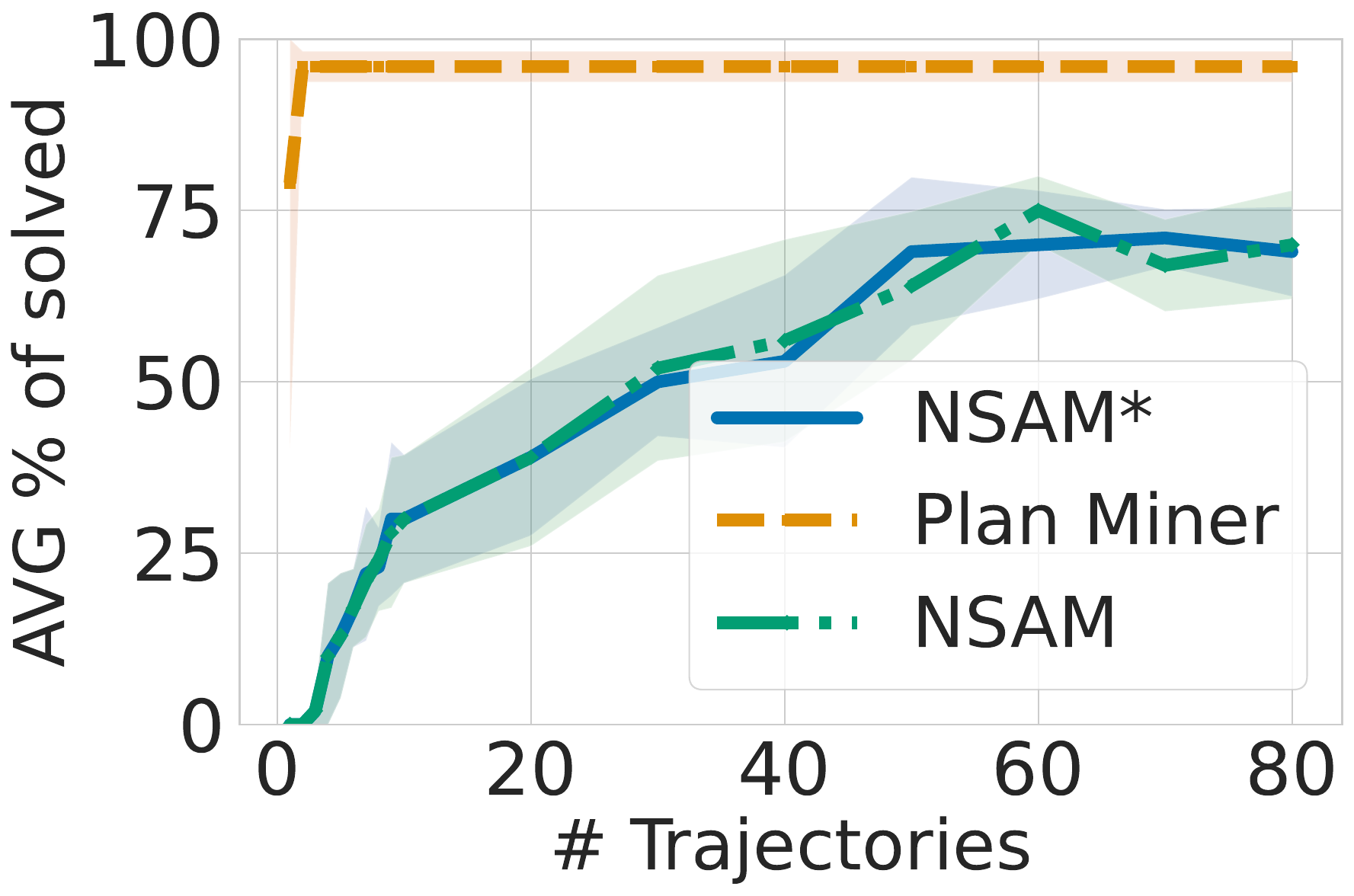}
     \end{subfigure}
       \begin{subfigure}[b]{0.44\textwidth}
         \centering
         \caption{Driverlog-L}
         \includegraphics[width=\columnwidth]{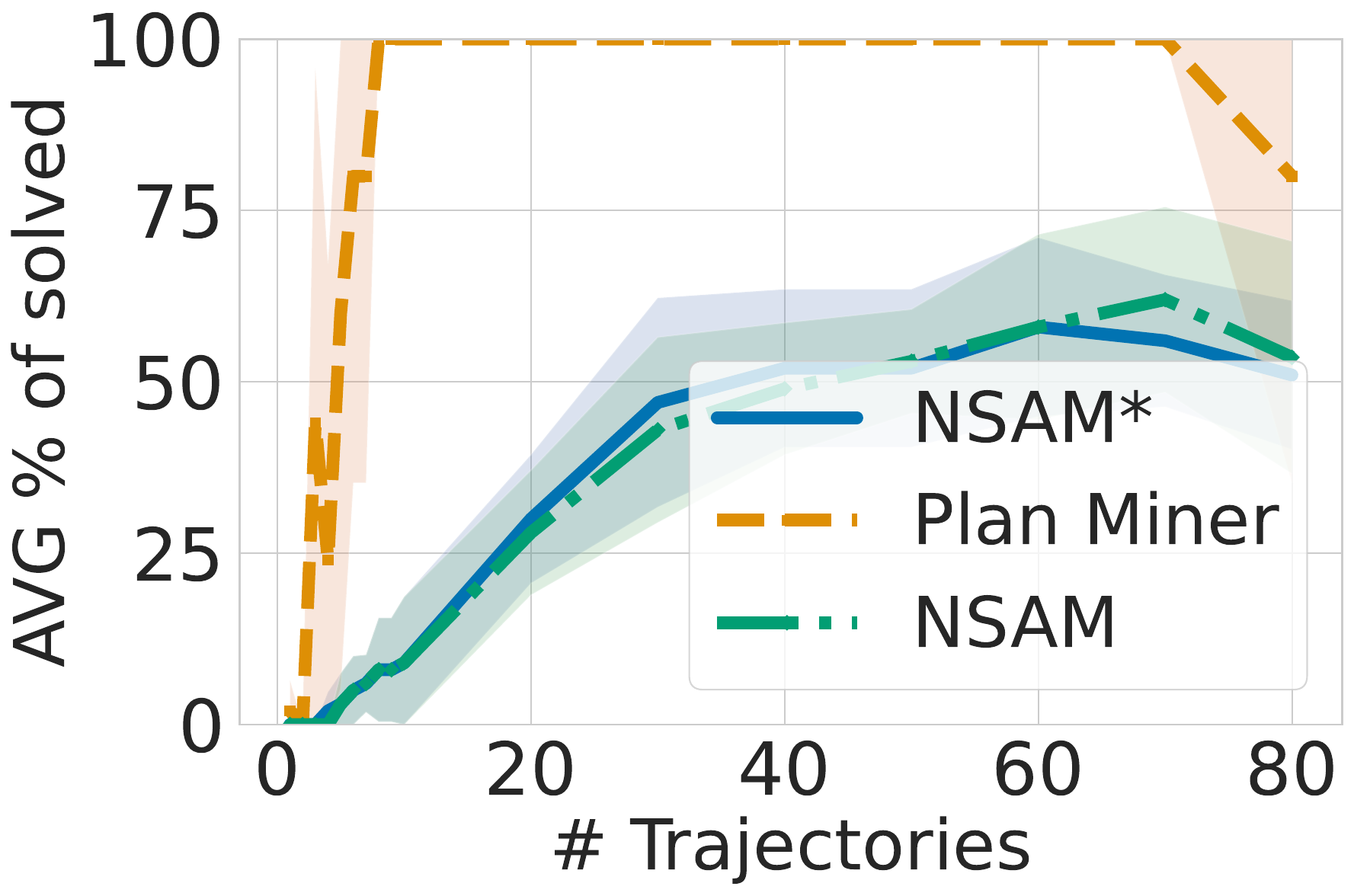}
     \end{subfigure}
    \begin{subfigure}[b]{0.44\textwidth}
         \centering
         \caption{Farmland}
         \includegraphics[width=\columnwidth]{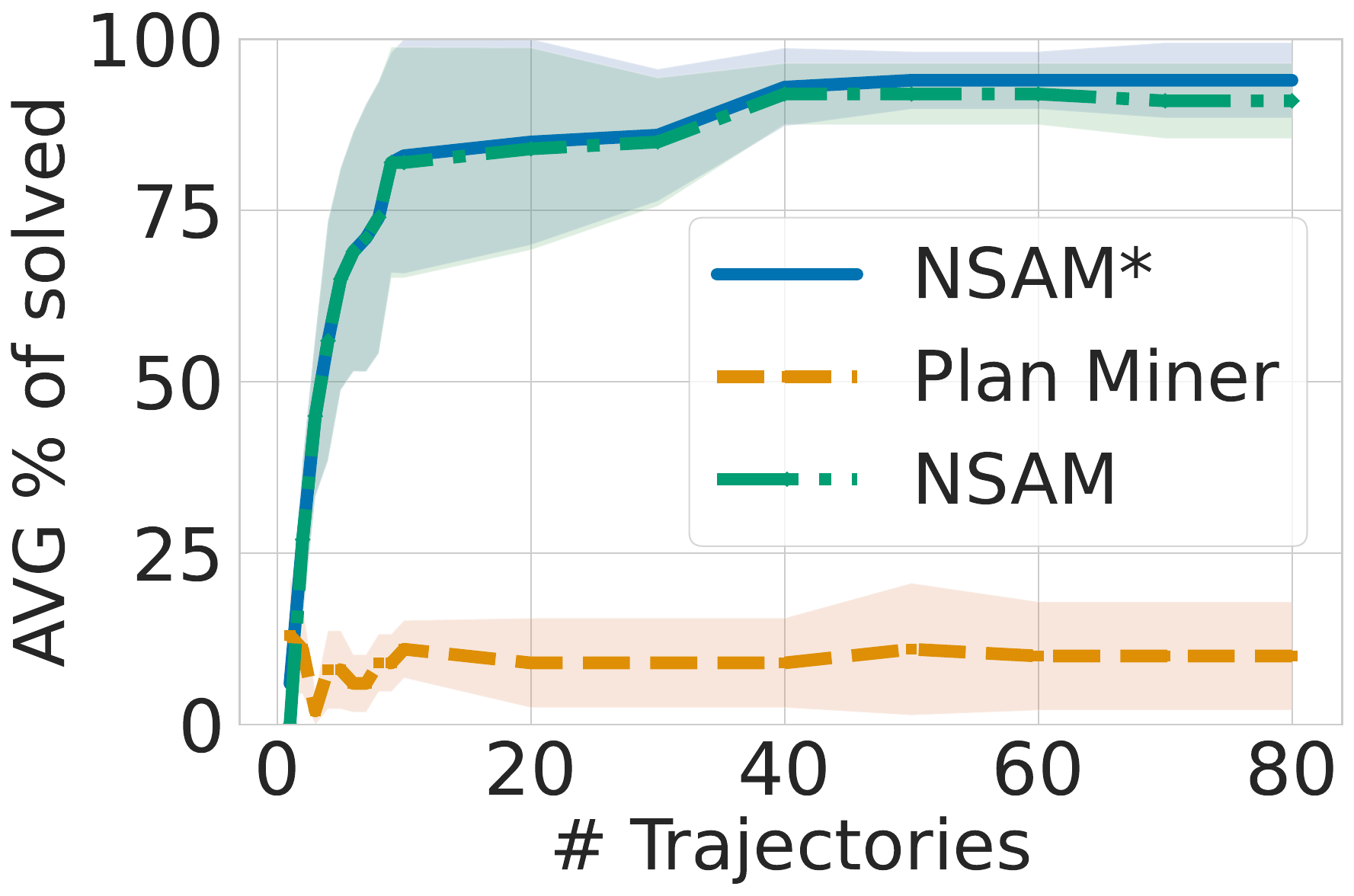}
     \end{subfigure}
    \begin{subfigure}[b]{0.44\textwidth}
         \centering
         \caption{Minecraft}
         \includegraphics[width=\columnwidth]{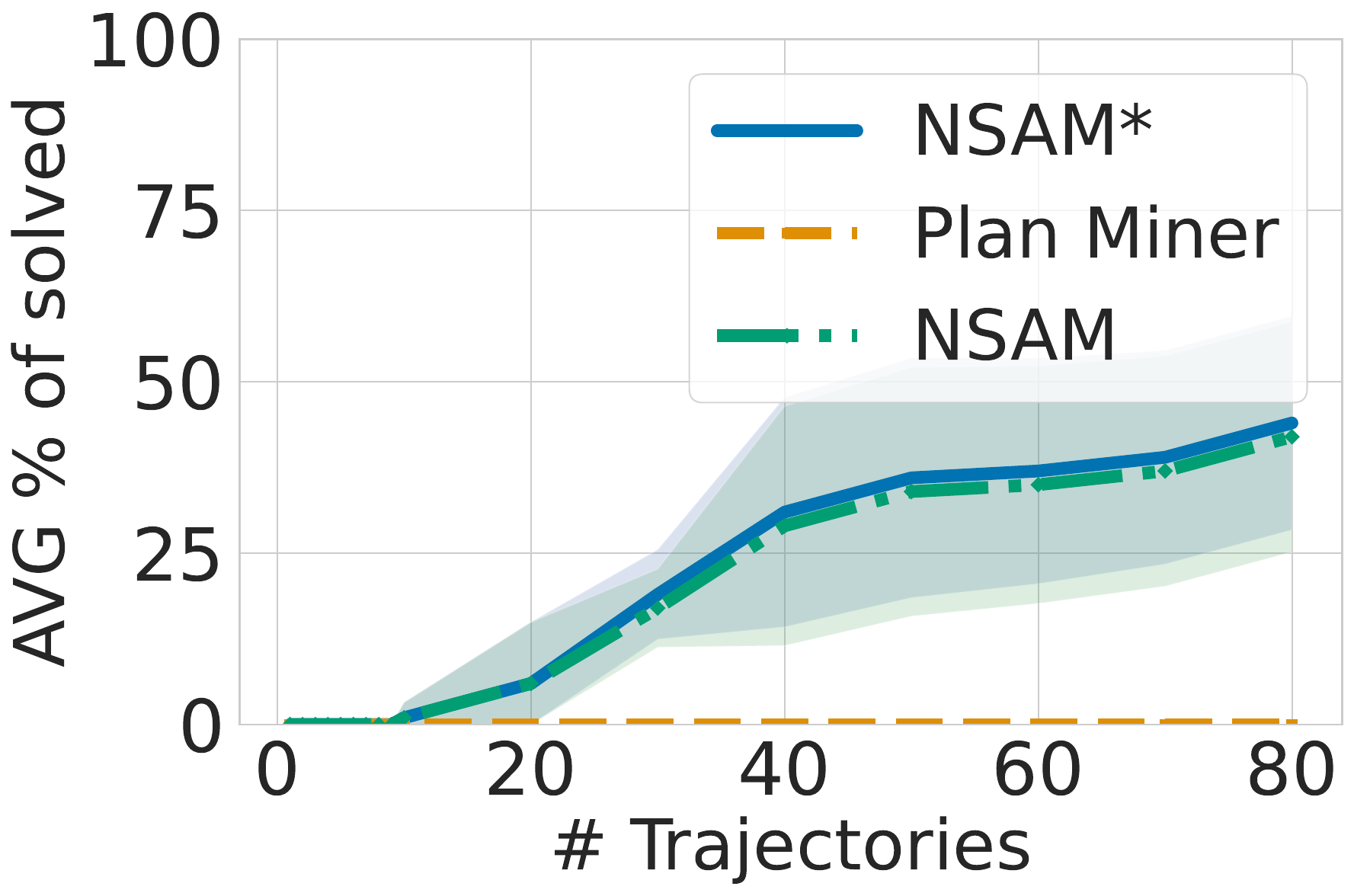}
     \end{subfigure}
     \begin{subfigure}[b]{0.44\textwidth}
         \centering
         \caption{Sailing}
         \includegraphics[width=\columnwidth]{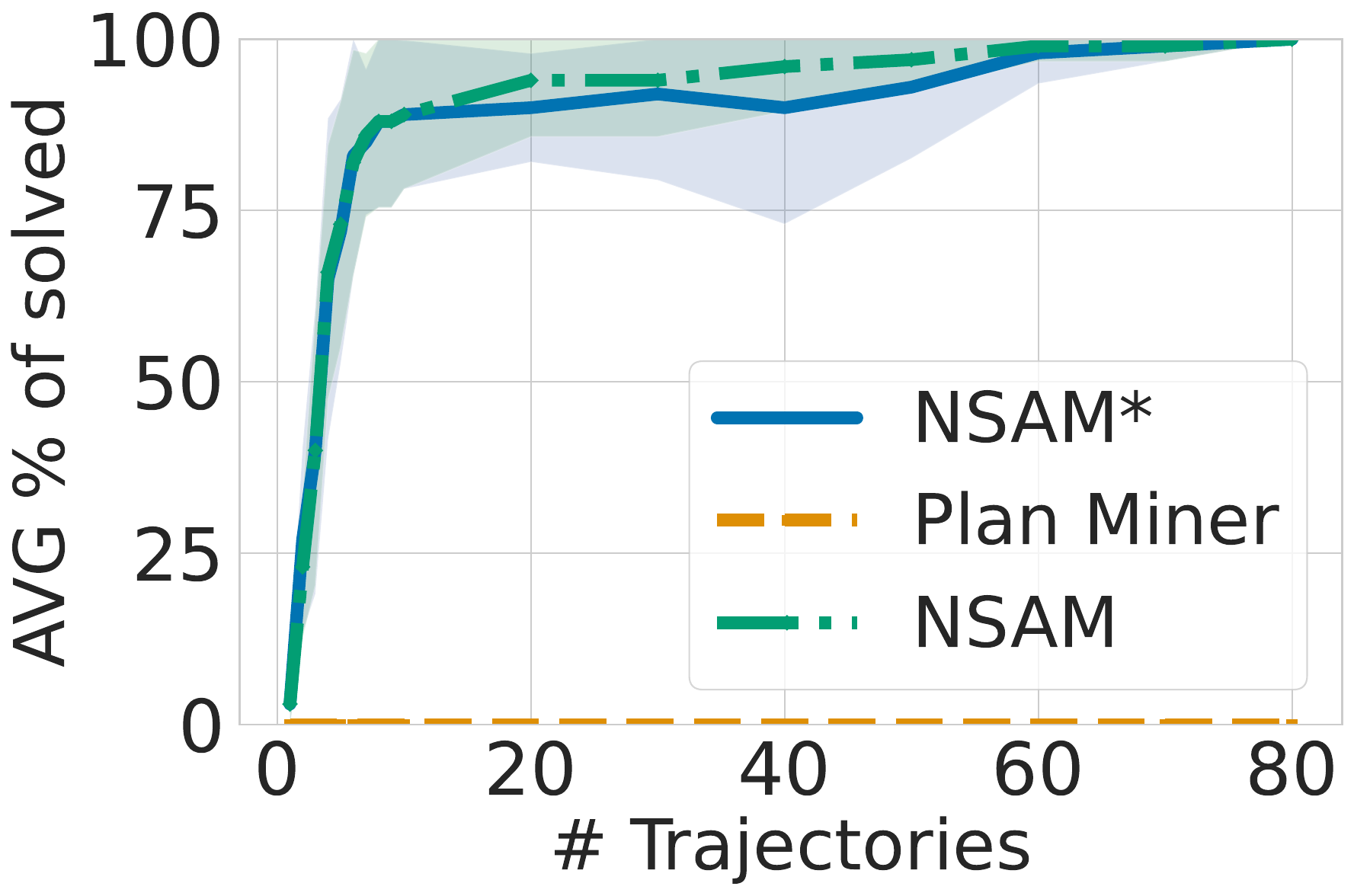}
     \end{subfigure}
    \caption{Comparison of the coverage rates as a function of the number of input trajectories. The orange dashed line shows \pmn, the blue line represents \algname, and the green dashed line denotes \nsam. The figure presents the results for the Counters, Depots, Driverlog-L, Farmland, Minecraft and Sailing domains.}
    \label{fig:problem-solved1}   
    \vspace{-0.3cm}     
\end{figure}

\begin{figure}[H]
     \centering
     \begin{subfigure}[b]{0.44\textwidth}
         \centering
         \caption{Rovers}
         \includegraphics[width=\columnwidth]{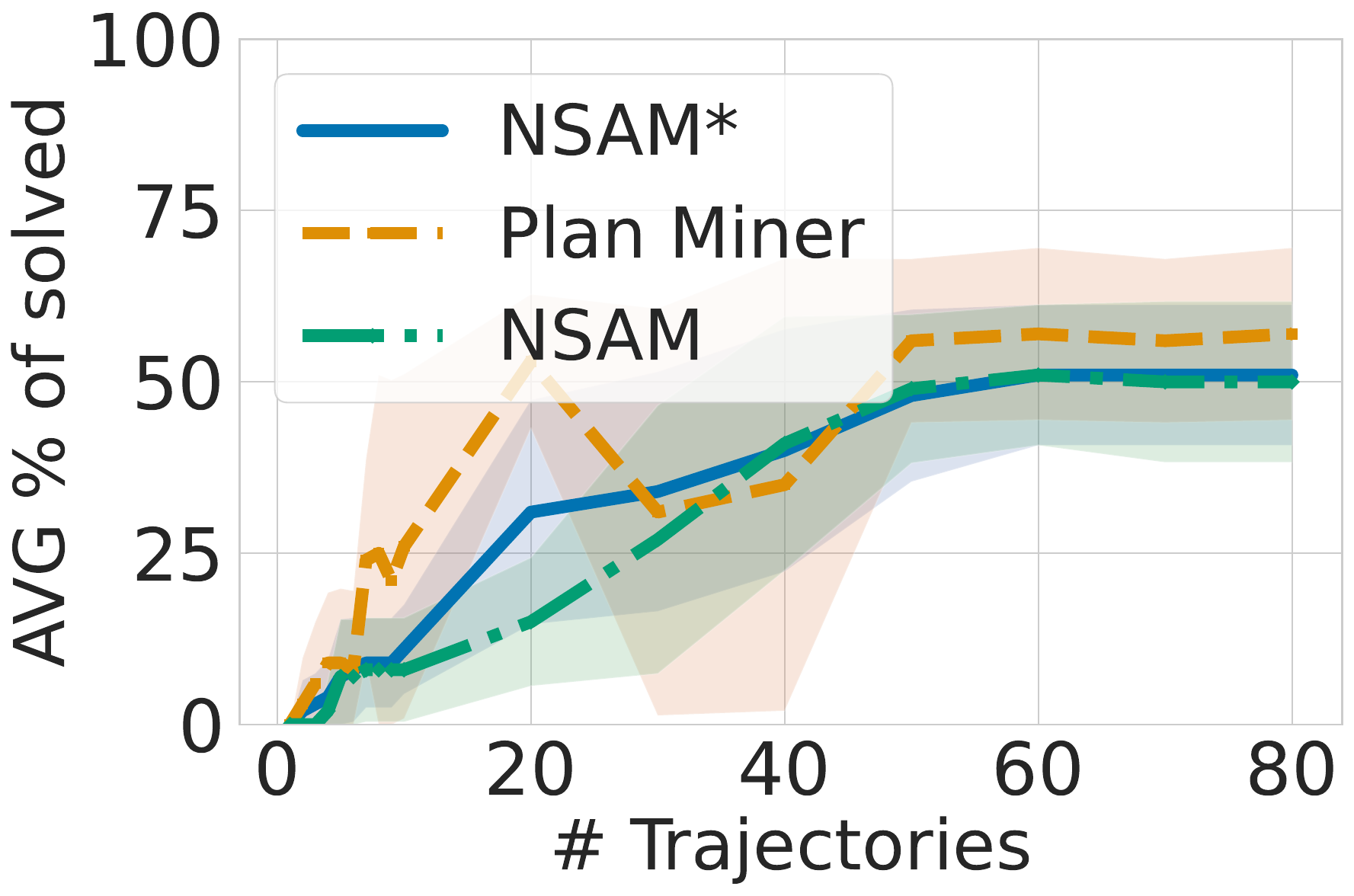}
     \end{subfigure}
      \begin{subfigure}[b]{0.44\textwidth}
         \centering
         \caption{Satellite}
         \includegraphics[width=\columnwidth]{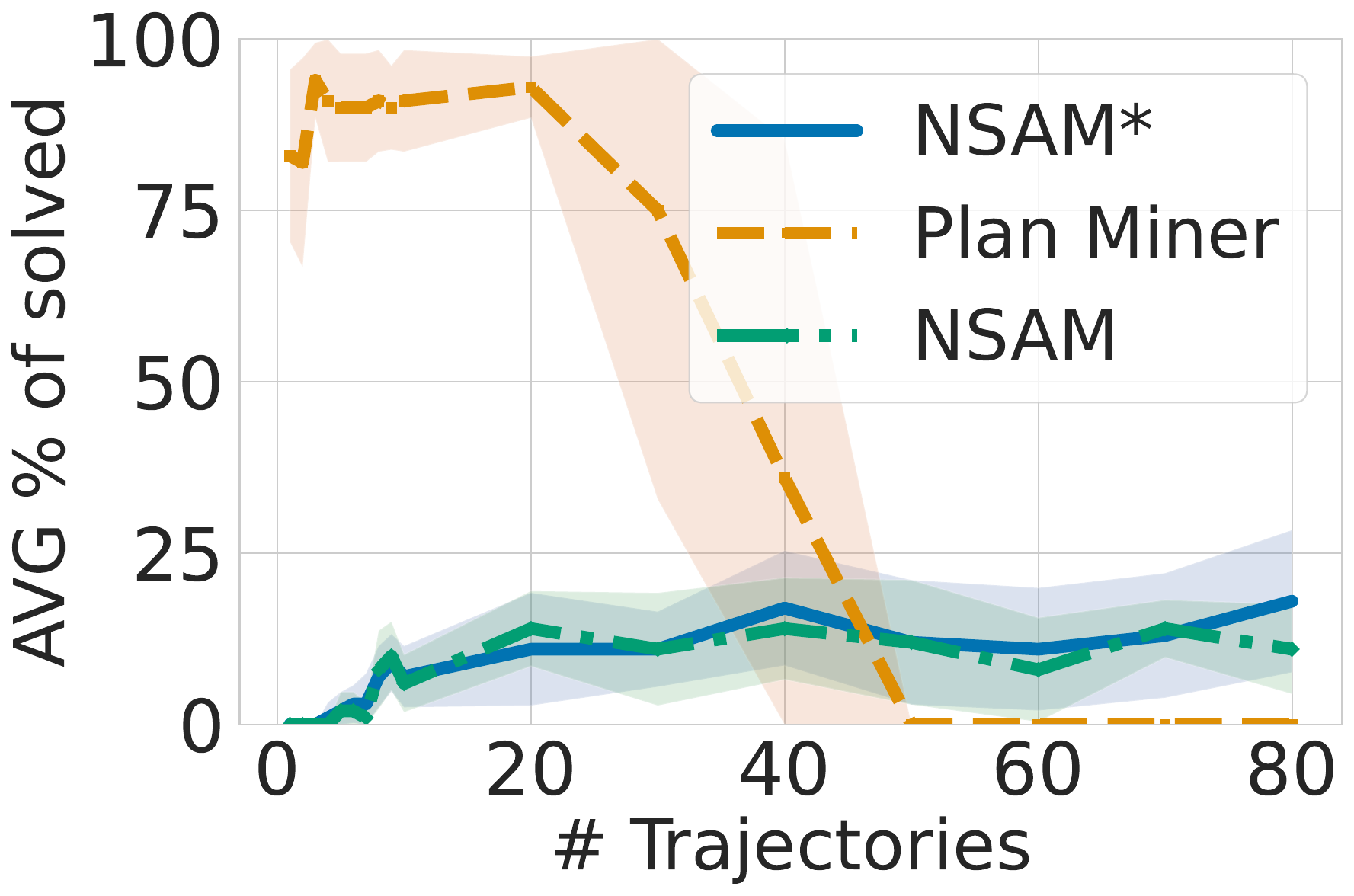}
     \end{subfigure}
       \begin{subfigure}[b]{0.44\textwidth}
         \centering
         \caption{Driverlog-P}
         \includegraphics[width=\columnwidth]{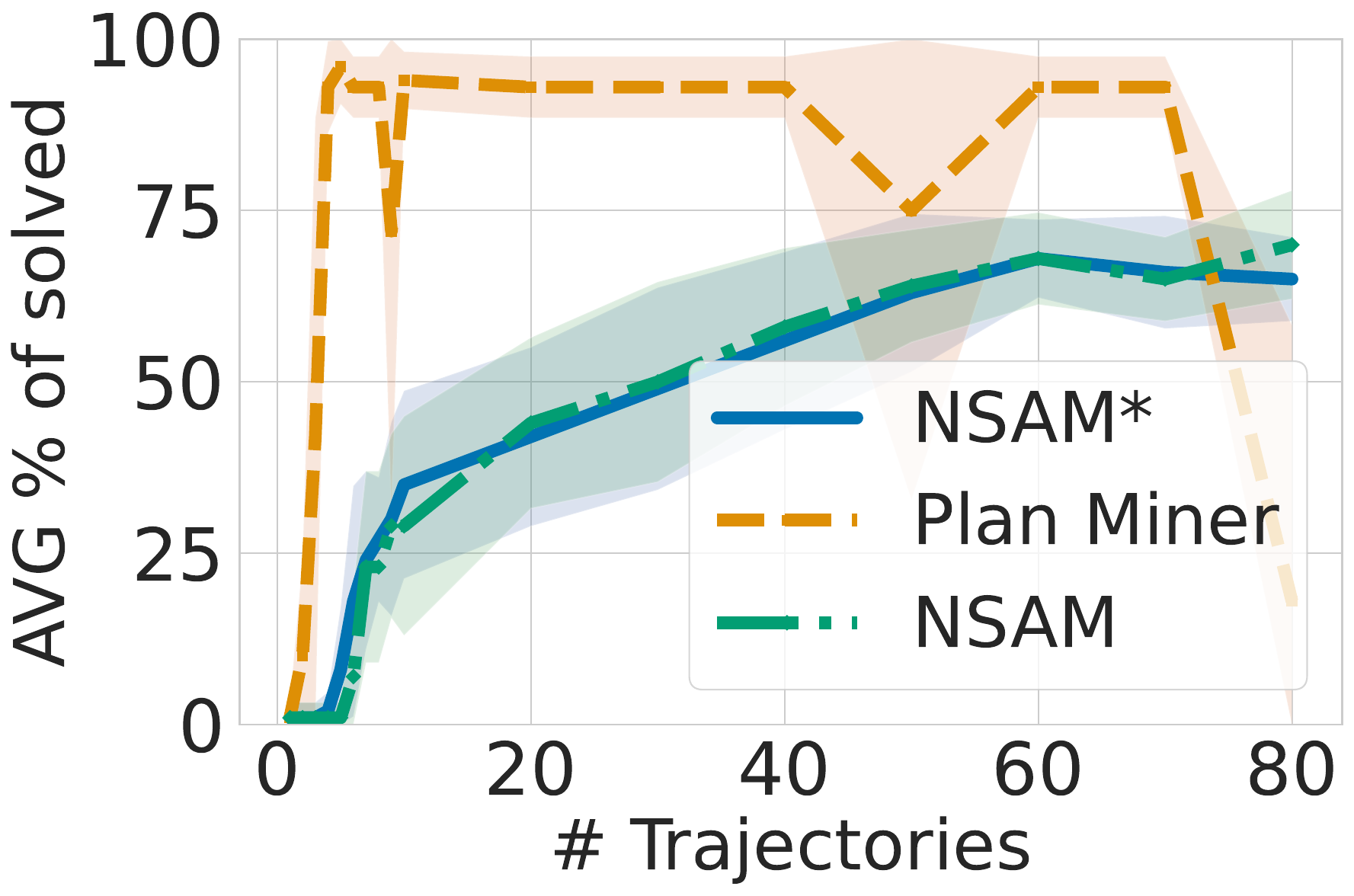}
     \end{subfigure}
    \begin{subfigure}[b]{0.44\textwidth}
         \centering
         \caption{Zenotravel}
         \includegraphics[width=\columnwidth]{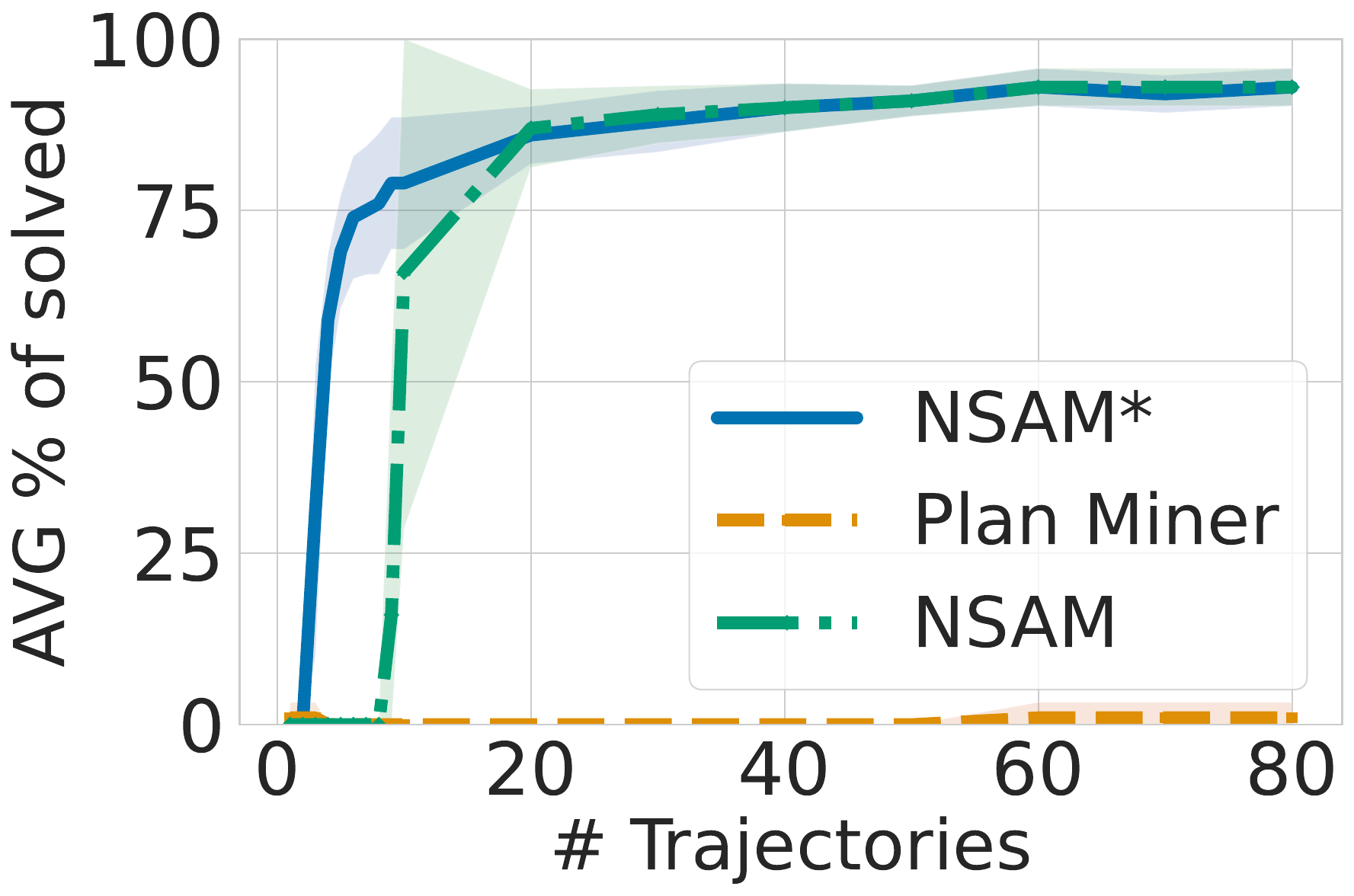}
     \end{subfigure}
    \caption{Comparison of the coverage rates as a function of the number of input trajectories. The orange dashed line shows \pmn, the blue line represents \algname, and the green dashed line denotes \nsam. The figure presents the results for the Rovers, Satellite, Driverlog-P, and Zenotravel domains.}
    \label{fig:problem-solved2}   
    \vspace{-0.3cm}     
\end{figure}

\section{Discussion: Numeric Imprecision}
\label{sec:numeric-precision}
So far, we assumed the learning process does not introduce any numerical error. 
This assumption is not correct when performing numeric calculations on digital computers. 
In this section, we discuss the implications of numeric impressions on the behavior and properties of \nsam . % on we present a discussion on the additional assumptions and possible effects of numerical imprecision.

There are multiple ways to define and bound the numeric impressions introduced by the search. In the remainder of this section, we consider the following, intuitive, bound on the numeric imprecision introduced by the effects, which is based on the $L_\infty$ distance between the real effects and the ones learned. 

\begin{definition}[Bounded Effects Imprecision]
For a given $\epsilon>0$, the bounded effects imprecision assumption states that for every action $a$, state $s$, and action model $M$ returned by \nsam, it holds that
$|a_\realm(s)-a_M(s)|_\infty\leq \epsilon$.
\end{definition}
The Bounded Effects Imprecision (\bei) assumption means that for every grounded function $f$, the difference between $f(a_\realm(s))$ and $f(a_M(s))$ is at most $\epsilon$. 
Clearly, under the \bei the action model $M$, returned by \nsam is not safe, since plans generated with it may be inapplicable or not achieve their goal.

\begin{definition}[The In-Boundary assumption]
Let $\ell$ be an upper bound on the size of plans in the domain, and let $CH_\realm(a)$ be the convex hull defined by the preconditions of action $a$ according to action model $\realm$. 
A set of trajectories satisfies the In-Boundary assumption if for every action $a$ and every state $s$ in which $a$ is applied in the given set of trajectories, it holds that the $L_\infty$ distance between $s$ and $CH_\realm(a)$ is at least $\ell\cdot\epsilon$. 
\end{definition}

Under the In-Boundary and \bei assumptions, we are guaranteed that plans generated by the action model learned by \nsam are safe, in the sense that they are also applicable according to the real action model.
Even under these assumptions, safety is only guaranteed if the goal does not include numeric functions. Ensuring a goal with numeric equality is impossible, due to the numeric imprecision. 
To examine the impact of numeric imprecision, we conducted experiments on the Zenotravel domain, configuring numeric computations to a fixed decimal accuracy. Our previously presented experiments used calculations precise up to four decimal digits; here, we evaluated additional accuracy with one, two, and eight decimal digits. 
Figure~\ref{fig:zenotracel-numeric-accuracy} illustrates the coverage rates as a function of the number of input trajectories for each numeric accuracy setting when using the \algname algorithm. Specifically, the blue solid line, orange dashed line, green solid line, and pink dashed line represent domains configured with one, two, four, and eight decimal digits, respectively. 

As detailed in Section~\ref{sec:experimental-setup}, both the planning algorithms and VAL employed a numeric tolerance of 0.1.
Given that numeric errors accumulate with increased plan length, domains configured with two or fewer decimal digits were expected to show lower coverage compared to more precise configurations. Interestingly, increasing numeric precision from four to eight decimal digits resulted in no noticeable improvement. This suggests that numeric precision beyond four decimal digits does not meaningfully impact the domain's ability to solve problems.

\begin{figure}[H]
    \centering
    \includegraphics[width=0.65\linewidth]{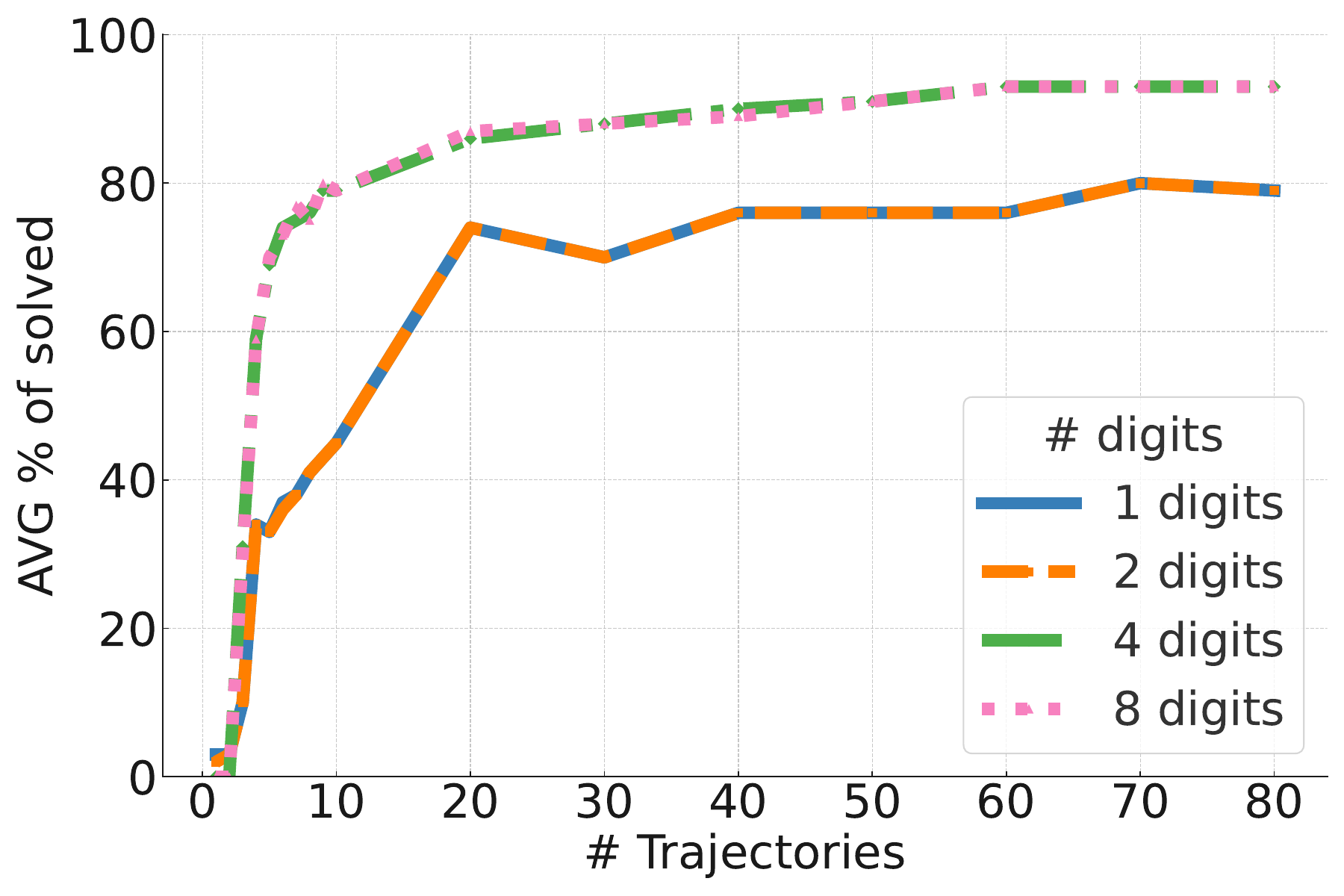}
    \caption{Zenotravel Domain's coverage rates when configured different numbers of decimal digits.}
    \label{fig:zenotracel-numeric-accuracy}
\end{figure}

\section{Conclusion and Future Work}
We analyzed the problem of learning a safe numeric action model and 
identified a set of reasonable assumptions that allow learning such a model. 
Then, we presented the \nsam algorithm, which under this set of assumptions is guaranteed to learn a safe action model for numeric planning domains. 
We showed that the worst-case sample complexity of \nsam does not scale gracefully, but it works well on standard benchmarks. 
A critical limitation of \nsam is that it requires observing for each lifted action $\lifta$ at least $|X(\lifta)|+1$ affine-independent observations to learn the action.
% To create a convex hull for any action $a$, from a set of $m$ examples in $\mathbb{R}^n$, \nsam requires at least $n+1$ affinely independent samples. 
If the input observations contain fewer independent samples, the \nsam considers the action unsafe, and it will always be inapplicable.
% it cannot learn unless presented with at least $|X(a)|+1$ linearly independent observations. 
To overcome this limitation, we presented \algname, an enhancement to the \nsam algorithm that can learn applicable action models with as little as one observation of every action. 
We provided theoretical guarantees to \algname, proving it is an optimal approach to learning numeric action models while still maintaining the safety property. 
Our experimental results show that \algname can learn domains faster, and has an overall comparable performance to \nsam, and in two domains, \algname highly outperforms \nsam.
Comparison of \nsam and \algname to \pmn showed that in domains containing more numeric preconditions and effects, both \nsam and \algname outperform \pmn.
We also provided a discussion on the implications of numeric imprecision during the action model learning process.
% The experiments on the Zenotravel domain with multiple numerical accuracy settings showed that when using less decimal digits the coverage rates worsen. NOT RELEVANT
In future work, we intend to extend our approach to support online numeric action model learning.
We also intend to explore the effects of reducing some of the safety guarantees and their effect on the resulting learned action models.

\section*{Acknowledgments}
% \shahaf{Roni's grants here}
This work was supported by the Israel Science Foundation (ISF)
grant \#909/23, by Israel's Ministry of Innovation, Science and Technology (MOST) grant \#1001706842, in collaboration with Israel National Road Safety Authority and Netivei Israel, and by BSF grant \#2024614, all awarded to Shahaf Shperberg.
This work was also supported by ISF grant \#1238/23 and BSF \#2018684 grant to Roni Stern. 
This work was partially performed while Brendan Juba was at the Simons Institute for the Theory of Computing, and partially funded by NSF awards IIS-1908287, IIS-1939677, and IIS-1942336 to Brendan Juba.

\bibliographystyle{elsarticle-harv}
\bibliography{aij}

\begin{thebibliography}{48}
\expandafter\ifx\csname natexlab\endcsname\relax\def\natexlab#1{#1}\fi
\providecommand{\url}[1]{\texttt{#1}}
\providecommand{\href}[2]{#2}
\providecommand{\path}[1]{#1}
\providecommand{\DOIprefix}{doi:}
\providecommand{\ArXivprefix}{arXiv:}
\providecommand{\URLprefix}{URL: }
\providecommand{\Pubmedprefix}{pmid:}
\providecommand{\doi}[1]{\href{http://dx.doi.org/#1}{\path{#1}}}
\providecommand{\Pubmed}[1]{\href{pmid:#1}{\path{#1}}}
\providecommand{\bibinfo}[2]{#2}
\ifx\xfnm\relax \def\xfnm[#1]{\unskip,\space#1}\fi
%Type = Article
\bibitem[{Adele et~al.(1998)Adele, Knoblock, McDermott, Ram, Veloso, Weld, Sri, Barrett, Christianson et~al.}]{aeronautiques1998pddl}
\bibinfo{author}{Adele}, \bibinfo{author}{Knoblock, C.}, \bibinfo{author}{McDermott, I.D.}, \bibinfo{author}{Ram, A.}, \bibinfo{author}{Veloso, M.}, \bibinfo{author}{Weld, D.}, \bibinfo{author}{Sri, D.W.}, \bibinfo{author}{Barrett, A.}, \bibinfo{author}{Christianson, D.}, et~al., \bibinfo{year}{1998}.
\newblock \bibinfo{title}{Pddl| the planning domain definition language}.
\newblock \bibinfo{journal}{Technical Report, Tech. Rep.} .
%Type = Article
\bibitem[{Aineto et~al.(2019)Aineto, Celorrio and Onaindia}]{aineto2019learning}
\bibinfo{author}{Aineto, D.}, \bibinfo{author}{Celorrio, S.J.}, \bibinfo{author}{Onaindia, E.}, \bibinfo{year}{2019}.
\newblock \bibinfo{title}{Learning action models with minimal observability}.
\newblock \bibinfo{journal}{Artificial Intelligence} \bibinfo{volume}{275}, \bibinfo{pages}{104--137}.
%Type = Inproceedings
\bibitem[{Aldinger and Nebel(2017)}]{aldinger2017interval}
\bibinfo{author}{Aldinger, J.}, \bibinfo{author}{Nebel, B.}, \bibinfo{year}{2017}.
\newblock \bibinfo{title}{Interval based relaxation heuristics for numeric planning with action costs}, in: \bibinfo{booktitle}{Joint German/Austrian Conference on Artificial Intelligence (K{\"u}nstliche Intelligenz)}, \bibinfo{organization}{Springer}. pp. \bibinfo{pages}{15--28}.
%Type = Article
\bibitem[{Amir and Chang(2008)}]{amir2008learning}
\bibinfo{author}{Amir, E.}, \bibinfo{author}{Chang, A.}, \bibinfo{year}{2008}.
\newblock \bibinfo{title}{Learning partially observable deterministic action models}.
\newblock \bibinfo{journal}{Journal of Artificial Intelligence Research} \bibinfo{volume}{33}, \bibinfo{pages}{349--402}.
%Type = Article
\bibitem[{Anthony et~al.(1996)Anthony, Bartlett, Ishai and Shawe-Taylor}]{anthony1996valid}
\bibinfo{author}{Anthony, M.}, \bibinfo{author}{Bartlett, P.}, \bibinfo{author}{Ishai, Y.}, \bibinfo{author}{Shawe-Taylor, J.}, \bibinfo{year}{1996}.
\newblock \bibinfo{title}{Valid generalisation from approximate interpolation}.
\newblock \bibinfo{journal}{Combinatorics, Probability and Computing} \bibinfo{volume}{5}, \bibinfo{pages}{191--214}.
%Type = Inproceedings
\bibitem[{Asai and Fukunaga(2018)}]{asai2018classical}
\bibinfo{author}{Asai, M.}, \bibinfo{author}{Fukunaga, A.}, \bibinfo{year}{2018}.
\newblock \bibinfo{title}{Classical planning in deep latent space: Bridging the subsymbolic-symbolic boundary}, in: \bibinfo{booktitle}{Proceedings of the aaai conference on artificial intelligence}.
%Type = Misc
\bibitem[{Barber(2011)}]{barber2011qhull}
\bibinfo{author}{Barber, C.}, \bibinfo{year}{2011}.
\newblock \bibinfo{title}{Qhull code for convex hull, delaunay triangulation, voroni diagram, and halfspace intersection}.
%Type = Article
\bibitem[{Barber et~al.(1996)Barber, Dobkin and Huhdanpaa}]{barber1996quickhull}
\bibinfo{author}{Barber, C.B.}, \bibinfo{author}{Dobkin, D.P.}, \bibinfo{author}{Huhdanpaa, H.}, \bibinfo{year}{1996}.
\newblock \bibinfo{title}{The quickhull algorithm for convex hulls}.
\newblock \bibinfo{journal}{ACM Transactions on Mathematical Software (TOMS)} \bibinfo{volume}{22}, \bibinfo{pages}{469--483}.
%Type = Inproceedings
\bibitem[{Benyamin et~al.(2023)Benyamin, Mordoch, Shperberg and Stern}]{benyamin2023model}
\bibinfo{author}{Benyamin, Y.}, \bibinfo{author}{Mordoch, A.}, \bibinfo{author}{Shperberg, S.S.}, \bibinfo{author}{Stern, R.}, \bibinfo{year}{2023}.
\newblock \bibinfo{title}{Model learning to solve minecraft tasks}, in: \bibinfo{booktitle}{PRL Workshop in ICAPS}.
%Type = Article
\bibitem[{Chen and Thi{\'e}baux(2024)}]{chen2024graph}
\bibinfo{author}{Chen, D.}, \bibinfo{author}{Thi{\'e}baux, S.}, \bibinfo{year}{2024}.
\newblock \bibinfo{title}{Graph learning for numeric planning}.
\newblock \bibinfo{journal}{Advances in Neural Information Processing Systems} \bibinfo{volume}{37}, \bibinfo{pages}{91156--91183}.
%Type = Inproceedings
\bibitem[{Cresswell and Gregory(2011)}]{cresswell2011generalised}
\bibinfo{author}{Cresswell, S.}, \bibinfo{author}{Gregory, P.}, \bibinfo{year}{2011}.
\newblock \bibinfo{title}{Generalised domain model acquisition from action traces}, in: \bibinfo{booktitle}{International Conference on Automated Planning and Scheduling (ICAPS)}, pp. \bibinfo{pages}{42--49}.
%Type = Article
\bibitem[{Cresswell et~al.(2013)Cresswell, McCluskey and West}]{cresswell2013acquiring}
\bibinfo{author}{Cresswell, S.}, \bibinfo{author}{McCluskey, T.}, \bibinfo{author}{West, M.}, \bibinfo{year}{2013}.
\newblock \bibinfo{title}{Acquiring planning domain models using locm}.
\newblock \bibinfo{journal}{The Knowledge Engineering Review} \bibinfo{volume}{28}, \bibinfo{pages}{195--213}.
%Type = Inproceedings
\bibitem[{Fox and Long(2002)}]{fox2002pddl+}
\bibinfo{author}{Fox, M.}, \bibinfo{author}{Long, D.}, \bibinfo{year}{2002}.
\newblock \bibinfo{title}{Pddl+: Modeling continuous time dependent effects}, in: \bibinfo{booktitle}{the International NASA Workshop on Planning and Scheduling for Space}, p.~\bibinfo{pages}{34}.
%Type = Article
\bibitem[{Fox and Long(2003)}]{fox2003pddl2}
\bibinfo{author}{Fox, M.}, \bibinfo{author}{Long, D.}, \bibinfo{year}{2003}.
\newblock \bibinfo{title}{Pddl2.1: An extension to pddl for expressing temporal planning domains}.
\newblock \bibinfo{journal}{Journal of Artificial Intelligence Research} \bibinfo{volume}{20}, \bibinfo{pages}{61--124}.
%Type = Phdthesis
\bibitem[{Goldberg(1992)}]{goldberg1992pac}
\bibinfo{author}{Goldberg, P.W.}, \bibinfo{year}{1992}.
\newblock \bibinfo{title}{PAC-learning geometrical figures}.
\newblock Ph.D. thesis. University of Edinburgh.
%Type = Inproceedings
\bibitem[{Gregory and Lindsay(2016)}]{gregory2016domain}
\bibinfo{author}{Gregory, P.}, \bibinfo{author}{Lindsay, A.}, \bibinfo{year}{2016}.
\newblock \bibinfo{title}{Domain model acquisition in domains with action costs.}, in: \bibinfo{booktitle}{International Conference on Automated Planning and Scheduling (ICAPS)}, pp. \bibinfo{pages}{149--157}.
%Type = Article
\bibitem[{Hoffmann(2001)}]{hoffmann2001ff}
\bibinfo{author}{Hoffmann, J.}, \bibinfo{year}{2001}.
\newblock \bibinfo{title}{Ff: The fast-forward planning system}.
\newblock \bibinfo{journal}{AI magazine} \bibinfo{volume}{22}, \bibinfo{pages}{57--57}.
%Type = Article
\bibitem[{Hoffmann(2003)}]{hoffmann2003metric}
\bibinfo{author}{Hoffmann, J.}, \bibinfo{year}{2003}.
\newblock \bibinfo{title}{The metric-ff planning system: Translating ``ignoring delete lists'' to numeric state variables}.
\newblock \bibinfo{journal}{Journal of Artificial Intelligence Research} \bibinfo{volume}{20}, \bibinfo{pages}{291--341}.
%Type = Inproceedings
\bibitem[{Howey et~al.(2004)Howey, Long and Fox}]{howey2004val}
\bibinfo{author}{Howey, R.}, \bibinfo{author}{Long, D.}, \bibinfo{author}{Fox, M.}, \bibinfo{year}{2004}.
\newblock \bibinfo{title}{Val: Automatic plan validation, continuous effects and mixed initiative planning using pddl}, in: \bibinfo{booktitle}{16th IEEE International Conference on Tools with Artificial Intelligence}, \bibinfo{organization}{IEEE}. pp. \bibinfo{pages}{294--301}.
%Type = Inproceedings
\bibitem[{Juba et~al.(2021)Juba, Le and Stern}]{juba2021safe}
\bibinfo{author}{Juba, B.}, \bibinfo{author}{Le, H.S.}, \bibinfo{author}{Stern, R.}, \bibinfo{year}{2021}.
\newblock \bibinfo{title}{Safe learning of lifted action models}, in: \bibinfo{booktitle}{International Conference on Principles of Knowledge Representation and Reasoning ({KR})}, pp. \bibinfo{pages}{379--389}.
%Type = Inproceedings
\bibitem[{Juba and Stern(2022)}]{juba2022learning}
\bibinfo{author}{Juba, B.}, \bibinfo{author}{Stern, R.}, \bibinfo{year}{2022}.
\newblock \bibinfo{title}{Learning probably approximately complete and safe action models for stochastic worlds}, in: \bibinfo{booktitle}{AAAI Conference on Artificial Intelligence}.
%Type = Article
\bibitem[{Kearns et~al.(1994)Kearns, Li and Valiant}]{kearns1994learning}
\bibinfo{author}{Kearns, M.}, \bibinfo{author}{Li, M.}, \bibinfo{author}{Valiant, L.}, \bibinfo{year}{1994}.
\newblock \bibinfo{title}{Learning boolean formulas}.
\newblock \bibinfo{journal}{Journal of the ACM (JACM)} \bibinfo{volume}{41}, \bibinfo{pages}{1298--1328}.
%Type = Article
\bibitem[{Kivinen(1995)}]{kivinen1995learning}
\bibinfo{author}{Kivinen, J.}, \bibinfo{year}{1995}.
\newblock \bibinfo{title}{Learning reliably and with one-sided error}.
\newblock \bibinfo{journal}{Mathematical systems theory} \bibinfo{volume}{28}, \bibinfo{pages}{141--172}.
%Type = Inproceedings
\bibitem[{Kuroiwa et~al.(2023)Kuroiwa, Shleyfman and Beck}]{kuroiwa2023bound}
\bibinfo{author}{Kuroiwa, R.}, \bibinfo{author}{Shleyfman, A.}, \bibinfo{author}{Beck, J.C.}, \bibinfo{year}{2023}.
\newblock \bibinfo{title}{Extracting and exploiting bounds of numeric variables for optimal linear numeric planning.}, in: \bibinfo{booktitle}{Proc. ECAI}.
%Type = Article
\bibitem[{Kuroiwa et~al.(2022)Kuroiwa, Shleyfman, Piacentini, Castro and Beck}]{kuroiwa2022lmcut}
\bibinfo{author}{Kuroiwa, R.}, \bibinfo{author}{Shleyfman, A.}, \bibinfo{author}{Piacentini, C.}, \bibinfo{author}{Castro, M.P.}, \bibinfo{author}{Beck, J.C.}, \bibinfo{year}{2022}.
\newblock \bibinfo{title}{The lm-cut heuristic family for optimal numeric planning with simple conditions}.
\newblock \bibinfo{journal}{JAIR} \bibinfo{volume}{75}, \bibinfo{pages}{1477--1548}.
%Type = Inproceedings
\bibitem[{Lamanna and Serafini(2024)}]{Lamanna24}
\bibinfo{author}{Lamanna, L.}, \bibinfo{author}{Serafini, L.}, \bibinfo{year}{2024}.
\newblock \bibinfo{title}{Action model learning from noisy traces: a probabilistic approach}, in: \bibinfo{booktitle}{{ICAPS}}, \bibinfo{publisher}{{AAAI} Press}. pp. \bibinfo{pages}{342--350}.
%Type = Book
\bibitem[{Lay(2007)}]{lay2007convex}
\bibinfo{author}{Lay, S.R.}, \bibinfo{year}{2007}.
\newblock \bibinfo{title}{Convex sets and their applications}.
\newblock \bibinfo{publisher}{Courier Corporation}.
%Type = Inproceedings
\bibitem[{Le et~al.(2024)Le, Juba and Stern}]{le2024learning}
\bibinfo{author}{Le, H.S.}, \bibinfo{author}{Juba, B.}, \bibinfo{author}{Stern, R.}, \bibinfo{year}{2024}.
\newblock \bibinfo{title}{Learning safe action models with partial observability}, in: \bibinfo{booktitle}{Proceedings of the AAAI Conference on Artificial Intelligence}, pp. \bibinfo{pages}{20159--20167}.
%Type = Article
\bibitem[{Leon et~al.(2013)Leon, Bj{\"o}rck and Gander}]{leon2013gram}
\bibinfo{author}{Leon, S.J.}, \bibinfo{author}{Bj{\"o}rck, {\AA}.}, \bibinfo{author}{Gander, W.}, \bibinfo{year}{2013}.
\newblock \bibinfo{title}{Gram-schmidt orthogonalization: 100 years and more}.
\newblock \bibinfo{journal}{Numerical Linear Algebra with Applications} \bibinfo{volume}{20}, \bibinfo{pages}{492--532}.
%Type = Inproceedings
\bibitem[{Li et~al.(2018)Li, Scala, Haslum and Bogomolov}]{li2018effect}
\bibinfo{author}{Li, D.}, \bibinfo{author}{Scala, E.}, \bibinfo{author}{Haslum, P.}, \bibinfo{author}{Bogomolov, S.}, \bibinfo{year}{2018}.
\newblock \bibinfo{title}{Effect-abstraction based relaxation for linear numeric planning.}, in: \bibinfo{booktitle}{International Joint Conference on Artificial Intelligence ({IJCAI})}, pp. \bibinfo{pages}{4787--4793}.
%Type = Article
\bibitem[{Long and Fox(2003)}]{long20033rd}
\bibinfo{author}{Long, D.}, \bibinfo{author}{Fox, M.}, \bibinfo{year}{2003}.
\newblock \bibinfo{title}{The 3rd international planning competition: Results and analysis}.
\newblock \bibinfo{journal}{Journal of Artificial Intelligence Research} \bibinfo{volume}{20}, \bibinfo{pages}{1--59}.
%Type = Inproceedings
\bibitem[{Mordoch et~al.(2023)Mordoch, Juba and Stern}]{mordoch2023learning}
\bibinfo{author}{Mordoch, A.}, \bibinfo{author}{Juba, B.}, \bibinfo{author}{Stern, R.}, \bibinfo{year}{2023}.
\newblock \bibinfo{title}{Learning safe numeric action models}, in: \bibinfo{booktitle}{{AAAI}}, \bibinfo{publisher}{{AAAI} Press}. pp. \bibinfo{pages}{12079--12086}.
%Type = Inproceedings
\bibitem[{Mordoch et~al.(2024)Mordoch, Scala, Stern and Juba}]{mordoch2024safe}
\bibinfo{author}{Mordoch, A.}, \bibinfo{author}{Scala, E.}, \bibinfo{author}{Stern, R.}, \bibinfo{author}{Juba, B.}, \bibinfo{year}{2024}.
\newblock \bibinfo{title}{Safe learning of pddl domains with conditional effects}, in: \bibinfo{booktitle}{Proceedings of the International Conference on Automated Planning and Scheduling}, pp. \bibinfo{pages}{387--395}.
%Type = Article
\bibitem[{Natarajan(1991)}]{natarajan1991probably}
\bibinfo{author}{Natarajan, B.K.}, \bibinfo{year}{1991}.
\newblock \bibinfo{title}{Probably approximate learning of sets and functions}.
\newblock \bibinfo{journal}{SIAM Journal on Computing} \bibinfo{volume}{20}, \bibinfo{pages}{328--351}.
%Type = Inproceedings
\bibitem[{Scala et~al.(2017)Scala, Haslum, Magazzeni, Thi{\'e}baux et~al.}]{scala2017landmarks}
\bibinfo{author}{Scala, E.}, \bibinfo{author}{Haslum, P.}, \bibinfo{author}{Magazzeni, D.}, \bibinfo{author}{Thi{\'e}baux, S.}, et~al., \bibinfo{year}{2017}.
\newblock \bibinfo{title}{Landmarks for numeric planning problems.}, in: \bibinfo{booktitle}{International Joint Conference on Artificial Intelligence ({IJCAI})}, pp. \bibinfo{pages}{4384--4390}.
%Type = Inproceedings
\bibitem[{Scala et~al.(2016a)Scala, Haslum and Thi{\'e}baux}]{scala2016heuristics}
\bibinfo{author}{Scala, E.}, \bibinfo{author}{Haslum, P.}, \bibinfo{author}{Thi{\'e}baux, S.}, \bibinfo{year}{2016}a.
\newblock \bibinfo{title}{Heuristics for numeric planning via subgoaling}, in: \bibinfo{booktitle}{IJCAI}, pp. \bibinfo{pages}{3228--3234}.
%Type = Incollection
\bibitem[{Scala et~al.(2016b)Scala, Haslum, Thi{\'e}baux and Ramirez}]{scala2016interval}
\bibinfo{author}{Scala, E.}, \bibinfo{author}{Haslum, P.}, \bibinfo{author}{Thi{\'e}baux, S.}, \bibinfo{author}{Ramirez, M.}, \bibinfo{year}{2016}b.
\newblock \bibinfo{title}{Interval-based relaxation for general numeric planning}, in: \bibinfo{booktitle}{European Conference on Artificial Intelligence (ECAI)}, pp. \bibinfo{pages}{655--663}.
%Type = Article
\bibitem[{Scala et~al.(2020)Scala, Haslum, Thi{\'e}baux and Ramirez}]{scala2020subgoaling}
\bibinfo{author}{Scala, E.}, \bibinfo{author}{Haslum, P.}, \bibinfo{author}{Thi{\'e}baux, S.}, \bibinfo{author}{Ramirez, M.}, \bibinfo{year}{2020}.
\newblock \bibinfo{title}{Subgoaling techniques for satisficing and optimal numeric planning}.
\newblock \bibinfo{journal}{Journal of Artificial Intelligence Research} \bibinfo{volume}{68}, \bibinfo{pages}{691--752}.
%Type = Article
\bibitem[{Segura-Muros et~al.(2021)Segura-Muros, P{\'e}rez and Fern{\'a}ndez-Olivares}]{segura2021discovering}
\bibinfo{author}{Segura-Muros, J.{\'A}.}, \bibinfo{author}{P{\'e}rez, R.}, \bibinfo{author}{Fern{\'a}ndez-Olivares, J.}, \bibinfo{year}{2021}.
\newblock \bibinfo{title}{Discovering relational and numerical expressions from plan traces for learning action models}.
\newblock \bibinfo{journal}{Applied Intelligence} , \bibinfo{pages}{1--17}.
%Type = Misc
\bibitem[{Seipp et~al.(2022)Seipp, Torralba and Hoffmann}]{seipp-et-al-zenodo2022}
\bibinfo{author}{Seipp, J.}, \bibinfo{author}{Torralba, {\'A}.}, \bibinfo{author}{Hoffmann, J.}, \bibinfo{year}{2022}.
\newblock \bibinfo{title}{{PDDL} generators}.
\newblock \bibinfo{howpublished}{\url{https://doi.org/10.5281/zenodo.6382173}}.
%Type = Inproceedings
\bibitem[{Shleyfman et~al.(2023)Shleyfman, Kuroiwa and Beck}]{shleyfman2023symmetry}
\bibinfo{author}{Shleyfman, A.}, \bibinfo{author}{Kuroiwa, R.}, \bibinfo{author}{Beck, J.C.}, \bibinfo{year}{2023}.
\newblock \bibinfo{title}{Symmetry detection and breaking in linear cost-optimal numeric planning}, in: \bibinfo{booktitle}{Proceedings of the International Conference on Automated Planning and Scheduling}, pp. \bibinfo{pages}{393--401}.
%Type = Inproceedings
\bibitem[{Stern and Juba(2017)}]{stern2017efficientAndSafe}
\bibinfo{author}{Stern, R.}, \bibinfo{author}{Juba, B.}, \bibinfo{year}{2017}.
\newblock \bibinfo{title}{Efficient, safe, and probably approximately complete learning of action models}, in: \bibinfo{booktitle}{International Joint Conference on Artificial Intelligence ({IJCAI})}, pp. \bibinfo{pages}{4405--4411}.
%Type = Misc
\bibitem[{Taitler et~al.(2024)Taitler, Alford, Espasa, Behnke, Fi{\v{s}}er, Gimelfarb, Pommerening, Sanner, Scala, Schreiber et~al.}]{taitler20242023}
\bibinfo{author}{Taitler, A.}, \bibinfo{author}{Alford, R.}, \bibinfo{author}{Espasa, J.}, \bibinfo{author}{Behnke, G.}, \bibinfo{author}{Fi{\v{s}}er, D.}, \bibinfo{author}{Gimelfarb, M.}, \bibinfo{author}{Pommerening, F.}, \bibinfo{author}{Sanner, S.}, \bibinfo{author}{Scala, E.}, \bibinfo{author}{Schreiber, D.}, et~al., \bibinfo{year}{2024}.
\newblock \bibinfo{title}{The 2023 international planning competition}.
%Type = Article
\bibitem[{Toyer et~al.(2020)Toyer, Thi{\'e}baux, Trevizan and Xie}]{toyer2020asnets}
\bibinfo{author}{Toyer, S.}, \bibinfo{author}{Thi{\'e}baux, S.}, \bibinfo{author}{Trevizan, F.}, \bibinfo{author}{Xie, L.}, \bibinfo{year}{2020}.
\newblock \bibinfo{title}{Asnets: Deep learning for generalised planning}.
\newblock \bibinfo{journal}{Journal of Artificial Intelligence Research} \bibinfo{volume}{68}, \bibinfo{pages}{1--68}.
%Type = Inproceedings
\bibitem[{Wang and Thi{\'e}baux(2024)}]{wang2024learning}
\bibinfo{author}{Wang, R.X.}, \bibinfo{author}{Thi{\'e}baux, S.}, \bibinfo{year}{2024}.
\newblock \bibinfo{title}{Learning generalised policies for numeric planning}, in: \bibinfo{booktitle}{Proceedings of the International Conference on Automated Planning and Scheduling}, pp. \bibinfo{pages}{633--642}.
%Type = Article
\bibitem[{Watson(1967)}]{watson1967linear}
\bibinfo{author}{Watson, G.S.}, \bibinfo{year}{1967}.
\newblock \bibinfo{title}{Linear least squares regression}.
\newblock \bibinfo{journal}{The Annals of Mathematical Statistics} , \bibinfo{pages}{1679--1699}.
%Type = Inproceedings
\bibitem[{Xi et~al.(2024)Xi, Gould and Thi{\'e}baux}]{xi2024neuro}
\bibinfo{author}{Xi, K.}, \bibinfo{author}{Gould, S.}, \bibinfo{author}{Thi{\'e}baux, S.}, \bibinfo{year}{2024}.
\newblock \bibinfo{title}{Neuro-symbolic learning of lifted action models from visual traces}, in: \bibinfo{booktitle}{Proceedings of the International Conference on Automated Planning and Scheduling}, pp. \bibinfo{pages}{653--662}.
%Type = Article
\bibitem[{Yang et~al.(2007)Yang, Wu and Jiang}]{yang2007learning}
\bibinfo{author}{Yang, Q.}, \bibinfo{author}{Wu, K.}, \bibinfo{author}{Jiang, Y.}, \bibinfo{year}{2007}.
\newblock \bibinfo{title}{Learning action models from plan examples using weighted max-sat}.
\newblock \bibinfo{journal}{Artificial Intelligence} \bibinfo{volume}{171}, \bibinfo{pages}{107--143}.

\end{thebibliography}
\end{document}